\documentclass[10pt]{article}
\PassOptionsToPackage{table}{xcolor}
\usepackage[preprint]{tmlr}

\usepackage{amsmath,amsfonts,bm}

\def\eqref#1{equation~\ref{#1}}

\def\1{\bm{1}}

\DeclareMathAlphabet{\mathsfit}{\encodingdefault}{\sfdefault}{m}{sl}
\SetMathAlphabet{\mathsfit}{bold}{\encodingdefault}{\sfdefault}{bx}{n}

\DeclareMathOperator*{\argmax}{arg\,max}

\newcommand{\Brel}{B_{\mathrm{rel}}}
\newcommand{\Bval}{\mathcal{B}_{\mathrm{val}}}
\newcommand{\Valid}{\operatorname{Valid}}

\newcommand{\Agent}{\mathsf{A}}

\usepackage{amsmath,amssymb,mathtools}
\usepackage{booktabs,array,multirow,tabularx}
\usepackage{graphicx}
\usepackage{algorithm}
\usepackage{algpseudocode}
\usepackage{float,capt-of}
\usepackage[table]{xcolor}
\usepackage{placeins,needspace}
\usepackage[skins]{tcolorbox}
\usepackage{tabularray}
\UseTblrLibrary{booktabs}
\usepackage{enumitem}
\usepackage{etoolbox}
\usepackage{microtype}
\usepackage{tikz}
\usetikzlibrary{arrows.meta,positioning,fit,calc,shapes.geometric,shapes.misc,backgrounds,decorations.pathreplacing}
\usepackage{hyperref}
\usepackage{url}
\usepackage{pifont}
\newcommand{\ci}[2]{{\scriptsize[#1,\,#2]}}
\newcommand{\cmark}{\ding{51}}
\newcommand{\xmark}{\ding{55}}
\hypersetup{
    colorlinks=true,
    linkcolor=cgblue,
    citecolor=cgblue,
    urlcolor=cgblue,
    pdftitle={Bad Genius: Counterfactual-Guided Harness Evolution Beyond Task-Specific Shortcuts},
    pdfauthor={Guojun Zhu, Xunheng Huang, Peng Yin, Jiahui Xie, Sanguo Zhang, Doudou Zhou}
}

\newtheorem{theorem}{Theorem}

\newtheorem{lemma}{Lemma}
\newtheorem{condition}{Condition}
\AtBeginEnvironment{condition}{\footnotesize}

\newcommand{\methodraw}{RawHarness}
\newcommand{\methodmeta}{Meta-Harness}
\newcommand{\methodhc}{HarnessCompass}
\newcommand{\methodevolve}{HarnessEvolve}
\newcommand{\methodchase}{CHASE}

\newcolumntype{Y}{>{\raggedright\arraybackslash}X}
\newcolumntype{C}{>{\centering\arraybackslash}X}

\definecolor{cgblue}{HTML}{2B6CB0}
\definecolor{cgorange}{HTML}{C66A1B}
\definecolor{cggreen}{HTML}{2F855A}
\definecolor{cgred}{HTML}{B33A3A}
\definecolor{cgpurple}{HTML}{6750A4}
\definecolor{cggray}{HTML}{5F6B7A}
\definecolor{cglightblue}{HTML}{EAF2FB}
\definecolor{cglightorange}{HTML}{FDF1E6}
\definecolor{cglightgreen}{HTML}{EAF6EF}
\definecolor{cglightpurple}{HTML}{F0ECFA}
\definecolor{cglightgray}{HTML}{F3F5F7}
\newtcolorbox{harnesspanel}[2]{
  enhanced,colback=#2,colframe=#1!45!white,
  boxrule=0.4pt,arc=1pt,boxsep=0pt,
  left=6pt,right=6pt,top=5pt,bottom=5pt,
  before skip=0pt,after skip=0pt,
  fontupper=\small,fontlower=\small,
  before upper={\raggedright\setlength{\parindent}{0pt}},
  before lower={\raggedright\setlength{\parindent}{0pt}},
  segmentation style={solid,draw=#1!30!white,line width=0.35pt},
  middle=4pt
}

\title{Bad Genius: Counterfactual-Guided Harness Evolution\\Beyond Task-Specific Shortcuts}

\author{
{\name Guojun Zhu\textsuperscript{1,2}\quad
Xunheng Huang\textsuperscript{2}\quad
Peng Yin\textsuperscript{3}}\\[0.3em]
{\name Jiahui Xie\textsuperscript{2}\quad
Sanguo Zhang\textsuperscript{1}\quad
Doudou Zhou\textsuperscript{2,\ensuremath{\dagger}}}\\[0.65em]
{\addr\textsuperscript{1}School of Mathematical Sciences, University of Chinese Academy of Sciences, China}\\
{\addr\textsuperscript{2}Department of Statistics \& Data Science, National University of Singapore, Singapore}\\
{\addr\textsuperscript{3}Institute of Automation, Chinese Academy of Sciences, China}\\[0.35em]
{\addr\textsuperscript{\ensuremath{\dagger}}Corresponding author: \texttt{ddzhou@nus.edu.sg}}
}
\date{}
\begin{document}
\maketitle

\begin{abstract}
An optimized agent harness can improve on held-out tasks by relying on formatting conventions that those tasks share with the rest of the benchmark. We introduce Counterfactual Harness Search and Evolution (CHASE) to test whether harness improvements survive changes to the evaluation protocol. A Proposer edits the harness around a fixed foundation model. A Challenger then searches for a protocol change that preserves the tasks but reduces the updated harness's gain over the initial harness. Executable checks test whether the transformation preserves task semantics. If the gain loss is confirmed on held-out tasks, CHASE adds the transformation to an archive that constrains subsequent harness updates. We give conditions under which this finite archive bounds the gain loss on an ideal benchmark with shortcuts removed. We evaluate CHASE on Syn-Ledger, OfficeQA, FinanceBench, and SWE-Bench Pro V2. On Syn-Ledger, removing the known shortcuts yields 38.30\% accuracy for CHASE and 18.27\% for HarnessEvolve. On the 21 certification tasks from SWE-Bench Pro V2 HARD-51, changing the issue formatting and tool-output field order reduces HarnessEvolve's gain over the initial harness from 8.57 to 3.81 percentage points. The final CHASE harness retains 10.48 points of its 11.43-point gain under the same change.
\end{abstract}

\section{Introduction}
\label{sec:introduction}
Reliable agent evaluation is central to progress in agentic AI, with benchmarks playing a key role in determining which systems appear capable, safe, and ready for deployment. A reported score, however, does not reflect the foundation model alone, but also depends on the harness through which it operates and the benchmark protocol under which it is evaluated. {A harness specifies the executable context around the model, including what information is stored and retrieved, how tools and workspace are exposed, and how outputs are handled.} In the literature, the evaluated agent consists of the foundation model with its harness. Evaluation is increasingly concerned not only with performance under a fixed harness, but also with the performance attainable after optimizing the harness around a fixed foundation model. Meta-Harness uses a coding agent (the Proposer) to revise the harness based on previous code, evaluation scores, and execution traces~\citep{lee2026metaharness}. Subsequent work studies held-out evaluation, optimizer quality, priority ranking, and reliable harness selection~\citep{wang2026rethinking,ong2026priority,ursekar2026harnessopt,zhao2026beyondprompts}. Harness-Bench further shows that performance and failure modes vary materially across model--harness configurations under shared tasks~\citep{yao2026harnessbench}. Together, these results make iterative harness optimization, which we call harness evolution, a substantive evaluation target~\citep{zhu2026unified,zhu2026targets}.

Repeated use of benchmark feedback creates a distinct generalization problem. The ``Bad genius''\footnote{The title refers to the film \emph{Bad Genius}, where a gifted student helps others cheat through clever schemes.} Proposer is optimized for benchmark score, not for preserving what that score is intended to measure, and can modify the harness using benchmark feedback and execution traces. Classical shortcut learning concerns predictive rules induced by data or task artifacts~\citep{geirhos2020shortcut, ribeiro2020checklist}. Harness evolution adds an executable route: the Proposer can write the shortcut into the harness. Existing methods ban task identifiers, filenames, and per-task repair recipes, then evaluate on held-out tasks~\citep{wang2026rethinking, zhang2026harnesscompass}. These methods limit a \emph{task-specific shortcut}---for example, detecting one question ID and opening a known document---but not a {\bfseries\itshape benchmark-wide shortcut}, such as searching one document channel first because the released protocol places supporting evidence there more often. As illustrated in Figure~\hyperref[fig:motivation]{1(a)}, task holdout therefore leaves benchmark-wide shortcuts intact: tasks change, but the protocol correlation does not. We examine this problem in OfficeQA Full, a benchmark in which agents answer questions using U.S. Treasury Bulletin documents~\citep{opsahlong2026officeqa}. For example, 58.1\% of questions in OfficeQA Full mention numerical scales such as millions. In its 697-document corpus, the next nonblank line after 95.2\% of unit statements begins a table (Appendix~\ref{app:officeqa-preprocess}). The Proposer may propose a harness that induces the model to inspect the line before a table when looking for units. {A harness is expected to adapt to the documents and tools, and an instruction such as ``keep an explicit unit for every operand'' uses information that solving the task requires. The shortcut above instead relies on a protocol correlation that the task does not require and that only this release happens to fix. Appendix~\ref{app:legitimate-vs-shortcut} develops this distinction with further examples.}

This is distinct from overfitting to particular tasks: the shortcut can persist on held-out tasks. The natural next step is therefore to retain held-out tasks while varying the protocol as well. Figure~\hyperref[fig:motivation]{1(b)--(c)} illustrates this comparison. Under the released protocol shown in Figure~\hyperref[fig:motivation]{1(b)}, the shortcut ``search channel A first'' (e.g., inspect the text immediately preceding a table to identify its unit) still places the supporting evidence first. Figure~\hyperref[fig:motivation]{1(c)} then reassigns the same documents across channels, removing this shortcut advantage while keeping the question, answer, and documents fixed. Changing the protocol may also affect the initial harness $H_0$. We therefore compare the evolved harness's gain over $H_0$ before and after the change. A large loss of gain reveals benchmark dependence that held-out tasks alone cannot expose. This raises a question: {\it can we construct a benchmark that neutralizes all benchmark-wide shortcuts while preserving the underlying tasks?}

\begin{figure}[ht]
  \centering
  \includegraphics[width=0.72\textwidth]{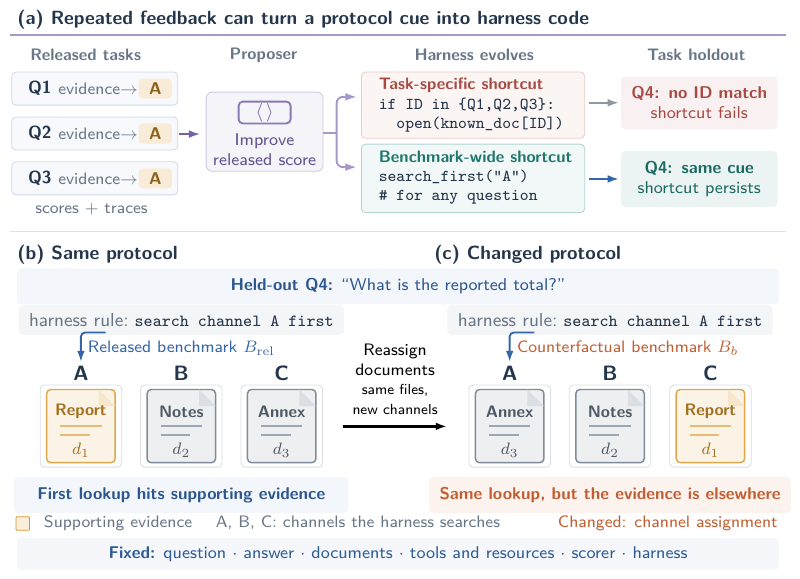}
  \caption{\textbf{Motivation: Why task holdout can miss benchmark-wide shortcuts.} {(a) The Proposer edits the harness from released scores and traces. A task-specific shortcut keys on task identity and fails on a held-out task. A benchmark-wide shortcut keys on a protocol correlation that the held-out task shares, so it persists. (b)--(c) The same held-out task under the released protocol and under a counterfactual protocol. The harness rule is unchanged. The only change is which channel holds each document: ``reassign documents'' keeps every file, question, and answer, and changes the location that the harness searches first. The OfficeQA counterfactual is of the same kind: it changes the format in which retrieval results are returned and keeps every task fixed (Section~\ref{sec:experimental-results}).}}
  \label{fig:motivation}
\end{figure}

If every shortcut mechanism were known, one could construct an ideal shortcut-neutralized benchmark $B_0$ and optimize the harness on it. In realistic benchmarks, however, benchmark-wide shortcuts can arise from many parts of the protocol and remain hidden from both manual inspection and automated checks, making them difficult to pre-specify with fixed rules. We therefore introduce a coding agent named {\bfseries\itshape Challenger} that searches online, after each harness update, for executable protocol transformations, each producing a counterfactual benchmark designed to expose shortcuts in the updated harness. Whereas prior work introduces auxiliary agents as {\it debuggers} or {\it critics} to guide harness repair~\citep{lin2026agentic,li2026evosafeharness}, our Challenger attacks shortcuts at their source---the benchmark protocol. Each valid counterfactual confirmed on held-out tasks becomes a constraint on harness evolution, limiting how much future gains can depend on the same protocol correlation. {CHASE thus pursues two goals: it raises the released score and keeps that gain under valid transformations, and each confirmed counterfactual adds an evaluation column.}

Our contributions are fourfold. First, we distinguish task-specific shortcuts from benchmark-wide shortcuts and define a gain-destruction estimand that assesses whether observed harness gains remain reliable under valid counterfactual benchmarks. Second, we propose \textbf{Counterfactual Harness Search and Evolution (CHASE)}, in which a Challenger searches for counterfactual benchmarks. Third, we establish theoretical guarantees that certify how much harness gain survives shortcut neutralization through counterfactual-guided evolution and when Challenger search can stop. Fourth, across the synthetic benchmark Syn-Ledger {and the real-world benchmarks OfficeQA, FinanceBench, and SWE-Bench Pro}, we show that CHASE yields confirmed counterfactuals and better generalization to counterfactual benchmarks than other baselines.

\section{Related Work}
\label{sec:related-work}
\noindent{\bf Harness optimization.} Earlier work automates large language model (LLM) system optimization through metric-driven pipeline compilation in {\it DSPy}~\citep{khattab2024dspy}, prompt search in {\it OPRO}~\citep{yang2024large}, and feedback propagation in {\it TextGrad}~\citep{yuksekgonul2025optimizing}. {\it Meta-Harness} optimizes prompts, memory, tools, and control code using evaluation feedback~\citep{lee2026metaharness}. {\it HarnessOpt-Bench} evaluates the optimizer under a fixed budget~\citep{ursekar2026harnessopt}, while {\it Priority ranking} tests whether optimizers can identify promising harness components~\citep{ong2026priority}. {\it Harness-Bench} isolates harness effects across model backends~\citep{yao2026harnessbench}. {\it HarnessLens} reduces verification cost by selecting tasks for each candidate and confirming promising edits~\citep{xu2026harnesslens}. {\it AutoSaddler} further combines failure-trace diagnosis with validation-based harness updates~\citep{park2026autosaddler}. These works characterize the optimization object, optimizer quality, and verification efficiency. CHASE addresses a different question: whether feedback from the released benchmark $\Brel$ selects a harness whose gain depends on protocol correlations shared across tasks.

\noindent{\bf Harness generalization.} Held-out-task evaluation shows that optimizing and evaluating harnesses on the same tasks can substantially overstate their gains~\citep{wang2026rethinking, esakkiraja2026starharness}. {\it HarnessCompass} uses a fixed generalization gate to allow only task-agnostic modifications, together with component-wise feedback to guide harness updates~\citep{zhang2026harnesscompass}. {\it HarnessEvolve} uses reference trajectories, quality and performance gates, and held-out validation to filter and select harness updates~\citep{jiang2026harnessevolve}. {\it EvoSafeHarness} uses a separate critic to check whether proposed safety rules still work when tool names or attack wording change~\citep{li2026evosafeharness}. {\it Harness continual learning} evaluates new harness updates for retention of previously acquired behavior~\citep{kang2026harnesscontinual}. More broadly, continually updated benchmarks, lifelong test sets, and overfitting alarms reduce repeated reliance on a fixed evaluation set~\citep{prabhu2024lifelong, ishida2026capbencher}. Overall, existing methods mainly vary tasks, apply fixed filters to harness edits, or validate harness candidate updates before adoption. Task variation leaves the benchmark protocol unchanged, while the fixed filters can only cover anticipated shortcut patterns. Pre-adoption validation assesses candidates under selected test conditions. CHASE instead searches for counterfactual benchmarks against the updated harness while preserving task semantics, without relying on a fixed filter over harness edits.

\noindent{\bf Benchmark validity.} Agent scores are properties of a model--harness--environment--protocol configuration rather than the foundation model alone~\citep{zhu2026unified, zhu2026targets, yao2026harnessbench}. Recent benchmark audits show that agent scores can be distorted by protocol flaws, search-time contamination, broken tasks, and scoring errors~\citep{shao2026protocol, wang2026searchtime, dong2026misscore}. Contamination tests and refreshed benchmarks address exposure of benchmark content~\citep{oren2023contamination, white2025livebench, wu2025antileak}. {\it Auditing Harness Tampering} studies a related problem in self-improving agents, where harness edits can produce apparent performance gains without genuine capability improvement, and develops audits to detect and locate such edits~\citep{wang2026tampering}. {\it HackProbe} similarly detects reward hacking during self-evolution and uses the resulting signal for harness reselection~\citep{yang2026hackprobe}. These works diagnose or reduce specific sources of invalid benchmark gains. CHASE instead tests whether harness gains survive valid protocol transformations, without pre-specifying the shortcut.

\noindent{\bf {Counterexample-guided repair and metamorphic testing}.} {Counterexample-guided repair repeatedly finds an input on which the current model violates a specification and updates the model to remove the violation}~\citep{bauer2021specrepair, boetius2023counterexample}. {\it Model-written evaluations} and {\it automated red teaming} use one model to generate tests that expose failures of another~\citep{perez2022redteam,perez2023modelwritten}. {\it Metamorphic testing} evaluates a system under transformations that keep the task semantics unchanged~\citep{hyun2023metal,cho2025metamorphic}. CHASE combines these ideas for harness evolution: after each Proposer update, the Challenger searches for a valid counterfactual benchmark that destroys the current harness's gain, and a confirmed counterfactual constrains subsequent harness updates. {CHASE shares the loop of generating a failure and removing it, but the Challenger returns a counterfactual benchmark, a transformation of the protocol under which every task is unchanged, rather than a violating input.}

\section{Counterfactual Harness Search and Evolution}
\label{sec:method}

\subsection{Problem Setup}
\label{sec:problem-setup}
{Fix a foundation model $\Agent$. Throughout, the model, the tasks with their answers and evidence, the scoring rule, and the tool operations stay fixed. The harness is what the Proposer edits, and the protocol is what the Challenger will later change. A task is an instance $U$ with response $Y=\psi(U)$, where $\psi(\cdot)$ is the target rule, and $P$ is the distribution over tasks. The released protocol $(V_{\rm rel},Q_{\rm rel})$ is how the benchmark presents its tasks: $V_{\rm rel}$ contains file names, directory layout, metadata, tool aliases, demonstration order, feedback format, and other benchmark-specific details, and $Q_{\rm rel}$ specifies how the agent interacts with these components. The released benchmark $\Brel$ of Section~\ref{sec:introduction} combines the task distribution $P$, the protocol $(V_{\rm rel},Q_{\rm rel})$, the available tools and resource budgets, and a scorer. For simplicity, we write $\Brel=(P,V_{\rm rel},Q_{\rm rel},\psi)$.} For any benchmark variant $B$ sharing the semantic-task distribution $P$ and any harness $H$, let $\tau^{H,B}\sim p_{\Agent}(\cdot\mid H,B,U)$ denote the resulting execution trajectory and let $r_B(\tau^{H,B},U)\in[0,1]$ denote the score assigned by benchmark $B$.

Suppose harness optimization runs for $T$ rounds. Let $H_0$ be the initial harness. At each round $t=1,\ldots,T$, a Proposer---for example, gpt-5.6-sol configured as a coding agent---observes the current harness, previous scores, and execution traces, edits $H_{t-1}$, and returns $H_t$. Let $D_s=\{U_i\}_{i=1}^{n_s}$ denote the search-task set. If $\tau_{i}^{H_t,B}$ is the rollout of $(\Agent,H_t)$ on $U_i$, define the empirical score $\widehat{R}_{B, D_s}(\Agent, H_t)$ and the population score $R_B(\Agent,H)$:
\begin{equation*}
    \widehat{R}_{B, D_s}(\Agent, H_t):=\frac{1}{n_s}\sum_{i=1}^{n_s}r_B(\tau_{i}^{H_t,B},U_i),\quad R_{B}(\Agent,H):=\mathbb{E}_{U\sim P}\mathbb{E}_{\tau^{H,B}\sim p_{\Agent}(\cdot\mid H,B,U)}[r_B(\tau^{H,B},U)].
\end{equation*}
The outer loop therefore selects $H_0,H_1,\ldots,H_T$ using the empirical scores. After breaking ties by a fixed rule, the search-optimal harness among the candidates is:
\begin{equation}
    H^\star_{\rm rel}\in\left\{H_t:\widehat{R}_{\Brel, D_s}(\Agent, H_t)=\max_{0\le j\le T}\widehat{R}_{\Brel, D_s}(\Agent, H_j)\right\}.
    \label{eq:selected-harness}
\end{equation}
The Proposer's direct objective is to increase $\widehat{R}_{\Brel, D_s}(\Agent, H_t)$ by modifying the harness, rather than to ensure that each task is solved through the intended capability. Because it can inspect prompts, memory, retrieval logic, code, tool calls, and execution traces, the resulting harness may achieve a higher released-benchmark score on \(D_s\) by exploiting a benchmark-wide shortcut.

If a shortcut mechanism is known, an ideal transformation may construct a neutralized benchmark $B_0:=\Phi_0(\Brel)$ through $\Phi_0(\Brel)=(P,V_0,Q_0,\psi)$, where $V_0$ and $Q_0$ are the shortcut-neutralized counterparts of $V_{\rm rel}$ and $Q_{\rm rel}$, respectively. A valid $\Phi_0$ removes the specified protocol correlation while leaving $U$, $Y=\psi(U)$, and the semantic-task distribution $P$ unchanged. Write $R_0:=R_{B_0}$ and $R_{\rm rel}:=R_{\Brel}$. Define $G_{\rm rel}(H;H_0)=R_{\rm rel}(\Agent,H)-R_{\rm rel}(\Agent,H_0)$ and $G_0(H;H_0)=R_0(\Agent,H)-R_0(\Agent,H_0)$. Then, we define $\Delta_{\rm BS}(H;H_0)$ with:
\begin{equation*}
    G_{\rm rel}(H;H_0)=G_0(H;H_0)+\Delta_{\rm BS}(H;H_0),
\end{equation*}
where $G_{\mathrm{rel}}\left(H;H_0\right)$ is the released gain during harness optimization, while $G_0\left(H;H_0\right)$ is the gain that survives shortcut neutralization. The remaining $\Delta_{\mathrm{BS}}\left(H;H_0\right)$ is the signed gain difference associated with the ``bad genius'' Proposer's changes in reliance on the benchmark-wide shortcut rather than improved task-solving capability.

{Task holdout, the usual guard against shortcuts, detects only task-specific shortcuts, which raise the gain on the search set above the gain on held-out tasks~\citep{wang2026rethinking}. A benchmark-wide shortcut raises both gains, because the search set and the held-out set share the protocol. Appendix~\ref{app:proofs} formalizes this distinction through the excess search-set gain $\Delta_{\rm TS}$.} Our goal is to achieve a substantial positive $G_{\rm rel}(H;H_0)$ while keeping $\Delta_{\rm BS}(H;H_0)$ below a small tolerance.

\subsection{From Counterfactuals to Harness Evolution}\label{sec:From Counterfactuals to Harness Evolution}

The ideal $B_0$ above defines the target estimand when all shortcut mechanisms are known. In practice, however, the complete set of shortcut mechanisms is unknown, so $B_0$ cannot be constructed. CHASE therefore replaces the single $B_0$ with a family of valid counterfactual benchmarks. {We use ``counterfactual'' in the interventional sense of counterfactual invariance and counterfactual task variants~\citep{veitch2021counterfactual, wu2024reasoning}: a counterfactual benchmark asks what the harness would score if the protocol were changed while every task stayed fixed, and we answer it by executing the transformed benchmark rather than by estimating it from observational data.} Index each protocol transformation by $b$. Each $\Phi_b$ acts on benchmark configurations. Its action on the released benchmark preserves the same semantic task and maps $B_b:=\Phi_b(\Brel)=(P,V_b,Q_b,\psi)$. Write $\Valid(\Phi_b)=1$ when the transformation preserves the semantic task, target, evidence, and scoring semantics while changing only the permitted components of the protocol. {The permitted components are those that solving the task does not require (Appendix~\ref{app:legitimate-vs-shortcut}).} The valid family is $\Bval=\{B_b=\Phi_b(\Brel):\Valid(\Phi_b)=1\}$, and the identity transformation includes $\Brel$ in $\Bval$.

\noindent{\bf Validity Firewall.} The benchmark-specific firewall sets $\Valid(\Phi_b)=1$ only when all executable checks comparing $\Brel$ with $\Phi_b(\Brel)$ pass. Otherwise, $\Phi_b$ is rejected. These checks include preserving the semantic task and ground-truth answer and restricting changes to the permitted components. The complete set of checks is given in Appendix~\ref{app:validity-firewall}.

For any candidate $H$ and $B_b\in\Bval$, let $R_b=R_{B_b}$ and define $G_b(H;H_0)=R_b(\Agent,H)-R_b(\Agent,H_0)$. Taking $B_b=\Brel$ or $B_b=B_0$ gives $G_{\rm rel}(H;H_0)$ or $G_0(H;H_0)$, respectively. The released gain destroyed by $B_b$ is $\Delta_b(H;H_0)=G_{\rm rel}(H;H_0)-G_b(H;H_0)$. When $B_b=B_0$, $\Delta_b(H;H_0)$ reduces to $\Delta_{\rm BS}(H;H_0)$, so the valid family $\Bval$ directly extends ideal neutralization $B_0$. Negative $\Delta_b$ means that the harness gain increases rather than decreases under the counterfactual. 

\looseness=-1 CHASE then alternates between a Proposer and a Challenger. Since $\Bval$ is an unknown infinite family, before round $t$, the Proposer has access to a finite archive $\mathcal A_{t-1}\subset\Bval$ containing $\Brel$ and the counterfactuals confirmed in earlier Challenger rounds. At the population level, the Proposer solves:
\begin{equation}
    H_t\in\argmax_H G_{\rm rel}(H;H_0)\quad\text{s.t.}\quad \Delta_b(H;H_0)\le\varepsilon,\quad \forall B_b\in\mathcal A_{t-1},
    \label{eq:proposer-update}
\end{equation}
where $\varepsilon\ge0$ controls how much of the released-benchmark gain may disappear under any archived counterfactuals. The Challenger targets the resulting $H_t$ by searching:
\begin{equation}
    B_{b_t}\in\argmax_{B_b\in\Bval}\Delta_b(H_t;H_0),
    \label{eq:challenger-update}
\end{equation}
and returns a typed, executable specification of $\Phi_{b_t}$. The resulting counterfactual $B_{b_t}$ is then considered for inclusion in $\mathcal A_t$, subject to the validity and confirmation checks below (Section~\ref{sec:finite-archive}).

Finally, the selection criterion in \eqref{eq:selected-harness} is inadequate for CHASE, since a high-scoring harness may violate counterfactual constraints discovered in later rounds. After updating the archive, we therefore select the final harness:
\begin{equation}
    H^\star\in\argmax_{H\in\{H_0,H_1,\ldots,H_T\}}G_{\rm rel}(H;H_0)
    \quad\text{s.t.}\quad \Delta_b(H;H_0)\le\varepsilon,\quad \forall B_b\in\mathcal A_T.
    \label{eq:final-selection}
\end{equation}

\begin{figure}[tbp]
  \centering
  \includegraphics[width=\textwidth]{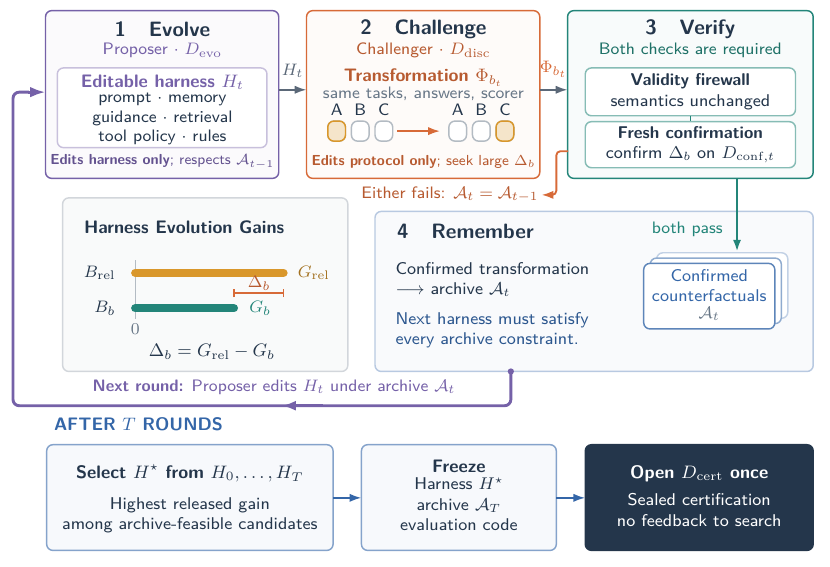}
  \caption{An overview of \textbf{CHASE}. {Each round has four steps. Evolve: the Proposer edits only the harness on $D_{\rm evo}$ and must keep its gain under every archived counterfactual in $\mathcal A_{t-1}$. Challenge: the Challenger edits only the protocol and proposes a transformation $\Phi_{b_t}$ on $D_{\rm disc}$ that destroys the gain of $H_t$ while tasks, answers, and scoring stay fixed. Verify: the validity firewall checks the transformation, and the gain destruction is confirmed on the fresh tasks $D_{{\rm conf},t}$. Remember: a confirmed transformation is appended to the archive $\mathcal A_t$, which constrains the Proposer in the next round. After $T$ rounds, the final harness is selected among archive-feasible candidates, frozen, and certified once on the sealed $D_{\rm cert}$.}}
  \label{fig:CHASE}
\end{figure}

\subsection{Finite Archive and Confirmation}
\label{sec:finite-archive}
The objectives in equations~\ref{eq:proposer-update}--\ref{eq:final-selection} above are written at the population level. The implementation and the analysis below use their empirical counterparts. For any finite task set $D$, define $\widehat R_{B,D}(\Agent,H):=\frac{1}{|D|}\sum_{U_i\in D}r_B(\tau_i^{H,B},U_i), \widehat G_{b,D}(H;H_0):=\widehat R_{B_b,D}(\Agent,H)-\widehat R_{B_b,D}(\Agent,H_0)$, and $\widehat\Delta_{b,D}(H;H_0):=\widehat G_{{\rm rel},D}(H;H_0)-\widehat G_{b,D}(H;H_0)$. When a task is evaluated with multiple rollouts, the rollout scores are first averaged within task. To separate search, confirmation, and final certification, the evaluation tasks are assigned disjoint roles:
\begin{equation*}
    D_{\rm evo},\ D_{\rm disc},\ D_{\rm conf,1},\ldots,D_{\rm conf,T},\ D_{\rm cert}\quad\text{are mutually disjoint}.
\end{equation*}
\looseness=-1 Here $D_{\rm evo}$ supplies the feedback used by the Proposer to search $H_t$ in \eqref{eq:proposer-update}, $D_{\rm disc}$ supplies the feedback used by the Challenger to search for $\Phi_{b_t}$ in \eqref{eq:challenger-update}, $D_{\rm{conf},t}$ is used to confirm $\Phi_{b_t}$ after it has been fixed, and $D_{\rm cert}$ remains sealed until final certification. This separation reserves fresh tasks for confirmation and certification, extending the search--test distinction {of task holdout (Section~\ref{sec:problem-setup})}.

Since $\Bval$ is unknown and cannot be exhaustively searched, CHASE maintains the sequence of finite archives $\{\mathcal A_t\}_{t=0}^T$. Initialize $\mathcal A_0=\{\Brel\}$. Once the Challenger proposes $\Phi_{b_t}$, its executable code is fixed and $\Delta_{b_t}(H_t;H_0)$ is re-estimated on $D_{\rm{conf},t}$. For a confirmation threshold $\eta_{\rm{conf},t}>0$, let $\mathsf{Conf}_t:=\mathbf 1\!\left\{\widehat\Delta_{b_t,D_{\rm{conf},t}}(H_t;H_0)\ge\eta_{\rm{conf},t}\right\}$. The archive is then updated by:
\begin{equation}
    \mathcal A_t=
    \begin{cases}
        \mathcal A_{t-1}\cup\{B_{b_t}\}, & \Valid(\Phi_{b_t})=1\ \text{and}\ \mathsf{Conf}_t=1,\\
        \mathcal A_{t-1}, & \text{otherwise}.
    \end{cases}
    \label{eq:archive-update}
\end{equation}
In practice, each archived $B_{b_t}$ is stored together with the executable code for $\Phi_{b_t}$. Thus a Challenger proposal becomes a constraint on subsequent harness evolution only when it is valid and its effect is confirmed on $D_{\rm{conf},t}$.

To summarize performance over any finite counterfactual archive $\mathcal A\subset\Bval$, define:
\[
    \Gamma_{\mathcal A}(H;H_0):=\max_{B_b\in\mathcal A}\Delta_b(H;H_0),\qquad G_{\min,\mathcal A}(H;H_0):=\min_{B_b\in\mathcal A}G_b(H;H_0),
\]
with empirical counterparts on a task set $D$,
\[
    \widehat{\Gamma}_{\mathcal A,D}(H;H_0):=\max_{B_b\in\mathcal A}\widehat{\Delta}_{b,D}(H;H_0),\qquad \widehat{G}_{\min,\mathcal A,D}(H;H_0):=\min_{B_b\in\mathcal A}\widehat{G}_{b,D}(H;H_0).
\]
Here, $\Gamma_{\mathcal A}(H;H_0)$ is the largest gain destruction, $G_{\min,\mathcal A}(H;H_0)$ is the smallest surviving gain over $\mathcal A$, and we have $G_{\mathrm{min}, \mathcal{A}}(H;H_0)+\Gamma_{\mathcal{A}}(H;H_0)=G_{\mathrm{rel}}(H;H_0)$. {Figure~\ref{fig:CHASE} summarizes CHASE.}

\section{Theoretical Guarantees}
\label{sec:theory}
Theoretical analyses of harness evolution remain limited. Recent work studies how to determine from finite evaluation data whether a harness update improves performance without degrading prior behavior under a fixed task distribution~\citep{cai2026safeharness}. It does not model protocol correlation or counterfactual discovery. We therefore ask what a finite counterfactual archive can certify about $B_0$ and, when it does not yet recover $B_0$, what can be concluded from sequential Challenger search.

For valid $\Phi_b,\Phi_{b'}$, define $B_{b'\circ b}:=\Phi_{b'}(\Phi_b(\Brel))$ by applying $\Phi_b$ first and $\Phi_{b'}$ second. This composition maps $(P,V_{\rm rel},Q_{\rm rel},\psi)$ to $(P,V_{b'\circ b},Q_{b'\circ b},\psi)$, and we assume $\Bval$ is closed under it. For nonempty finite $\mathcal A\subset\Bval$, define:
\[
    K(\mathcal A):=\min\left\{k\in\mathbb N^{+}:\,\exists\,B_{b_1},\ldots,B_{b_k}\in\mathcal A,B_{b_k\circ\cdots\circ b_1}=B_0\right\}.
\]
Thus $K(\mathcal A)=\infty$ if no composition of transformations in $\mathcal A$ produces $B_0$. Also define:
\[
    \rho:=\sup_{H}\sup_{B_b,B_{b'}\in\Bval}\left[\Delta_{b'\circ b}(H;H_0)-\Delta_b(H;H_0)-\Delta_{b'}(H;H_0)\right]_+.
\]
Thus $K(\mathcal A)$ is the smallest number of archived transformations whose composition produces $B_0$, while $\rho$ measures the worst-case excess gain destruction under composition. Equivalently, $\Delta_{b'\circ b}(H;H_0)\le \Delta_b(H;H_0)+\Delta_{b'}(H;H_0)+\rho$. Fix $\alpha\in(0,1)$. We obtain the following results.
\begin{theorem}
    \label{theorem:B_0}
    Let $\mathcal A\subset\Bval$ be nonempty and finite. If $\mathcal A$ and $H^\star$ are fixed before $D_{\rm cert}$ is opened and $K(\mathcal A)<\infty$, then, under the conditions in Appendix~\ref{app:proofs}, with probability at least $1-\alpha/2$:
    \begin{equation*}
        \begin{aligned}
            \Delta_{\rm BS}(H^\star;H_0)
            &\le K(\mathcal A)\left\{\widehat{\Gamma}_{\mathcal A,D_{\rm cert}}(H^\star;H_0)+\sqrt{8\log(4|\mathcal A|/\alpha)/|D_{\rm cert}|}\right\}+(K(\mathcal A)-1)\rho,\\
            G_0(H^\star;H_0)
            &\ge \widehat{G}_{\min,\mathcal A,D_{\rm cert}}(H^\star;H_0)-\sqrt{2\log(4|\mathcal A|/\alpha)/|D_{\rm cert}|}\\
            &\qquad -(K(\mathcal A)-1)\left\{\widehat{\Gamma}_{\mathcal A,D_{\rm cert}}(H^\star;H_0)+\sqrt{8\log(4|\mathcal A|/\alpha)/|D_{\rm cert}|}+\rho\right\}.
        \end{aligned}
    \end{equation*}
\end{theorem}
Theorem~\ref{theorem:B_0} converts a finite archive into guarantees for $B_0$. The bounds tighten as $K(\mathcal A)$ and $\rho$ decrease. If $B_0\in\mathcal A$, then $K(\mathcal A)=1$ and every term containing $\rho$ disappears. 

For each search set $D\in\{D_{\rm evo},D_{\rm disc}\}$, let $\mathfrak R_D$ denote the Rademacher complexity term defined in Appendix~\ref{app:uniform-evolution}, which typically decreases as $|D|$ grows. It controls uniform estimation of $R_{\rm rel}(\Agent,H)$ over harnesses $H$ and of $\Delta_b(H;H_0)$ over harness--counterfactual pairs $(H,B_b)$. We use $r=\max_{D\in\{D_{\rm evo},D_{\rm disc}\}}\left\{2\mathfrak R_D+\sqrt{8\log(8/\alpha)/|D|}\right\}$ which bounds the empirical-to-population deviation. For confirmation, set $\eta_{{\rm conf},t}=\varepsilon+\gamma+\sqrt{8\log(8T/\alpha)/|D_{{\rm conf},t}|}, 0<\gamma\le1$. Appendix~\ref{app:proofs} discusses $\mathfrak R_D$ and $\gamma$.

\begin{theorem}
    \label{theorem:Conf_t=1}
    Suppose $\operatorname{Valid}(\Phi_{b_t})=1$ and $\widehat{\Delta}_{b,D_{\rm evo}}(H_t;H_0)\leq\varepsilon$ for every $B_b\in \mathcal A_{t-1}$ at each round. Assume that, for every $0<q\le1$, any $m$ harnesses with $\sup_{B_b\in\Bval}|\Delta_b(H^{(j)};H_0)-\Delta_b(H^{(k)};H_0)|\ge q$, for all $j\ne k$, obey $m\le(C/q)^d$, for constants $C\ge1$ and $d>0$. If $r<\gamma$, then, with probability at least $1-\alpha/2$,
    \[
        \#\{t\le T:\mathsf{Conf}_t=1\}\le\left(\frac{C}{\gamma-r}\right)^d.
    \]
    Consequently, if $T>\left(\frac{C}{\gamma-r}\right)^d$, then $\mathsf{Conf}_t=0$ for at least one $t\le T$.
\end{theorem}

Theorem~\ref{theorem:Conf_t=1} shows that, under the packing condition, only finitely many counterfactuals can be confirmed. In particular, for a sufficiently large $T$, the process must reach a round with $\mathsf{Conf}_t=0$. We next characterize what can be concluded at such a round.

\begin{theorem}
    \label{theorem:Conf_t=0}
    Suppose $H_t$ and $B_{b_t}$ solve the empirical counterparts of \eqref{eq:proposer-update} and \eqref{eq:challenger-update}. On the common evolution event defined in Lemma~\ref{lem:uniform-evolution}, which has probability at least $1-\alpha/2$, every round $t$ for which $\Valid(\Phi_{b_t})=1$ and $\mathsf{Conf}_t=0$ satisfies:
    \[
        \begin{aligned}
            \sup_{B_b\in\Bval}\Delta_b(H_t;H_0)
            &<\varepsilon+\gamma+2r+2\sqrt{8\log(8T/\alpha)/|D_{{\rm conf},t}|},\\
            G_{\rm rel}(H_t;H_0)
            &\ge\sup_{\substack{\widetilde H:\ \Delta_b(\widetilde H;H_0)\le\varepsilon-r, \forall B_b\in\Bval}}G_{\rm rel}(\widetilde H;H_0)-2r.
        \end{aligned}
    \]
\end{theorem}
Under these conditions, Theorem~\ref{theorem:Conf_t=0} gives the complementary conclusion: a non-confirmed round certifies small gain destruction over the valid family $\Bval$, up to estimation error, while retaining near-optimal released gain among harnesses satisfying the constraints. Thus, confirmed rounds expand the archive, whereas a non-confirmed round suggests stopping with the current harness.

\section{Experiments}
\label{sec:experiments}
\raggedbottom

\subsection{Benchmarks and experimental setup}
\label{sec:officeqa-setup}

OfficeQA is a benchmark suite for question answering over U.S. government financial documents~\citep{opsahlong2026officeqa}. Its questions require document discovery, text and table retrieval, numerical reasoning, and exact answer extraction. We use two OfficeQA releases: Full, which contains 246 questions{\footnote{{OfficeQA Full: \url{https://huggingface.co/datasets/databricks/officeqa}}}}, and Pro V2, which contains 90 questions over a separate receipts-and-expenditures corpus~\citep{databricks2026officeqa}{\footnote{{OfficeQA Pro V2: \url{https://huggingface.co/datasets/databricks/officeqa-pro-v2}}}}. We use Full for harness optimization and Pro V2 for cross-corpus evaluation. OfficeQA is sensitive to harness design because search, evidence handling, tool use, computation, and answer formatting are harness-controlled. {Every method edits the same harness components: the system prompt, the long-term memory, the retrieval policy, the tool policy, the policy graph (control rules over query and budget features), and the guidance modules (conditional guidance). The available tool operations, the scorer, and the budget stay fixed (Appendix~\ref{app:agent-harnesses}).} At a fixed model, {\it EnvHarness} improved OfficeQA exact-match accuracy from $54.40\%$ to $56.20\%$ over its original-environment baseline~\citep{huang2026envharness}. We retain all released questions and answers and let the foundation model $\Agent$ search the full 697-document transformed-text corpus for OfficeQA Full. Appendix~\ref{app:benchmark-details} gives release, corpus, and scoring details.

Taking $T=3$, we divide the 246 questions into \(D_{\rm evo}\) (49 questions), \(D_{\rm disc}\) (49), \(D_{\rm cert}\) (76), and three round-specific confirmation task sets \(D_{{\rm conf},t}\), each containing 24 questions. Pro V2 is reported separately as a cross-corpus analysis. Our use of task splits follows prior OfficeQA evaluations \citep{alzubi2026evoskill, ursekar2026harnessopt}. 

To isolate benchmark-wide shortcuts in this type of multi-document numerical reasoning, we also construct Syn-Ledger with 320 synthetic ledger tasks. Controlled shortcuts allow us to construct $B_0$ and directly measure how much harness gain survives their neutralization (Appendix~\ref{app:synledger}). {Finally, FinanceBench~\citep{islam2023financebench}\footnote{{FinanceBench: \url{https://github.com/patronus-ai/financebench}}} and SWE-Bench Pro~\citep{dan2026swebenchprov2}\footnote{{SWE-Bench Pro V2: \url{https://github.com/scaleapi/SWE-bench_Pro-os}}} test whether the comparison carries over to company filings, coding agents, and other agent models (Appendices~\ref{app:financebench} and~\ref{app:swebench}).}

\subsection{Methods and evaluation}
\label{sec:experimental-methods-evaluation}

On OfficeQA, we compare five methods. \textbf{\methodraw{}} keeps the initial harness $H_0$ unchanged and provides the reference for all gain and gain-destruction metrics. \textbf{\methodmeta{}} optimizes $H_0$ for released-benchmark performance~\citep{lee2026metaharness}. \textbf{\methodhc{}} optimizes $H_0$ for released-benchmark performance under the fixed generalization gate of HarnessCompass~\citep{zhang2026harnesscompass}. \textbf{\methodevolve{}} uses reference trajectories and quality, performance, and held-out validation gates to select harness updates~\citep{jiang2026harnessevolve}. \textbf{\methodchase{}} uses HarnessCompass's Proposer backbone without the fixed generalization gate. The Challenger supplies confirmed counterfactuals that enter $\mathcal A_t$ and constrain subsequent Proposer rounds and final harness selection. Syn-Ledger includes these five methods and adds \textbf{$B_0$-Access} with access to $B_0$. It uses HarnessCompass's three-round schedule and fixed generalization gate, but selects candidates by their $B_0$ performance. Within each benchmark, all methods share $\Agent$, $H_0$. All optimized methods except HarnessEvolve (Appendix~\ref{app:harnessevolve-adaptation}) use three Proposer rounds. {Under our budget, we fix $\eta_{{\rm conf},t}$ and $\varepsilon$ (Appendix~\ref{app:implementation}; calibration in Section~\ref{sec:theory}).}

On OfficeQA, the primary evaluation applies each method's final harness $H$ to the 76 certification questions under every benchmark in the final CHASE archive $\mathcal A_3$. We report the released-benchmark score $\widehat R_{\Brel,D_{\rm cert}}(\Agent,H)$, the average and worst-case scores over $\mathcal A_3$. For a finite archive $\mathcal A$, define $\widehat R_{{\rm avg},\mathcal A,D_{\rm cert}}(\Agent,H):=|\mathcal A|^{-1}\sum_{B_b\in\mathcal A}\widehat R_{B_b,D_{\rm cert}}(\Agent,H)$ and $\widehat R_{\min,\mathcal A,D_{\rm cert}}(\Agent,H):=\min_{B_b\in\mathcal A}\widehat R_{B_b,D_{\rm cert}}(\Agent,H)$. For brevity, we write the three certification metrics as $\widehat R_{\rm rel}$, $\widehat R_{{\rm avg},\mathcal A_3}$ and $\widehat R_{\min,\mathcal A_3}$. Separately, $\widehat R_{\rm ProV2}$ denotes the released-benchmark score on the 90-question Pro V2 release. {Appendix~\ref{app:benchmark-details} gives details, and Appendix~\ref{app:resource-use} compares token use.}

On Syn-Ledger, we evaluate all six methods on the 208 certification tasks under $\Brel$ and $B_0$. We report $\widehat R_{\rm rel}$, $\widehat R_0$, $\widehat G_{\rm rel}$, $\widehat G_0$, and $\widehat\Delta_{\rm BS}$, with the certification-set and harness arguments suppressed. Appendix~\ref{app:synledger} describes the task construction, shortcut generation, scoring, and task allocation.

\subsection{Results}
\label{sec:experimental-results}

{The supplementary code includes the protocol transformations, the validity firewall, the Syn-Ledger generator, the final harnesses, and evaluation scores.} Table~\ref{tab:main-results} reports OfficeQA results on Pro V2 and Full certification results. On $D_{\rm cert}$, \methodchase{} achieves a competitive released-benchmark score and the highest average and worst-case scores over the final archive $\mathcal A_3$. {The released-score gains of \methodchase{} and \methodevolve{} over \methodraw{} stay below 1.5\%, comparable to that of EnvHarness (Section~\ref{sec:officeqa-setup}), whereas the archive scores reorder the two: \methodchase{} leads under both, since it rejects the released gains that the archived counterfactual destroys.} In contrast, \methodhc{} performs below the initial harness, suggesting that its fixed generalization gate does not ensure improved held-out performance on OfficeQA. {\methodhc{} was designed for SWE-bench Verified~\citep{openai2024sweverified}. Appendix~\ref{app:Generalization-Gate} details our implementation of its gate.} Notably, the advantage of \methodchase{} carries over to the Pro V2 corpus, where it scores 30.37\%. {Paired bootstrap intervals (Table~\ref{tab:main-results}; Appendix~\ref{app:officeqa-revision}) show that the advantage of \methodchase{} over \methodevolve{}, the strongest comparator on $D_{\rm cert}$, excludes zero on Pro V2, as do its advantages over \methodhc{} in both $\widehat R_{{\rm avg},\mathcal A_3}$ and $\widehat R_{\min,\mathcal A_3}$ on the 76 sealed questions of $D_{\rm cert}$ with three rollouts per question.}

\begin{table}[htbp]
\centering
\caption{OfficeQA results. {Pro V2 (90 questions, separate corpus) plays no role in optimization. The certification columns use the 76 sealed questions of $D_{\rm cert}$. Bold marks the best score in each column.} All evaluations use three rollouts per question. {All scores are in \%. Small brackets give question-level paired bootstrap 95\% confidence intervals (differences in Appendix~\ref{app:officeqa-revision}).}}
\label{tab:main-results}
\small
\setlength{\tabcolsep}{0pt}
\begin{tabular*}{\textwidth}{@{\extracolsep{\fill}}l r@{\extracolsep{0pt}\hspace{3pt}}l@{\extracolsep{\fill}} r@{\extracolsep{0pt}\hspace{3pt}}l@{\extracolsep{\fill}} r@{\extracolsep{0pt}\hspace{3pt}}l@{\extracolsep{\fill}} r@{\extracolsep{0pt}\hspace{3pt}}l@{}}
\toprule
& \multicolumn{2}{c}{OfficeQA Pro V2} & \multicolumn{6}{c}{Certification of OfficeQA Full} \\
\cmidrule(lr){2-3}\cmidrule(l){4-9}
Method & \multicolumn{2}{c}{$\widehat R_{\rm ProV2}(\uparrow)$} & \multicolumn{2}{c}{$\widehat R_{\rm rel}(\uparrow)$} & \multicolumn{2}{c}{$\widehat R_{{\rm avg},\mathcal A_3}(\uparrow)$} & \multicolumn{2}{c}{$\widehat R_{\min,\mathcal A_3}(\uparrow)$} \\
\midrule
\methodraw{} & 27.04 & \ci{19.26}{35.19} & 67.98 & \ci{59.21}{76.75} & 66.23 & \ci{57.24}{74.78} & 64.47 & \ci{54.82}{73.68} \\
\methodmeta{} & 29.26 & \ci{21.48}{37.41} & 63.60 & \ci{54.39}{72.81} & 64.04 & \ci{54.82}{73.03} & 63.60 & \ci{53.95}{71.93} \\
\methodhc{} & 26.30 & \ci{18.52}{34.81} & 64.04 & \ci{54.39}{73.68} & 63.16 & \ci{53.73}{72.37} & 62.28 & \ci{52.19}{71.05} \\
\methodevolve{} & 24.07 & \ci{16.67}{32.22} & {\bf 69.30} & \ci{59.65}{78.51} & 68.20 & \ci{58.77}{77.19} & 67.11 & \ci{57.02}{76.32} \\
\methodchase{} & {\bf 30.37} & \ci{22.59}{38.52} & 68.86 & \ci{59.21}{78.07} & {\bf 68.42} & \ci{59.43}{77.19} & {\bf 67.98} & \ci{58.33}{76.32} \\
\bottomrule
\end{tabular*}
\end{table}

The first-round Challenger proposes collecting a table's associated context before the table and supplies an executable transformation specification. It hypothesizes that the evolved harness $H_1$, instructed to ``keep an explicit unit for every operand,'' may rely on customary locations of units and notes. Figure~\ref{fig:officeqa-counterexample} illustrates a schematic visualization of the proposal. The host executes the generated specification by re-encoding retrieved text as table context in the agent-visible JSON search results (Appendix~\ref{app:officeqa-additional-results}). On $D_{\rm evo}$, $H_1$'s gain over $H_0$ decreases from $8.16\%$ under $\Brel$ to $-5.10\%$ under $B_{b_1}$. {This reversal is evidence that $H_1$'s released-benchmark gain is sensitive to the format in which retrieval results are returned.} Once $B_{b_1}$ enters the archive, $H_1$ and all second-round Proposer candidates violate the $\varepsilon=0.05$ constraint, so CHASE sets $H_2=H_0$. In the third round, one candidate recovers a $4.08\%$ released gain while satisfying the constraint and is selected as $H_3$.

\begin{figure}[htbp]
\centering
\includegraphics[width=\linewidth]{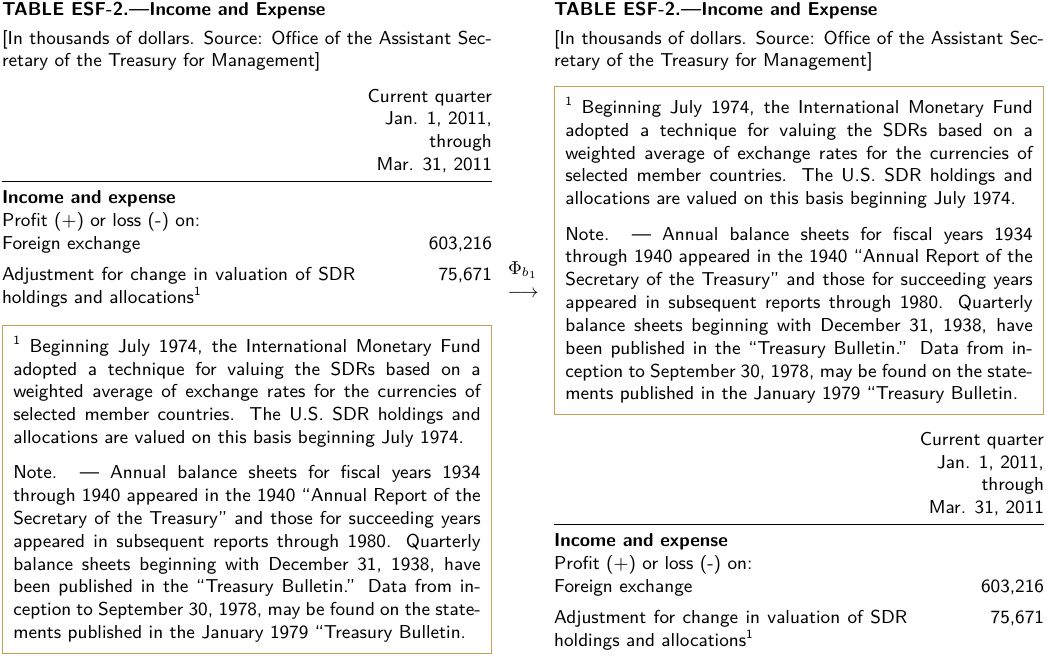}
\caption{{The Challenger's first-round proposal on a real OfficeQA excerpt (retypeset for readability) collects the table's unit and note blocks before the table. The question, the answer, and every table value stay fixed. Appendix~\ref{app:officeqa-additional-results} gives the executed specification, which re-encodes the retrieved text as table context in the search output, and Appendix~\ref{app:officeqa-revision} compares the two.}}
\label{fig:officeqa-counterexample}
\end{figure}

We next turn to Syn-Ledger. Table~\ref{tab:synledger-results} shows that, although \methodchase{} trails \methodevolve{} on $\Brel$, it achieves the highest score and gain under $B_0$. Its gain is larger under $B_0$ than under the released benchmark, yielding a negative $\widehat\Delta_{\rm BS}$. {Paired bootstrap intervals (Table~\ref{tab:synledger-results}; Appendix~\ref{app:synledger-revision}) show that the advantage of \methodchase{} in $\widehat R_0$ excludes zero against every baseline, and its lower $\widehat\Delta_{\rm BS}$ excludes zero against every baseline without access to $B_0$.}

\begin{table}[htbp]
\centering
\caption{Syn-Ledger results {on the 208 certification tasks of $D_{\rm cert}$}. {$B_0$ is the shortcut-neutralized benchmark. Gains are over \methodraw{}, and $\widehat\Delta_{\rm BS}=\widehat G_{\rm rel}-\widehat G_0$ is better when lower. Bold marks the best value in each column.} All final evaluations use three rollouts per task. {All scores are in \%. Small brackets give task-level paired bootstrap 95\% confidence intervals (differences in Appendix~\ref{app:synledger-revision}).}}
\label{tab:synledger-results}
\footnotesize
\renewcommand{\ci}[2]{{\tiny[#1,\,#2]}}
\setlength{\tabcolsep}{0pt}
\begin{tabular*}{\textwidth}{@{\extracolsep{\fill}}l r@{\extracolsep{0pt}\hspace{2pt}}l@{\extracolsep{\fill}} r@{\extracolsep{0pt}\hspace{2pt}}l@{\extracolsep{\fill}} r@{\extracolsep{0pt}\hspace{2pt}}l@{\extracolsep{\fill}} r@{\extracolsep{0pt}\hspace{2pt}}l@{\extracolsep{\fill}} r@{\extracolsep{0pt}\hspace{2pt}}l@{}}
\toprule
Method & \multicolumn{2}{c}{$\widehat R_{\rm rel}\,(\uparrow)$} & \multicolumn{2}{c}{$\widehat R_0\,(\uparrow)$} & \multicolumn{2}{c}{$\widehat G_{\rm rel}\,(\uparrow)$} & \multicolumn{2}{c}{$\widehat G_0\,(\uparrow)$} & \multicolumn{2}{c}{$\widehat\Delta_{\rm BS}\,(\downarrow)$} \\
\midrule
\methodraw{} & 81.89 & \ci{77.72}{85.74} & 15.22 & \ci{11.70}{18.91} & \multicolumn{2}{c}{--} & \multicolumn{2}{c}{--} & \multicolumn{2}{c}{--} \\
\methodmeta{} & 82.05 & \ci{77.56}{86.22} & 19.87 & \ci{15.71}{24.20} & +0.16 & \ci{$-$4.33}{4.49} & +4.65 & \ci{0.16}{9.13} & $-$4.49 & \ci{$-$10.42}{1.44} \\
\methodhc{} & 80.45 & \ci{75.64}{84.94} & 14.74 & \ci{11.06}{18.75} & $-$1.44 & \ci{$-$6.89}{3.85} & $-$0.48 & \ci{$-$5.45}{4.33} & $-$0.96 & \ci{$-$7.85}{5.93} \\
\methodevolve{} & \textbf{93.59} & \ci{90.38}{96.31} & 18.27 & \ci{14.42}{22.44} & \textbf{+11.70} & \ci{7.53}{15.87} & +3.04 & \ci{$-$1.44}{7.69} & +8.65 & \ci{2.40}{14.90} \\
$B_0$-Access & 83.01 & \ci{78.69}{87.18} & 32.69 & \ci{27.56}{37.98} & +1.12 & \ci{$-$3.37}{5.45} & +17.47 & \ci{12.34}{22.76} & $-$16.35 & \ci{$-$22.60}{$-$10.10} \\
\methodchase{} & 87.18 & \ci{83.33}{90.71} & \textbf{38.30} & \ci{32.69}{43.91} & +5.29 & \ci{0.80}{9.78} & \textbf{+23.08} & \ci{17.47}{28.69} & \textbf{{\boldmath$-$}17.79} & \ci{$-$23.88}{$-$11.54} \\
\bottomrule
\end{tabular*}
\end{table}

{On FinanceBench, \methodchase{} attains the best certification score on each evaluation metric (Appendix~\ref{app:financebench}). Table~\ref{tab:swebench-results} reports SWE-Bench Pro, where \methodchase{} again attains the best score on each evaluation metric and improves on \methodraw{} by 11.43 percentage points (paired 95\% interval $[1.90,22.86]$) under $\Brel$. Under the archived counterfactual, the gain of \methodevolve{} over \methodraw{} falls from $8.57\%$ to $3.81\%$, whereas the gain of \methodchase{} stays at $10.48\%$ (Appendix~\ref{app:swebench}).}

\begin{table}[htbp]
\centering
\caption{SWE-Bench Pro results on the 21 certification tasks of $D_{\rm cert}$, five rollouts each. The archive $\mathcal A$ contains $\Brel$ and one confirmed counterfactual. Bold marks the best score in each column. All scores are in \%, with task-level paired bootstrap 95\% confidence intervals in brackets.}
\label{tab:swebench-results}
\small
\setlength{\tabcolsep}{6pt}
\begin{tabular}{@{}lccc@{}}
\toprule
Method & $\widehat R_{\rm rel}(\uparrow)$ & $\widehat R_{{\rm avg},\mathcal A}(\uparrow)$ & $\widehat R_{\min,\mathcal A}(\uparrow)$ \\
\midrule
\methodraw{} & 56.19 \ci{39.05}{73.33} & 56.67 \ci{39.52}{73.33} & 56.19 \ci{37.14}{71.43} \\
\methodevolve{} & 64.76 \ci{48.57}{80.00} & 62.86 \ci{46.67}{78.57} & 60.95 \ci{43.81}{76.19} \\
\methodchase{} & {\bf 67.62} \ci{53.33}{80.95} & {\bf 67.62} \ci{51.90}{81.90} & {\bf 67.62} \ci{50.48}{80.95} \\
\bottomrule
\end{tabular}
\end{table}

\section{Discussion}

Harness optimization changes what must generalize. Unlike parameter optimization, where shortcuts typically arise from data or task artifacts~\citep{geirhos2020shortcut, jiang2026ladder}, harness evolution searches over executable programs. As a result, harness optimization can introduce benchmark-wide shortcuts. CHASE makes this distinction explicit: task generalization asks whether a final harness works across tasks, whereas CHASE asks whether the harness-evolution gain persists under valid $B_b\in\Bval${, in addition to its size on $\Brel$}. {The archive also serves evaluation without changing the benchmark. $\Brel$ stays fixed, and each archived counterfactual is an executable transformation of the same tasks, so it adds a column to a results table rather than replacing one: any system, old or new, can be scored on an archived counterfactual with one evaluation pass. A benchmark maintainer can publish the frozen archive alongside $\Brel$, and a harness developer can certify a frozen harness on sealed tasks under $\Brel$ and the archive, as in Table~\ref{tab:main-results} (Appendix~\ref{app:archive-evaluation}).}

\section*{Reproducibility Statement}
Appendices~\ref{app:proofs}--{\ref{app:swebench}} provide complete proofs, implementation details, OfficeQA experimental procedures, the construction of Syn-Ledger{, the FinanceBench experiments, and the SWE-Bench Pro experiments}. The supplementary code contains the executable protocol transformations, validity checks, archive records, and scripts for reproducing the reported results from task-level outputs.

\bibliography{references}
\bibliographystyle{tmlr}

\Needspace{10\baselineskip}
\section*{Supplementary Material}
\appendix

\renewcommand{\theHfigure}{app.\thesection.\arabic{figure}}
\renewcommand{\theHtable}{app.\thesection.\arabic{table}}
\renewcommand{\theHequation}{app.\thesection.\arabic{equation}}

Section~\ref{app:proofs} contains the proofs, and Section~\ref{app:implementation} describes the method implementations. Sections~\ref{app:benchmark-details} and~\ref{app:synledger} give the OfficeQA experimental details and the Syn-Ledger benchmark construction and evaluation, respectively. {Section~\ref{app:financebench} reports the additional experiments on FinanceBench. {Section~\ref{app:swebench} reports the experiments on SWE-Bench Pro V2 HARD-51.}}

\renewcommand{\theequation}{A.\arabic{equation}}
\renewcommand{\thetable}{A.\arabic{table}}
\renewcommand{\thefigure}{A.\arabic{figure}}
\renewcommand{\thecondition}{S\arabic{condition}}
\setcounter{equation}{0} 
\setcounter{table}{0} 
\setcounter{figure}{0}

\section{Theory and Proofs}
\label{app:proofs}

Section~A.1 establishes the common evolution event. Sections~\ref{app:proof-theorem1}--\ref{app:proof-theorem3} prove Theorems~\ref{theorem:B_0}--\ref{theorem:Conf_t=0}, respectively, covering finite-archive certification, the number of confirmed rounds, and guarantees for a round with \(\mathsf{Conf}_t=0\).

\paragraph{{Task-specific versus benchmark-wide shortcuts.}}
{Prior work commonly uses task holdout to guard against shortcuts, but it addresses only task-specific shortcuts~\citep{wang2026rethinking}. Let $D_h$ be a test-task set disjoint from the search-task set $D_s$ of Section~\ref{sec:problem-setup}, and let $R_{B,D}(\Agent,H):=\mathbb{E}_{U\sim \widehat{P}_D}\mathbb{E}_{\tau^{H,B}\sim p_{\Agent}(\cdot\mid H,B,U)}[r_B(\tau,U)]$ denote the expected average score on a finite task set $D$, where $P$ is replaced by the empirical distribution $\widehat{P}_D$. The excess search-set gain is}
\begin{equation*}
    {\Delta_{\rm TS}(H;H_0)=\{R_{\Brel,D_s}(\Agent,H)-R_{\Brel,D_s}(\Agent,H_0)\}-\{R_{\Brel,D_h}(\Agent,H)-R_{\Brel,D_h}(\Agent,H_0)\}.}
\end{equation*}
{A task-specific shortcut can make $\Delta_{\rm TS}(H;H_0)>0$. A benchmark-wide shortcut can persist in both $D_s$ and $D_h$ because both use the same $(V_{\rm rel},Q_{\rm rel})$. Therefore $\Delta_{\rm TS}(H;H_0)\approx0$ does not imply $\Delta_{\rm BS}(H;H_0)\approx0$.}

\begin{condition}
\label{cond:sampling}
All task sets have positive sizes fixed before they are opened. For each $D\in\{D_{\rm evo},D_{\rm disc},D_{{\rm conf},1},\ldots,D_{{\rm conf},T},D_{\rm cert}\}$, let $\mathcal F_D^-$ denote the information available before $D$ is opened. The tasks $U_i\in D$ satisfy $U_i\stackrel{\rm iid}{\sim}P$. Conditional on $\mathcal F_D^-$, their task-level evaluation vectors are independent, with each vector containing all rollout scores for one task. For every harness--benchmark pair $(H,B)$ evaluated on \(D\),
\[
    0\le r_B(\tau_i^{H,B},U_i)\le1,\qquad \mathbb E\!\left[r_B(\tau_i^{H,B},U_i)\mid\mathcal F_D^-,U_i\right]=\mathbb E_{\tau\sim p_{\Agent}(\cdot\mid H,B,U_i)}[r_B(\tau,U_i)].
\]
For $D_{\rm evo}$ and $D_{\rm disc}$, the classes of harnesses and harness--counterfactual pairs over which the displayed suprema are taken are fixed before the corresponding task set is observed. Moreover, $(H_t,b_t)$ is $\mathcal F_{D_{{\rm conf},t}}^-$-measurable, and $(H^\star,\mathcal A)$ is $\mathcal F_{D_{\rm cert}}^-$-measurable. Dependence among evaluations of the same task is unrestricted.
\end{condition}
Condition~\ref{cond:sampling} is only a technical condition for the concentration analysis. It allows within-task dependence and evolution across rounds.

\subsection{Auxiliary Lemma and its Proof}
\label{app:uniform-evolution}

For $U_i\in D$ and $B_b\in\Bval$, write the task-level gain destruction as:
\[
    \delta_{b,i}(H):=r_{\Brel}(\tau_i^{H,\Brel},U_i)-r_{\Brel}(\tau_i^{H_0,\Brel},U_i)-r_{B_b}(\tau_i^{H,B_b},U_i)+r_{B_b}(\tau_i^{H_0,B_b},U_i)\in[-2,2].
\]
For $D\in\{D_{\rm evo},D_{\rm disc}\}$, let $\sigma_i$ be independent Rademacher signs and define:
\[
    \mathfrak R_D:=\mathbb E\max\left\{\sup_H\left|\frac1{|D|}\sum_{U_i\in D}\sigma_ir_{\Brel}(\tau_i^{H,\Brel},U_i)\right|,\;\sup_{H}\sup_{B_b\in\Bval}\left|\frac1{|D|}\sum_{U_i\in D}\sigma_i\delta_{b,i}(H)\right|\right\},
\]
\looseness=-1 where the expectation is over the task-level evaluations and the Rademacher signs. Here, \(\mathfrak R_D\) is the expected maximum of two worst-case Rademacher averages: one over released-benchmark scores as \(H\) varies and the other over gain destruction as both \(H\) and \(B_b\) vary. It measures the complexity cost of searching over harnesses and counterfactuals using \(D\). For fixed-complexity bounded classes, \(\mathfrak R_D\) is typically of order \(|D|^{-1/2}\). Hence, sufficiently large search sets allow a positive margin \(\gamma>r\), as required below.

\begin{lemma}
\label{lem:uniform-evolution}
Under Condition~\ref{cond:sampling}, define:
\begin{equation}
    r:=\max_{D\in\{D_{\rm evo},D_{\rm disc}\}}\left\{2\mathfrak R_D+\sqrt{\frac{8\log(8/\alpha)}{|D|}}\right\},\qquad x_t:=\sqrt{\frac{8\log(8T/\alpha)}{|D_{{\rm conf},t}|}}.
    \label{eq:search-radius}
\end{equation}
Then, with probability at least $1-\alpha/2$, simultaneously,
\begin{equation}
    \begin{aligned}
        \sup_H\left|\widehat R_{\Brel,D_{\rm evo}}(\Agent,H)-R_{\rm rel}(\Agent,H)\right|&\le r,\\
        \sup_H\sup_{\substack{B_b\in\Bval\\D\in\{D_{\rm evo},D_{\rm disc}\}}}\left|\widehat\Delta_{b,D}(H;H_0)-\Delta_b(H;H_0)\right|&\le r,\\
        \left|\widehat\Delta_{b_t,D_{{\rm conf},t}}(H_t;H_0)-\Delta_{b_t}(H_t;H_0)\right|&\le x_t
    \end{aligned}
    \label{eq:uniform-search}
\end{equation}
for every round \(t\) in which a proposal passes the validity firewall and is evaluated on \(D_{{\rm conf},t}\).
\end{lemma}

\noindent{\bf Proof.} For either search set $D$ with $n=|D|$, standard symmetrization \citep{zhang2023mathematical} and the bounded-differences inequality give:
\[
    \begin{aligned}
        \Pr\Bigl\{\max\Bigl[&\sup_H\left|\widehat R_{\Brel,D}(\Agent,H)-R_{\rm rel}(\Agent,H)\right|,\\
        &\sup_H\sup_{B_b\in\Bval}\left|\widehat\Delta_{b,D}(H;H_0)-\Delta_b(H;H_0)\right|\Bigr]>2\mathfrak R_D+z\Bigr\}\le e^{-nz^2/8},
    \end{aligned}
\]
since all task-level quantities above lie in $[-2,2]$. Taking $z=\sqrt{8\log(8/\alpha)/n}$ and a union bound over $D_{\rm evo}$ and $D_{\rm disc}$ gives total failure probability at most $\alpha/4$.

For each round, conditional on $\mathcal F_{D_{{\rm conf},t}}^-$, $\widehat\Delta_{b_t,D_{{\rm conf},t}}(H_t;H_0)$ is an average of independent $[-2,2]$ variables with mean $\Delta_{b_t}(H_t;H_0)$. Hence, Hoeffding's inequality gives:
\[
    \Pr\!\left\{\left|\widehat\Delta_{b_t,D_{{\rm conf},t}}(H_t;H_0)-\Delta_{b_t}(H_t;H_0)\right|>x_t\,\middle|\,\mathcal F_{D_{{\rm conf},t}}^-\right\}\le2e^{-|D_{{\rm conf},t}|x_t^2/8}=\frac{\alpha}{4T}.
\]
Taking expectations and a union bound over $t=1,\ldots,T$ contributes at most another $\alpha/4$.

\subsection{Proof of Theorem~1}
\label{app:proof-theorem1}

Let $K=K(\mathcal A)<\infty$. By definition, there exist $B_{b_1},\ldots,B_{b_K}\in\mathcal A$ such that $B_{b_K\circ\cdots\circ b_1}=B_0$. Closure of $\Bval$ gives $B_{b_j\circ\cdots\circ b_1}\in\Bval$ for every $j=1,\ldots,K$.

For any $H$ and $B_b,B_{b'}\in\Bval$, the definition of $\rho$ gives:
\[
\begin{aligned}
    &\Delta_{b'\circ b}(H;H_0)-\Delta_b(H;H_0)-\Delta_{b'}(H;H_0)\\
    &\quad\le\left[\Delta_{b'\circ b}(H;H_0)-\Delta_b(H;H_0)-\Delta_{b'}(H;H_0)\right]_+\le\rho.
\end{aligned}
\]
Taking $b=b_{j-1}\circ\cdots\circ b_1$ and $b'=b_j$ therefore yields, for $j=2,\ldots,K$,
\[
    \Delta_{b_j\circ\cdots\circ b_1}(H;H_0)\le\Delta_{b_{j-1}\circ\cdots\circ b_1}(H;H_0)+\Delta_{b_j}(H;H_0)+\rho.
\]
The base case is:
\[
    \Delta_{b_1}(H;H_0)\le\Gamma_{\mathcal A}(H;H_0).
\]
If
\[
    \Delta_{b_{j-1}\circ\cdots\circ b_1}(H;H_0)\le(j-1)\Gamma_{\mathcal A}(H;H_0)+(j-2)\rho,
\]
then the preceding recurrence and $B_{b_j}\in\mathcal A$ give:
\[
\begin{aligned}
    \Delta_{b_j\circ\cdots\circ b_1}(H;H_0)&\le(j-1)\Gamma_{\mathcal A}(H;H_0)+(j-2)\rho+\Gamma_{\mathcal A}(H;H_0)+\rho\\
    &=j\Gamma_{\mathcal A}(H;H_0)+(j-1)\rho.
\end{aligned}
\]
Thus induction and $B_{b_K\circ\cdots\circ b_1}=B_0$ give:
\begin{equation}
    \Delta_{\rm BS}(H;H_0)\le K\Gamma_{\mathcal A}(H;H_0)+(K-1)\rho.
    \label{eq:supp-component-bound}
\end{equation}
Since $G_{\min,\mathcal A}=G_{\rm rel}-\Gamma_{\mathcal A}$ and $G_0=G_{\rm rel}-\Delta_{\rm BS}$, the preceding bound also gives:
\[
    G_0(H;H_0)\ge G_{\min,\mathcal A}(H;H_0)-(K-1)\{\Gamma_{\mathcal A}(H;H_0)+\rho\}.
\]

For certification, condition on $\mathcal F_{D_{\rm cert}}^-$. Then $H^\star$ and $\mathcal A$ are fixed. Let $m=|\mathcal A|$ and $n=|D_{\rm cert}|$. For each $B_b\in\mathcal A$, the task-level summands defining $\widehat G_{b,D_{\rm cert}}(H^\star;H_0)$ are
\[
    r_{B_b}(\tau_i^{H^\star,B_b},U_i)-r_{B_b}(\tau_i^{H_0,B_b},U_i)\in[-1,1],
\]
while those defining $\widehat\Delta_{b,D_{\rm cert}}(H^\star;H_0)$ are
\[
    r_{\Brel}(\tau_i^{H^\star,\Brel},U_i)-r_{\Brel}(\tau_i^{H_0,\Brel},U_i)-r_{B_b}(\tau_i^{H^\star,B_b},U_i)+r_{B_b}(\tau_i^{H_0,B_b},U_i)\in[-2,2].
\]
By the sampling conditions, these summands are conditionally independent across tasks and have conditional means $G_b(H^\star;H_0)$ and $\Delta_b(H^\star;H_0)$, respectively. Hence the one-sided Hoeffding inequalities give:
\[
    \begin{aligned}
        \Pr\!\left\{G_b(H^\star;H_0)<\widehat G_{b,D_{\rm cert}}(H^\star;H_0)-s_G\,\middle|\,\mathcal F_{D_{\rm cert}}^-
        \right\}
        &\le e^{-ns_G^2/2},\\
        \Pr\!\left\{\Delta_b(H^\star;H_0)>\widehat\Delta_{b,D_{\rm cert}}(H^\star;H_0)+s_\Delta\,\middle|\,\mathcal F_{D_{\rm cert}}^-\right\}
        &\le e^{-ns_\Delta^2/8},
    \end{aligned}
\]
where we set $s_G=\sqrt{\frac{2\log(4m/\alpha)}{n}},s_\Delta=\sqrt{\frac{8\log(4m/\alpha)}{n}}$. A union bound over all $B_b\in\mathcal A$ then gives, with conditional probability at least $1-\alpha/2$,
\[
    G_b(H^\star;H_0)\ge\widehat G_{b,D_{\rm cert}}(H^\star;H_0)-s_G,\qquad \Delta_b(H^\star;H_0)\le\widehat\Delta_{b,D_{\rm cert}}(H^\star;H_0)+s_\Delta,
\]
simultaneously for all $B_b\in\mathcal A$. Therefore,
\[
    G_{\min,\mathcal A}(H^\star;H_0)\ge\widehat G_{\min,\mathcal A,D_{\rm cert}}(H^\star;H_0)-s_G,\qquad\Gamma_{\mathcal A}(H^\star;H_0)\le\widehat\Gamma_{\mathcal A,D_{\rm cert}}(H^\star;H_0)+s_\Delta.
\]
Substituting these bounds into~\eqref{eq:supp-component-bound} and the corresponding lower bound for $G_0$ gives the two inequalities in Theorem~1. The same probability bound holds unconditionally by the tower property.

\subsection{Proof of Theorem~2}
\label{app:proof-theorem2}
Work on the common evolution event $\mathcal E$ of Lemma~\ref{lem:uniform-evolution}. At every confirmed round,
\begin{equation}
    \Delta_{b_t}(H_t;H_0)\ge\widehat\Delta_{b_t,D_{{\rm conf},t}}(H_t;H_0)-x_t\ge\varepsilon+\gamma.
    \label{eq:supp-confirmed-lower}
\end{equation}
Let $t_1<\cdots<t_m$ list the confirmed rounds. If $m=0$, the claim is immediate. For $j<k$, archive nesting gives $B_{b_{t_j}}\in\mathcal A_{t_k-1}$. Empirical feasibility at the later round and the uniform estimation bound in Lemma~\ref{lem:uniform-evolution} give:
\begin{equation}
    \Delta_{b_{t_j}}(H_{t_k};H_0)\le\widehat\Delta_{b_{t_j},D_{\rm evo}}(H_{t_k};H_0)+r\le\varepsilon+r.
    \label{eq:supp-later-feasibility}
\end{equation}
Combining this inequality with confirmation at the earlier round yields:
\[
    \sup_{B_b\in\Bval}|\Delta_b(H_{t_j};H_0)-\Delta_b(H_{t_k};H_0)|\ge\Delta_{b_{t_j}}(H_{t_j};H_0)-\Delta_{b_{t_j}}(H_{t_k};H_0)\ge\gamma-r.
\]
Because $0<\gamma-r\le1$, the packing condition in Theorem~\ref{theorem:Conf_t=1} applies and gives $m\le(C/(\gamma-r))^d$. If the round budget exceeds this bound and a valid proposal is tested every round, at least one round is not confirmed. 

\subsection{Proof of Theorem~3}
\label{app:proof-theorem3}

Work on the event $\mathcal E$ of Lemma~\ref{lem:uniform-evolution}. If $\mathsf{Conf}_t=0$, then:
\[
    \Delta_{b_t}(H_t;H_0)\le\widehat\Delta_{b_t,D_{{\rm conf},t}}(H_t;H_0)+x_t<\varepsilon+\gamma+2x_t.
\]
Because \(B_{b_t}\) solves the empirical Challenger problem, the \(D_{\rm disc}\) deviation bound in~\eqref{eq:uniform-search}~gives:
\[
    \begin{aligned}
        \sup_{B_b\in\Bval}\Delta_b(H_t;H_0)
        &\le\sup_{B_b\in\Bval}\widehat\Delta_{b,D_{\rm disc}}(H_t;H_0)+r\\
        &=\widehat\Delta_{b_t,D_{\rm disc}}(H_t;H_0)+r\\
        &\le\Delta_{b_t}(H_t;H_0)+2r
        <\varepsilon+\gamma+2x_t+2r.
    \end{aligned}
\]

Now take any $\widetilde H$ satisfying $\Delta_b(\widetilde H;H_0)\le\varepsilon-r$ for all $B_b\in\Bval$. Since $\mathcal A_{t-1}\subseteq\Bval$,
\[
    \widehat\Delta_{b,D_{\rm evo}}(\widetilde H;H_0)\le\Delta_b(\widetilde H;H_0)+r\le\varepsilon,\qquad B_b\in\mathcal A_{t-1},
\]
so $\widetilde H$ is feasible for the empirical Proposer. Hence,
\[
    \begin{aligned}
        R_{\rm rel}(\Agent,H_t)
        &\ge\widehat R_{\Brel,D_{\rm evo}}(\Agent,H_t)-r\\
        &\ge\widehat R_{\Brel,D_{\rm evo}}(\Agent,\widetilde H)-r\\
        &\ge R_{\rm rel}(\Agent,\widetilde H)-2r.
    \end{aligned}
\]
Subtracting $R_{\rm rel}(\Agent,H_0)$ and taking the supremum over $\widetilde H$ gives the second claim. If $\varepsilon\ge r$, the comparison set contains $H_0$.

\renewcommand{\theequation}{B.\arabic{equation}}
\renewcommand{\thetable}{B.\arabic{table}}
\renewcommand{\theHtable}{\Alph{section}.\arabic{table}}
\renewcommand{\thefigure}{B.\arabic{figure}}
\setcounter{equation}{0}
\setcounter{table}{0}
\setcounter{figure}{0}

\section{Method Implementation Details}
\label{app:implementation}

{Section~\ref{app:legitimate-vs-shortcut} separates gains that use task information from benchmark-wide shortcuts, and Section~\ref{app:archive-evaluation} describes how the archive is used in practice.} Section~\ref{app:comparison-methods} defines the comparison methods and shared settings. Section~\ref{app:CHASE-online} describes the Proposer and Challenger updates. Section~\ref{app:validity-firewall} specifies the protocol transformations and validity firewall. Section~\ref{app:agent-harnesses} summarizes the initial harness $H_0$, the fixed Proposer and Challenger harnesses, and the CHASE pseudocode. Section~\ref{app:Generalization-Gate} details the HarnessCompass generalization gate used in OfficeQA and Syn-Ledger. Section~\ref{app:harnessevolve-adaptation} describes HarnessEvolve and our budget-constrained adaptation.

{Three terms recur below. A protocol is the pair $(V,Q)$ of Section~\ref{sec:problem-setup}: file names and layout, the format of tool outputs, and the order of results. An archive is the set of confirmed counterfactuals, to which entries are only appended. Certification is a one-time evaluation on the sealed set $D_{\rm cert}$ under $\Brel$ and the archive.}

\subsection{\texorpdfstring{{Task Information versus Protocol Correlations}}{Task Information versus Protocol Correlations}}
\label{app:legitimate-vs-shortcut}

{Harness evolution is supposed to adapt a harness to the documents, to the tools, and to the way questions are asked, and we do not treat such adaptation as a shortcut. The distinction lies in what the harness relies on. A gain may rely on information that solving the task requires, such as the content of the documents, the semantics of the tools, or the format of the questions. Such a gain carries over to any release of the benchmark that keeps the task. A benchmark-wide shortcut instead relies on a protocol correlation, that is, on a protocol feature that solving the task does not require but that the released protocol ties to the evidence across tasks, such as a file name, the order of search results, or the layout of a converted document. Its gain carries over only while the correlation holds.}

{A valid transformation (Section~\ref{sec:From Counterfactuals to Harness Evolution}) preserves the question, answer, evidence, scoring rule, and the interface properties required by the benchmark. Gain destruction measures the change in the evolved harness score relative to the change in the initial harness score. Equal score drops cancel. A positive value identifies sensitivity to the tested change. Interpreting that sensitivity as a shortcut also requires specifying which properties the deployment guarantees. The protocol correlations of a benchmark cannot be settled by a fixed list: they are open-ended, and which of them a particular harness exploits is revealed only by that harness. This is why CHASE searches for counterfactuals against the current harness and confirms them on held-out tasks, instead of filtering harness edits with a fixed rule.}

{Three examples make the distinction concrete. In a multiple-choice benchmark, formatting the final answer as ``Answer: X'' so that the scorer can parse it adapts to the scoring interface, whereas preferring option C because correct answers often sit there is a shortcut. Shuffling the options is the corresponding valid transformation, and it is already standard practice. In SWE-bench Verified, learning a repository's test framework and the format of its issues is adaptation, whereas reading future commits through \texttt{git log --all}, which the evaluation images left reachable~\citep{caisi2025cheating}, is a shortcut. Removing future references from the repository is the corresponding valid transformation, which benchmark revisions apply~\citep{zheng2026swebenchproverified}. In OfficeQA, the instruction to keep an explicit unit for every operand uses information that the Treasury Bulletin documents contain. A harness that instead located the unit by its customary position before a table would use a layout convention of the released text conversion. The first OfficeQA counterfactual provides related evidence: the gain of $H_1$ is sensitive to the format in which retrieval results are returned, and we conjecture that this sensitivity may come from reliance on the customary position of units and notes (Section~\ref{sec:experimental-results}).}

{Which protocol components may vary is part of the benchmark specification. If a deployment guarantees part of a tool's output, such as the name and nesting of a field, a harness that reads that part relies on an interface contract, and the component can be left out of the permitted protocol components (Appendix~\ref{app:validity-firewall}). The shortcuts that CHASE targets instead rely on regularities that no interface guarantees, such as the order of the fields in a JSON object (Appendix~\ref{app:swebench}). A valid transformation can also remove a gain that comes from a reasonable but brittle edit. The archive constraint then leads the next round to repair the edit rather than discard it, as in the third OfficeQA round (Section~\ref{sec:experimental-results}).}

\subsection{\texorpdfstring{{Using the Archive in Practice}}{Using the Archive in Practice}}
\label{app:archive-evaluation}

{The protocol does not evolve. The archive is append-only. $\Brel$ is never modified, and every archived counterfactual $B_{b}$ is a deterministic, executable transformation of the same tasks, stored with its code, its validity checks, and its confirmation statistics (Section~\ref{sec:finite-archive}). Scores on $\Brel$ therefore keep their meaning, and an archived counterfactual adds a column to a results table instead of replacing one. Comparisons across systems and over time remain valid on every column that both systems have been scored on, and a missing column can be filled at any time, because the transformation can be applied to any system.}

{Each column is interpretable on its own. An archive entry records the Challenger's natural-language proposal, the typed transformation specification, the validity checks it passed, and the confirmed gain destruction on fresh tasks (Appendix~\ref{app:officeqa-additional-results} shows the entry for $B_{b_1}$). The archive thus documents which dependencies were neutralized and how each system's gain changed, in the same way that a benchmark revision documents which leakage channels it closed~\citep{zheng2026swebenchproverified}.}

{The costs separate into three parts. The Challenger's search is paid once by the developer during optimization. Certification costs one evaluation pass per archived counterfactual: on OfficeQA, 76 questions with three rollouts each, the same as one pass on $\Brel$ (Appendix~\ref{app:resource-use}). A previously evaluated system needs no re-optimization. Scoring it on a new counterfactual is one additional pass, and only when that column is wanted.}

{Two workflows follow. A benchmark maintainer can run CHASE against representative harnesses before release and publish the frozen archive alongside $\Brel$, together with the per-entry gain destruction of each harness. A harness developer can run the Challenger during development and certify the frozen harness on sealed tasks under $\Brel$ and the archive, which is the protocol of Section~\ref{sec:finite-archive}: $D_{\rm evo}$, $D_{\rm disc}$, $D_{{\rm conf},t}$, and $D_{\rm cert}$ are disjoint, and $D_{\rm cert}$ is used once. In both workflows, what changes is not the benchmark but the report: a single released score becomes a released score together with the gain destruction under each archived counterfactual.}

\paragraph{{Worked example: document question answering.}}
\looseness=-1 {A developer with gpt-5.6-sol and the initial harness $H_0$ runs CHASE on OfficeQA Full for three rounds before shipping (Section~\ref{sec:experimental-results}, Appendix~\ref{app:officeqa-additional-results}). After round 1, the Challenger finds that the $+8.16\%$ gain of $H_1$ on $D_{\rm evo}$ becomes $-5.10\%$ once retrieved text is re-encoded as table context. The transformation enters the archive and constrains rounds 2 and 3, so the final $H_3$ keeps its gain under both protocols. The developer ships $H_3$ together with the archive and reports three columns on the sealed $D_{\rm cert}$: $68.86\%$ under $\Brel$, $68.42\%$ on average over the archive, and $67.98\%$ in the worst case (Table~\ref{tab:main-results}). The first column is the released score and compares directly with any result on OfficeQA Full. The other two are additional, and a previously evaluated system obtains each of them with one evaluation pass. Selection by released score alone would favor $H_1$ over $H_0$, although the gain of $H_1$ reverses when the format of retrieval results changes.}

\paragraph{{Worked example: software patching.}}
{A developer evolves a coding-agent harness on SWE-bench Verified, where each task is an issue, a repository at a base commit, and hidden tests. Suppose the Proposer adds a step that inspects the repository history before editing. On evaluation images that left future references reachable, this raises the resolution rate, because the fixing commit can be visible. The Challenger's counterfactual removes every reference after the base commit, which keeps the issue, the code, and the tests unchanged. The gain disappears, the counterfactual is archived, and later Proposer rounds must find gains that survive it, for example better fault localization from the failing tests. The developer ships the harness with the archive and reports the released score, which compares directly with the leaderboard, together with the score under the archived counterfactual. Git-history leakage on SWE-bench Verified is documented by~\citet{caisi2025cheating}, and this counterfactual is the change that SWE-Bench Pro Verified applies to SWE-Bench Pro by rebuilding each repository as a single commit~\citep{zheng2026swebenchproverified}: a developer using CHASE would have found the dependence before submitting to the leaderboard, and a maintainer running CHASE against representative harnesses before release would have released the benchmark with this counterfactual in place.}

\subsection{Comparison Methods and Shared Settings}
\label{app:comparison-methods}

Within OfficeQA or Syn-Ledger, all methods use the same foundation model and reasoning setting, initial harness $H_0$, request schema, sampling parameters, corpus, retrieval indexes, tools, context policy, step limit, time limit, retry policy, and scorer. For OfficeQA, all methods use \texttt{gpt-5.6-sol with high reasoning effort} for every model role, including the foundation model, Proposer, and Challenger. The original HarnessCompass configuration uses \texttt{gpt-5.4 in the non-thinking setting} \citep{zhang2026harnesscompass}. We use a reasoning-enabled Proposer in OfficeQA so that any observed shortcut behavior cannot be simply attributed to insufficient reasoning. For Syn-Ledger, we instead use \texttt{gpt-5.6-sol in the non-thinking setting} for all model calls to reduce experimental cost. The foundation model $\Agent$ is not told which split each task belongs to. Each optimized method runs three Proposer rounds, except HarnessEvolve (Appendix~\ref{app:harnessevolve-adaptation}). Table~\ref{tab:method-implementations} summarizes the six methods. RawHarness keeps $H_0$ unchanged to measure gains from optimization. The fixed generalization gate applies to HarnessCompass and $B_0$-Access. CHASE uses archive constraints without the generalization gate. HarnessEvolve bounds regression across recent task batches, whereas CHASE bounds gain destruction across archived counterfactual benchmarks.

\begin{table}[htbp]
\centering
\caption{Method comparison in this study. ``Shortcut-aware'' indicates that harness evolution explicitly accounts for shortcut behavior. ``Worst-case constraints'' bound the maximum regression across recent task batches (HarnessEvolve) or the maximum gain destruction across archived counterfactual benchmarks (CHASE) during optimization. ``$B_0$-free'' indicates no access to $B_0$ during optimization.}
\label{tab:method-implementations}
\small
\setlength{\tabcolsep}{5pt}
\renewcommand{\arraystretch}{1.10}
\begin{tabular}{lccccc}
\toprule
Method & Proposer & Challenger & Shortcut-aware
& Worst-case constraints & $B_0$-free \\
\midrule

RawHarness
& \xmark & \xmark & \xmark & \xmark & \cmark \\

Meta-Harness
& \cmark & \xmark & \xmark & \xmark & \cmark \\

HarnessCompass
& \cmark & \xmark & \cmark & \xmark & \cmark \\

HarnessEvolve
& \cmark & \xmark & \cmark & \cmark & \cmark \\

$B_0$-Access
& \cmark & \xmark & \cmark & \xmark & \xmark \\

\rowcolor{red!8}
CHASE
& \cmark & \cmark & \cmark & \cmark & \cmark \\
\bottomrule
\end{tabular}
\end{table}

For cost control, we use two rollouts per harness--question pair on $D_{\rm evo}$, two on $D_{\rm disc}$, five on each $D_{{\rm conf},t}$, and three on $D_{\rm cert}$. We set $(\eta_{{\rm conf},t},\varepsilon)=(0.05,0.05)$ on Syn-Ledger and $(0.075,0.05)$~on~OfficeQA.

\subsection{Proposer and Challenger Updates}
\label{app:CHASE-online}

For controlled comparison, CHASE uses the same Proposer backbone as HarnessCompass, including its feedback procedure and separate structural and guidance tracks~\citep{zhang2026harnesscompass}. For HarnessCompass and $B_0$-Access, the generalization gate is specified once in the system prompt, fixed before optimization, and applied to every candidate edit. CHASE does not use the generalization gate.

At round $t$, HarnessCompass and CHASE each begin with their current harness $H_{t-1}$ and evaluate it on the full $D_{\rm evo}$. Following the history interface introduced by Meta-Harness~\citep{lee2026metaharness}, the Proposer can then inspect the method's complete optimization history through a controlled read-only file interface. This history contains the code of the current harness and all previously evaluated candidates, together with their scores and execution traces, including prompts, tool calls, model outputs, and state changes. The files retain their original names and relative paths so that the Proposer can navigate the history directly. HarnessCompass and CHASE have the same access to their own histories, but neither method can inspect the other method's rollouts. To reduce evaluation cost, we reuse previously computed estimates.

\noindent{\bf Proposer Update.} Following HarnessCompass, the Proposer analyzes this history through two complementary feedback passes~\citep{zhang2026harnesscompass}. Proactive feedback describes how the current harness affects the model's behavior and suggests possible improvements, while hindsight feedback uses the observed outcomes to identify behaviors associated with success or failure. The Proposer then develops one candidate for structural components and another for guidance components. Finally, it applies $R^3$---Revision, Recombination, and Refinement---to revise the two candidates and combine compatible changes into a third, integrated candidate~\citep{zhang2026harnesscompass}. 

HarnessCompass selects a new harness only when one of the three candidates achieves a higher released-benchmark score. Otherwise, it retains the current harness. In CHASE, before candidate generation, the Proposer is also given the current archive $\mathcal A_{t-1}$ and the corresponding $D_{\rm evo}$ scores and trajectories. The archive constraints are included in the Proposer prompt during candidate generation and are also enforced during candidate selection. After evaluation, CHASE forms a selection pool containing the three new candidates, $H_{t-1}$, and $H_0$. It removes every harness that violates $\widehat{\Delta}_{b,D_{\rm evo}}(H;H_0)\leq\varepsilon$ for at least one $B_b\in\mathcal A_{t-1}$, then selects the remaining harness with the highest released-benchmark score on $D_{\rm evo}$ as $H_t$. Thus, an $H_{t-1}$ that violates an archive constraint is excluded from the score comparison, while $H_0$ always provides a feasible fallback.

\noindent{\bf Challenger Update.} After CHASE selects $H_t$, the Challenger searches for protocol transformations $\Phi_b$ under which the gain of $H_t$ over $H_0$ may disappear. In each round, it initially generates eight proposals, each containing a natural-language description and a typed transformation specification. On OfficeQA, these specifications modify protocol-level features while preserving the underlying question and answer. See Appendix~\ref{app:validity-firewall} for details.

Before any model rollout, the host applies the validity firewall and duplicate screen to all eight proposals. A proposal is discarded if it fails any validity check in Table~\ref{tab:validity-audit} or repeats a transformation already stored in the archive. To detect such repeats, the host represents each transformation as an ordered sequence of allowlisted API operators and their normalized arguments. A new proposal is treated as a duplicate when this representation matches that of an archived transformation. A hash of the representation is stored with each archived transformation and used for subsequent comparisons.

Among the proposals that pass both screens, CHASE selects two before any model rollout on \(D_{\rm disc}\), by a deterministic pre-evaluation ranking rule. The rule favors proposals that produce a larger observable change when applied to a fixed example, appear relevant to more discovery questions, and target more of the suspected shortcut mechanisms. When possible, the two selected proposals represent different kinds of protocol transformation. The retained proposals are then evaluated on \(D_{\rm disc}\), and CHASE selects the proposal with the largest \(\widehat{\Delta}_{b,D_{\rm disc}}(H_t;H_0)\).

Since the Challenger searches over $D_{\rm disc}$, a large value of $\widehat{\Delta}_{b, D_{\rm disc }}\left(H_t ; H_0\right)$ may be inflated by selection. For OfficeQA, the selected transformation is fixed before confirmation and evaluated on a 24-question set $D_{{\rm conf},t}$, with five rollouts per question. We evaluate both $H_t$ and $H_0$ under the released benchmark $\Brel$ and the proposed counterfactual benchmark $B_{b_t}$. For each harness--benchmark pair, the five rollout scores are first averaged within each of the 24 questions and then across questions. The counterfactual benchmark \(B_{b_t}\) enters the archive when the resulting task-level gain destruction is at least \(\eta_{{\rm conf},t}=0.075\). After three rounds, \methodchase{} selects its final harness from $H_0,\ldots,H_3$ using the constrained selection in Section~\ref{sec:method}. For fixed task and rollout budgets, the cost of evaluating each harness scales linearly with the number of archived benchmarks $|\mathcal A_{t-1}|$.

\subsection{Protocol Transformations and Validity Firewall}
\label{app:validity-firewall}

OfficeQA permits reversible changes to identifiers, document layout, tool-interface schemas, and enumeration order, including valid compositions. For Syn-Ledger, the Challenger proposes transformations within the permitted protocol components. The canonical neutralizations and placebo transformations in Section~\ref{app:synledger-protocols} are used only to validate the benchmark construction, not Challenger inputs. Across both benchmarks, transformations preserve the semantic task, target, available evidence, and scoring semantics~\citep{ribeiro2020checklist}.

Table~\ref{tab:validity-audit} lists the seven checks evaluated by the host-side validity firewall. Each check returns a Boolean value, and a proposal passes the validity firewall only when all seven values are true. The Challenger's natural-language validity claims are retained in the audit record but are not used as substitutes for these executable checks.
\begin{table}[htbp]
\centering
\caption{Validity firewall: every check must pass.}
\label{tab:validity-audit}
\small
\setlength{\tabcolsep}{5pt}
\renewcommand{\arraystretch}{1.10}
\begin{tabularx}{\textwidth}{>{\raggedright\arraybackslash\bfseries}p{0.21\textwidth}Y}
\toprule
\rowcolor{cglightgray}
Invariant & \textbf{Required check} \\
\midrule
Task and answer
& The question, requested target, and ground-truth answer are unchanged. \\
\addlinespace[2pt]
Available information
& Document contents and membership match after undoing identifier mappings; no content is added, removed, altered, or moved between documents. \\
\addlinespace[2pt]
Access and resources
& The same documents and tool operations remain accessible, with unchanged tool-call, context, and step limits. \\
\addlinespace[2pt]
Scoring
& Answer normalization, numeric tolerance, and scorer verdicts on fixed correct and incorrect outputs are unchanged. \\
\addlinespace[2pt]
Declared changes
& Only the protocol metadata declared in the proposal are modified. \\
\addlinespace[2pt]
Task independence
& Rules do not depend on questions, answers, gold sources, or dataset splits. \\
\addlinespace[2pt]
Replay and inversion
& Identical inputs and settings reproduce the transformation; reversible changes recover the original state. \\
\bottomrule
\end{tabularx}
\end{table}

\begin{table}[htbp]
\centering
\caption{Host-side pseudocode for the Syn-Ledger validity firewall.}
\label{tab:synledger-firewall-pseudocode}
\begin{harnesspanel}{cgblue}{cglightblue}
\textbf{Syn-Ledger validity firewall: host-side pseudocode.}

\begin{tabularx}{\linewidth}{@{}>{\bfseries\color{cgblue}}r@{\hspace{0.65em}}Y@{}}
1 & Parse the Challenger output as a typed transformation specification \(s\); reject any operator, argument, or modified field outside the allowlisted protocol API. \\[2pt]

2 & Canonicalize \(s\) into an ordered operator--argument representation. Reject \(s\) if the same canonical representation already occurs in the archive. \\[2pt]

3 & For every task \(U\) in the fixed validity-check set, execute \(Z_1\leftarrow\operatorname{Render}(U,s,\mathrm{seed})\) and \(Z_2\leftarrow\operatorname{Render}(U,s,\mathrm{seed})\), and parse \(Z_1\) back into typed document records. \\[2pt]

4 & \textit{Task and answer:} verify that the question and requested target are unchanged and that recomputing the answer from the parsed records returns the original ground-truth answer. \\[2pt]

5 & \textit{Available information:} verify equality of document membership and of the canonical record multisets before and after transformation. \\[2pt]

6 & \textit{Access, resources, and scoring:} verify equality of the accessible tool operations, resource limits, scorer hash, and scorer verdicts on fixed correct and incorrect output fixtures. \\[2pt]

7 & \textit{Declared changes and task independence:} compute the modified fields and require them to be a subset of the fields declared by \(s\); reject any reference to a question identifier, answer, gold source, or dataset split. \\[2pt]

8 & \textit{Replay and inversion:} require \(Z_1=Z_2\) and require the typed document records recovered from \(Z_1\) to equal those of \(U\). \\[2pt]

9 & Return \texttt{PASS} if and only if all seven checks in Table~\ref{tab:validity-audit} pass; otherwise return the failed check names and reject the proposal before \(D_{\rm disc}\) evaluation.
\end{tabularx}
\end{harnesspanel}
\end{table}

Table~\ref{tab:synledger-firewall-pseudocode} gives the Syn-Ledger implementation of the checks in Table~\ref{tab:validity-audit}. The canonical parser reconstructs the typed records from each rendered document, after which the benchmark generator independently recomputes the ground-truth answer. Identity, valid single-field transformations, and valid compositions serve as positive controls, whereas transformations that modify the answer, introduce undeclared task-dependent fields, or use a non-allowlisted operator serve as negative controls. OfficeQA applies the same requirement that all seven checks pass with benchmark-specific JSON canonicalization and tool-output fixtures. We omit the analogous OfficeQA pseudocode.

\noindent\textbf{OfficeQA firewall example.} In the first OfficeQA round, the Challenger first returned eight typed transformation proposals. The validity firewall and the duplicate screen accepted seven and rejected one before any model rollout on \(D_{\rm disc}\). The rejected proposal specified a seeded permutation of the search results. When executed on the validity firewall, however, it left the result order unchanged, so the duplicate screen excluded it as an inactive transformation and removed it from the candidate pool.

\subsection{Initial Harness, Proposer Harness, Challenger Harness, and Pseudocode}
\label{app:agent-harnesses}

Tables~\ref{tab:initial-harness}--\ref{tab:challenger-harness} summarize the three harness designs, including their accessible context, tools, control, outputs, and condensed core instructions. The harness evolves from $H_0$, whereas the Proposer and Challenger use fixed harness configurations. The instructions below are illustrative templates. Benchmark-specific tool schemas and run histories are supplied separately.

\noindent{{\bf Harness and protocol in the code.} In the code, the harness is the set of files that the Proposer may edit: the system prompt, the long-term memory, the retrieval policy, the tool policy, the policy graph, and the guidance modules. The corpus, the tool operations, the scorer, and the budget stay fixed, and the protocol includes the format in which tool results are serialized. The Challenger edits only the protocol: the runtime applies its transformation to tool outputs after retrieval, under byte quotas that are identical for every transformation. The tool wrapper that calls the tools and post-processes their outputs belongs to the harness because the developer writes it. The output format belongs to the protocol because the benchmark provider serializes it.}

{{\bf Policy graph.} Before each search call, the runtime evaluates the policy graph on question-level and budget features and obtains the retrieval plan: the query list after term selection, rewriting, and decomposition, the query mode, the ranking and merge strategies, document diversity, and the low-coverage fallback, which the retrieval layer evaluates on the actual coverage of the top hit. The serialized policy graph is also part of the model's instructions, so a plan entry outside this list can reach the model, although the runtime never executes it. The harness descriptions in Appendix~\ref{app:officeqa-additional-results} refer to the executed behavior and mark such entries for $H_1$, and the released harness files contain only executed fields.}

\begin{table}[htbp]
\centering
\caption{Initial harness $H_0$.}
\label{tab:initial-harness}
\begin{harnesspanel}{cgpurple}{cglightpurple}
\setlength{\tabcolsep}{4pt}
\renewcommand{\arraystretch}{1.06}
\begin{tabularx}{\linewidth}{@{}>{\raggedright\arraybackslash\bfseries\leavevmode\color{cgpurple}}p{0.18\linewidth}Y@{}}
Accessibility
& \textbf{OfficeQA:} the evaluated release's full transformed-text corpus.
\textbf{Syn-Ledger:} the current task's twelve documents. No gold annotations or generation metadata. \\
Memory
& No task-specific long-term memory. \\
Tools / retrieval
& \textbf{OfficeQA:} content-only passage BM25; conjunctive \texttt{all} matching on the unmodified query, with native BM25 ranking.
\textbf{Syn-Ledger:} task-local tools and result order (Section~\ref{app:synledger-construction}). \\
Control
& Agent-directed tool use within the shared budget; no harness-side query rewriting, fallback retrieval, or reranking. \\
\end{tabularx}
\tcblower
\textbf{Core prompt.} \emph{Use the available tools to answer the question and return the answer in the required format.}
\end{harnesspanel}
\end{table}

\begin{table}[htbp]
\centering
\caption{Proposer harness: generating harness candidates.}
\label{tab:proposer-harness}
\label{tab:proposer-prompt}
\begin{harnesspanel}{cgblue}{cglightblue}
\setlength{\tabcolsep}{4pt}
\renewcommand{\arraystretch}{1.06}
\begin{tabularx}{\linewidth}{@{}>{\raggedright\arraybackslash\bfseries\leavevmode\color{cgblue}}p{0.18\linewidth}Y@{}}
Accessibility
& Harness code, scores, and complete $D_{\rm evo}$ histories through read-only files; CHASE also receives $\mathcal A_{t-1}$ and its $D_{\rm evo}$ feedback. \\
Edit target
& Foundation model prompts, memory, retrieval, tool wrappers, and control code; the foundation model, corpus, available tool operations, scorer, and budget stay fixed. \\
Control
& Diagnose failures and regression risks; develop structural and guidance candidates, then integrate compatible changes through $R^3$ (Section~\ref{app:CHASE-online}). \\
Output
& Executable edits, targeted failures, activation and stopping conditions, expected tool cost, and regression risks. \\
Selection
& \textbf{HarnessCompass:} released-benchmark score with the generalization gate.
\textbf{$B_0$-Access:} $\widehat R_{B_0,D_{\rm evo}}$ with the generalization gate.
\textbf{CHASE:} released-benchmark score with $\widehat\Delta_{b,D_{\rm evo}}(H;H_0)\le\varepsilon$ for every $B_b\in\mathcal A_{t-1}$. \\
\end{tabularx}
\tcblower
\textbf{Core prompt.} \emph{Improve scores on $D_{\rm evo}$. Inspect every task--rollout result; distinguish stable successes, unstable outcomes, and stable failures. Separate observations from hypotheses. Use distinct mechanisms and concise conditional guidance; preserve behavior outside each edit's activation condition. Do not hard-code question identifiers, dataset membership, answers, or source filenames learned as labels.}
\end{harnesspanel}
\end{table}

\begin{table}[htbp]
\centering
\caption{Challenger harness: generating protocol transformations.}
\label{tab:challenger-harness}
\begin{harnesspanel}{cggreen}{cglightgreen}
\setlength{\tabcolsep}{4pt}
\renewcommand{\arraystretch}{1.06}
\begin{tabularx}{\linewidth}{@{}>{\raggedright\arraybackslash\bfseries\leavevmode\color{cggreen}}p{0.18\linewidth}Y@{}}
Accessibility
& $H_t$, $H_0$, the archive, and scores and execution traces for $H_t$ and $H_0$ on $D_{\rm disc}$ through read-only history files. \\
Edit target
& Benchmark protocol transformations from Section~\ref{app:validity-firewall}; both $H_t$ and $H_0$ remain unchanged. \\
Control
& Generate distinct proposals. The host rejects invalid or archived duplicates, ranks valid proposals on $D_{\rm disc}$, and fixes the selected transformation before evaluation on \(D_{{\rm conf},t}\). \\
Output
& Typed transformation specification, declared changes, suspected shortcut, supporting execution traces, and replay settings. \\
\end{tabularx}
\tcblower
\textbf{Core prompt.} \emph{Maximize} $\widehat\Delta_{b,D_{\rm disc}}(H_t;H_0)$: \emph{destroy gain over $H_0$, not merely raw score. Preserve questions, answers, document contents and membership, tool access, budgets, and scorer. Apply the same question-independent rule to both harnesses.}
\end{harnesspanel}
\end{table}

Algorithm~\ref{alg:chase} summarizes how the fixed Proposer and Challenger harnesses above interact with archive-constrained selection, the validity firewall, and confirmation on \(D_{{\rm conf},t}\). {The supplementary code implements each step of Figure~\ref{fig:CHASE} as a separate component. For Evolve and Challenge, a role contract for each benchmark specifies what the Proposer and the Challenger may return and rejects outputs that violate it. For Verify, a benchmark-specific validity firewall runs the checks of Appendix~\ref{app:validity-firewall}, and a confirmation rule admits a counterfactual only if its gain destruction on $D_{{\rm conf},t}$ reaches $\eta_{{\rm conf},t}$. For Remember and the final selection, an archive-constrained selection implements~\eqref{eq:final-selection}. A single driver composes these components into Algorithm~\ref{alg:chase}.}

\begin{algorithm}[htbp]
\caption{\methodchase{} evolution procedure}
\label{alg:chase}
\small
\begin{algorithmic}[1]
\Require initial harness \(H_0\), foundation model \(\Agent\), Proposer,
Challenger, released benchmark \(\Brel\), task sets \(D_{\rm evo}\), \(D_{\rm disc}\), and \(\{D_{{\rm conf},t}\}_{t=1}^{T}\), rounds \(T\), tolerance \(\varepsilon\), confirmation thresholds \(\{\eta_{{\rm conf},t}\}_{t=1}^{T}\)
\Ensure final harness \(H^\star\) and counterfactual archive \(\mathcal A_T\)
\State \(\mathcal A_0 \gets \{\Brel\}\)
\For{\(t=1,\ldots,T\)}
    \State Evaluate \(H_{t-1}\) under every \(B_b\in\mathcal A_{t-1}\) on \(D_{\rm evo}\), and update the Proposer history
    \State \((H_t^{\rm str},H_t^{\rm gui},H_t^{\rm int})\gets\Call{ProposerUpdate}{H_{t-1},\mathcal A_{t-1},\text{Proposer history}}\)
    \State \(\mathcal S_t\gets\{H_t^{\rm str},H_t^{\rm gui},H_t^{\rm int}, H_{t-1},H_0\}\)
    \State Evaluate every \(H\in\mathcal S_t\) under every \(B_b\in\mathcal A_{t-1}\) on \(D_{\rm evo}\)
    \State \(\mathcal F_t\gets\left\{H\in\mathcal S_t:\widehat{\Delta}_{b,D_{\rm evo}}(H;H_0)\leq\varepsilon\ \text{for all } B_b\in\mathcal A_{t-1}\right\}\)
    \State \(H_t\gets\argmax_{H\in\mathcal F_t}\widehat R_{\Brel,D_{\rm evo}}(\Agent,H)\)\Comment{constrained selection}
    \State Evaluate \(H_t\) and \(H_0\) under \(\Brel\) on \(D_{\rm disc}\), and update the Challenger history

    \If{\(H_t=H_0\)}
        \State \(\mathcal A_t\gets\mathcal A_{t-1}\)
        \State \textbf{continue to the next round}
    \EndIf
    \State \(\mathcal P_t\gets\Call{ChallengerUpdate}{H_t,H_0,\mathcal A_{t-1},\text{Challenger history}}\)\Comment{generate typed proposals}
    \State \(\mathcal P_t^{\rm sel}\gets\Call{ScreenAndSelect}{\mathcal P_t,\mathcal A_{t-1}}\)\Comment{retain at most two preselected proposals}
    \State \(\mathcal A_t\gets\mathcal A_{t-1}\)
    \If{\(\mathcal P_t^{\rm sel}\neq\varnothing\)}
        \State Evaluate \(H_t\) and \(H_0\) under each \(B_b\), \(b\in\mathcal P_t^{\rm sel}\), on \(D_{\rm disc}\)
        \State \(b_t\gets\argmax_{b\in\mathcal P_t^{\rm sel}}\widehat{\Delta}_{b,D_{\rm disc}}(H_t;H_0)\)
        \State Fix \(\Phi_{b_t}\) and set \(B_{b_t}\gets\Phi_{b_t}(\Brel)\)
        \State Evaluate \(H_t\) and \(H_0\) under \(\Brel\) and \(B_{b_t}\) on \(D_{{\rm conf},t}\)\Comment{confirmation on \(D_{{\rm conf},t}\)}
        \If{\(\Valid(\Phi_{b_t})=1\) and \(\widehat{\Delta}_{b_t,D_{{\rm conf},t}} (H_t;H_0)\geq\eta_{{\rm conf},t}\)}
            \State \(\mathcal A_t\gets\mathcal A_{t-1}\cup\{B_{b_t}\}\)\Comment{add the confirmed counterfactual}
        \EndIf
    \EndIf
\EndFor
\State Evaluate \(H_0,\ldots,H_T\) under every\(B_b\in\mathcal A_T\) on \(D_{\rm evo}\)
\State \(\mathcal F^\star\gets\left\{H_t:t=0,\ldots,T,\ \widehat{\Delta}_{b,D_{\rm evo}}(H_t;H_0)\leq\varepsilon\ \text{for all }B_b\in\mathcal A_T\right\}\)
\State \(H^\star\gets\argmax_{H\in\mathcal F^\star}\widehat R_{\Brel,D_{\rm evo}}(\Agent,H)\)
\State \Return \(H^\star,\mathcal A_T\)
\end{algorithmic}
\end{algorithm}

\subsection{Generalization Gate Contracts and Enforcement}
\label{app:Generalization-Gate}

Our HarnessCompass baseline retains the optimization components reported by \citet{zhang2026harnesscompass}: proactive and hindsight feedback, separate structural and guidance tracks, $R^3$ integration, a fixed generalization gate, and released-score-based updates. The shared OfficeQA settings for HarnessCompass and CHASE are specified in Section~\ref{app:comparison-methods}.

\looseness=-1 HarnessCompass places a fixed generalization gate between candidate generation and evaluation~\citep{zhang2026harnesscompass}. Its content requirement permits transferable decision rules with explicit applicability conditions and excludes rules tied to a task instance, test, private symbol, path, or task-specific token. Its placement requirement assigns executable capability changes to structural components and behavioral guidance to the system prompt or memory. We fixed a benchmark-specific realization of these requirements before optimization and applied it to every HarnessCompass and $B_0$-Access candidate in every round. The Proposer received the fixed contract in its controlled context: the full OfficeQA generalization-gate contract was available through the read-only interface, and the Syn-Ledger contract was supplied as \texttt{generalization\_gate.txt}. Table~\ref{tab:generalization-gate-contract} records the Proposer-facing contract, and Table~\ref{tab:generalization-gate-enforcement} summarizes the host-side checks applied before foundation model evaluation.

\begin{table}[htbp]
    \caption{Fixed generalization gate contract supplied to the Proposer.}
    \label{tab:generalization-gate-contract}
    \begin{harnesspanel}{cgblue}{cglightblue}
    
    \textbf{Common contract.} Propose only rules that use information observable at inference time and can apply to unseen tasks. State the condition under which each rule applies. Do not encode a task identifier, split role, answer, gold source, private symbol or path, exact observed question, or data-derived lookup table. Put executable retrieval or control changes in structural fields; put behavioral advice in the system prompt, long-term memory, or guidance modules.
    
    \textbf{OfficeQA realization.} Apply the contract jointly to the system prompt, long-term memory, retrieval policy, tool policy, policy graph, and guidance modules. Do not use question identifiers, labels revealing a task's split assignment or whether it is evaluated under the released or a counterfactual benchmark, observed-question text, ground-truth answers or source names, question-specific date--number pairs, or branches on task membership.
    
    \textbf{Syn-Ledger realization.} Base candidate logic only on observable current inputs. Do not use task IDs, split roles, answers, gold sources, exact question lists, private filenames or paths, labels revealing which observable-feature levels are favorable, data-derived lookup tables, or branches on hidden roles, answers, sources, or feature assignments.
    \end{harnesspanel}
\end{table}

\begin{table}[htbp]
    \caption{Host-side enforcement of the fixed generalization gate before model evaluation.}
    \label{tab:generalization-gate-enforcement}
    \begin{harnesspanel}{cggreen}{cglightgreen}
    
    \textbf{OfficeQA.} The host validates the declared schemas and track boundaries, then scans all six editable harness components. It rejects identifiers and role or gold labels; any shared eight-token span with an observed question; any gold source filename; and any pairing of a year from an observed question with a sufficiently long number from its answer. A structural candidate must change a declared retrieval or control-policy field without changing any guidance components; a guidance candidate must leave structural policy unchanged; and an integrated candidate must contain both types of change. Candidate harnesses that make no effective change or duplicate a previously evaluated harness are also rejected.
    
    \textbf{Syn-Ledger.} The host requires an integer-only answer schema, $1$--$12$ search results, $1$--$6$ opened files, and at most four search-query templates using only \texttt{entity}, \texttt{period}, \texttt{periods}, and \texttt{question} as placeholders. It enforces the structural and guidance track boundaries and requires an integrated candidate to combine both types of change from the structural and guidance candidates. It rejects overlap with exact task strings withheld from the Proposer; hidden role labels, answers, gold-source names, and terms used only by the hidden generator; experiment or task identity strings; and rules based on filename length, directory depth, search rank, or serialization format. A candidate proceeds only after its static audit returns \texttt{PASS}.
    \end{harnesspanel}
\end{table}

\subsection{Budget-Constrained HarnessEvolve Adaptation}
\label{app:harnessevolve-adaptation}

All optimized methods in our experiments, except HarnessEvolve, use three outer optimization rounds. HarnessEvolve is organized on two nested scales: within each epoch, it processes a sequence of batches, and each batch can propose a harness update, which is accepted only if it passes the quality and performance gates~\citep{jiang2026harnessevolve}. Held-out validation then selects a harness at the end of the epoch. A HarnessEvolve epoch can contain multiple potential harness updates and does not correspond to one round in our experiments. In our adaptation, round $t$ is one batch-level update opportunity from $H_{t-1}$ to $H_t$. At the start of each epoch, we independently repartition $D_{\rm evo}\cup \left(\bigcup_{t=1}^3 D_{{\rm conf},t}\right)$ into three disjoint batches, with one batch processed in each round. Each task therefore serves as a current-batch task exactly once per epoch and three times over the full three-epoch run, although it may be reevaluated when the performance gate revisits an earlier batch.

For each task, the foundation model $\Agent$ receives the ground-truth answer only to generate a candidate reference trajectory. The {\it trajectory verifier} admits the trajectory only if its actions are grounded in observations and tool outputs and do not directly restate the supplied answer. If no reference is admitted within the allowed number of attempts, the task uses the failed-trajectory fallback. In round $t$, $H_{t-1}$ is executed on the current batch without access to the ground-truth answer. For each failed execution, the Proposer compares the failed trajectory with its verified reference, when available, and locates their first action divergence. It then diagnoses the causes of the failures, groups failures with similar causes, and uses the resulting groups to propose an edit to $H_{t-1}$.

The candidate edit is evaluated by a quality gate and a performance gate, following~\citep{jiang2026harnessevolve}. The quality gate screens for task leakage and prompt bloat and may return the edit for revision. The performance gate requires the candidate to perform at least as well as $H_{t-1}$ on the current batch and limits its score decrease on each of up to $R$ previous-round batches to $\epsilon_{\rm HE}$. An accepted candidate defines $H_t$ and enters the cumulative candidate pool. If the candidate is rejected, $H_t=H_{t-1}$. Table~\ref{tab:harnessevolve-adaptation} lists the settings adjusted for this three-epoch (nine-round) implementation.

\begin{table}[htbp]
\centering
\caption{HarnessEvolve settings adjusted for the nine-round outer optimization budget.}
\label{tab:harnessevolve-adaptation}
\small
\setlength{\tabcolsep}{5pt}
\renewcommand{\arraystretch}{1.15}
\begin{tabularx}{\textwidth}{@{}>{\raggedright\arraybackslash}p{0.20\textwidth}>{\raggedright\arraybackslash}p{0.31\textwidth}Y@{}}
\toprule
Budget item & HarnessEvolve~\citep{jiang2026harnessevolve} & Budget-constrained adaptation \\
\midrule
Optimization schedule
& 20 epochs with batch size 40. Each batch provides one gated update opportunity.
& 3 epochs, each with a three-batch partition of $D_{\rm evo}\cup \left(\bigcup_{t=1}^3 D_{{\rm conf},t}\right)$: $40/40/41$ tasks for OfficeQA and $27/27/26$ for Syn-Ledger.\\
Performance gate
& $R=2$ and $\epsilon_{\rm HE}=0.025$.
& $R=2$ and $\epsilon_{\rm HE}=0.05$. \\
\bottomrule
\end{tabularx}
\end{table}
After each of our three epochs, the score on $D_{\rm disc}$ selects one harness from the pool containing $H_0$ and all distinct accepted harnesses obtained so far. The selections after epochs 1 and 2 initialize the next epoch, whereas the selection after epoch 3 is the final harness used for certification. Within each benchmark, our budget-constrained adaptation of HarnessEvolve uses the same initial harness, foundation model, tools, scorer, and final certification procedure as the other methods.

\FloatBarrier
\renewcommand{\theequation}{C.\arabic{equation}}
\renewcommand{\thetable}{C.\arabic{table}}
\renewcommand{\thefigure}{C.\arabic{figure}}
\setcounter{equation}{0}
\setcounter{table}{0}
\setcounter{figure}{0}

\section{Additional Results on OfficeQA}
\label{app:benchmark-details}

Section~\ref{app:officeqa-preprocess} describes corpus preprocessing and data allocation, and Section~\ref{app:analysis} specifies evaluation and reporting. Section~\ref{app:officeqa-additional-results} presents the OfficeQA counterfactual and its implementation. Section~\ref{app:resource-use} reports token use during pre-certification optimization and certification. {Section~\ref{app:officeqa-revision} reports paired confidence intervals, a counterfactual withheld from optimization, and a sensitivity analysis that relocates table context.}

\subsection{Preprocessing and Task Allocation}
\label{app:officeqa-preprocess}

OfficeQA Full contains 246 questions, which use the same corpus of 697 documents from the U.S. Treasury Bulletin collection spanning 1939--2025~\citep{opsahlong2026officeqa,databricks2026officeqa}. Each question record includes an identifier, question, answer, source URLs, corresponding source files, and difficulty label. The corpus is released as original PDFs, parsed JSON, and transformed text with tables represented in Markdown. We use the transformed-text corpus throughout. The original PDFs and parsed JSON are not used. Question identifiers and texts are unique within the release.

We examine the question texts and corpus layout to quantify the {protocol correlation} described in Section~\ref{sec:introduction} (Table~\ref{tab:officeqa-raw-layout}). A question mentions a numerical scale if it contains the whole word \emph{thousand}, \emph{million}, \emph{billion}, or \emph{trillion}, allowing plural forms and ignoring case. In the corpus, we count tables with at least two columns, a header, a delimiter row, and a data row. This yields 94,303 tables, including contents tables, repeated tables across bulletin editions, and separately rendered parts of longer tables. We identify unit statements by matching non-table lines containing explicit numerical-scale or unit expressions involving hundreds, thousands, millions, billions, trillions, dollar(s), cent(s), percent/percentage(s), basis points, ounces, or units. Matching ignores case and allows leading heading markers, an opening bracket, and an Amounts, Figures, or Dollar Amounts prefix, optionally preceded by All. This rule identifies 66,951 unit statements. A Note/Source line begins with either label and a colon, period, or dash, allowing plurals and leading heading markers. Adjacency ignores blank lines only.

\begin{table}[htbp]
\centering
\small
\caption{Descriptive layout {statistics} in OfficeQA Full and its 697-document text corpus.}
\label{tab:officeqa-raw-layout}
\begin{tabular}{@{}lrr@{}}
\toprule
Criterion & Count / denominator & Percentage \\
\midrule
Questions mentioning numerical scales
& 143 / 246 & 58.1\% \\
Identified unit statements immediately before a table
& 63,747 / 66,951 & 95.2\% \\
Table blocks preceded by an identified unit statement
& 63,747 / 94,303 & 67.6\% \\
Table blocks followed by a Note/Source line
& 25,253 / 94,303 & 26.8\% \\
\bottomrule
\end{tabular}
\end{table}

All 697 documents contain examples of units before tables and Note/Source lines after tables. These counts describe marked text and its position. Notes may apply to a table group or an entire bulletin. For a concrete example, Table FFO-7 in the April 1980 bulletin places its million-dollar unit above the table and explains below it that interfund payments are excluded when calculating trust-fund receipts and outlays. Reading near the table boundaries can therefore supply both the numerical scale and an aggregation rule. The question count records explicit scale mentions. It does not identify which questions require a particular adjacent note.

OfficeQA Pro V2 contains 90 questions over a separate corpus of 1,435 parsed documents~\citep{databricks2026officeqa}. We apply the released conversion functions to all parsed inputs and freeze the resulting text before evaluation. The conversion does not load questions, answers, gold sources, or model outputs. Pro V2 results are analyzed separately from the primary OfficeQA experiment.

For each question, we collect all annotated source filenames, including those listed directly in the record and those recovered from its official source URLs. This yields 296 unique source filenames across the 246 questions in OfficeQA Full. We use shared source files to group related questions into 100 source components. The questions are then divided into $D_{\rm evo}$ (49 questions), $D_{\rm disc}$ (49 questions), three round-specific confirmation sets $D_{{\rm conf},t}$ (24 questions each), and $D_{\rm cert}$ (76 questions). The source components assigned to $D_{\rm evo}$ and $D_{\rm cert}$ do not appear in any other split. By design, $D_{\rm disc}$ and the confirmation sets may share source components: each confirmation set contains 11 questions from components that also appear in $D_{\rm disc}$ and 13 questions from components that do not appear in $D_{\rm disc}$. Questions never repeat across splits. Each confirmation set contains 11 easy and 13 hard questions. The allocation uses only question identifiers, difficulty labels, and source membership, without using model outcomes. The certification set $D_{\rm cert}$ is accessed only after the final harnesses, $\mathcal A_3$, and the analysis code are fixed.

A purely random question-level split would ignore the fact that multiple OfficeQA questions can rely on the same source files, making the comparison sensitive to accidental source overlap. We therefore use source components to control how related questions are distributed. We test whether a discovered benchmark-wide shortcut reproduces on new questions involving related sources, and whether the same shortcut extends beyond the sources used for discovery.

\subsection{Evaluation and Reporting}
\label{app:analysis}

We compute all OfficeQA correctness indicators using the released \texttt{fuzzy\_match\_answer} function at \(0\%\) numerical tolerance, following the scoring procedure in \citet{alzubi2026evoskill}. The same answer normalization and scoring rule are held fixed across $\Brel$ and every counterfactual benchmark $B_b$. Gold source annotations are not exposed to the foundation model and do not enter the score.

The primary certification evaluates the final harness $H$ from each method in the main comparison---\methodraw{}, \methodmeta{}, \methodhc{}, \methodevolve{} and \methodchase{}---on the 76 certification questions under every benchmark in $\mathcal A_3$. The archive contains $\Brel$ and the counterfactuals confirmed during \methodchase{} optimization. This evaluation contains $5\times76\times|\mathcal A_3|$ method--question--benchmark combinations, with three rollouts per combination. Rollout scores are averaged within question and then across questions. OfficeQA Pro V2 is evaluated separately under its released protocol, with three rollouts per method--question pair. Its score is denoted by $\widehat R_{\rm ProV2}$. We also report the archive summaries $\widehat R_{{\rm avg},\mathcal A_3}$ and $\widehat R_{\min,\mathcal A_3}$, using the shorthand defined in the main text.

\subsection{Tracing Counterfactual Harness Search and Evolution}
\label{app:officeqa-additional-results}

We examine $H_1$, $B_{b_1}$, and the harness candidates proposed in rounds 2--3 of the OfficeQA experiment.

\noindent\textbf{The first-round Proposer and a possible shortcut.} The initial harness $H_0$ is specified in Table~\ref{tab:initial-harness}. Table~\ref{tab:officeqa-harness-changes} reports the changes in $H_1$. Its tool-call budget remains unchanged from $H_0$. $H_1$ only adds model instructions and a conditional search-query rewrite.

\begin{table}[htbp]
\centering
\caption{Changes from $H_0$ to $H_1$.}
\label{tab:officeqa-harness-changes}
\small
\setlength{\tabcolsep}{4pt}
\renewcommand{\arraystretch}{1.12}
\begin{tabularx}{\textwidth}{@{}p{0.20\textwidth}Y@{}}
\toprule
Component & Changes from $H_0$ \\
\midrule
Prompt and memory &
Adds reminders for recovery and numerical checking. Once the needed values are found, the model is encouraged to verify or calculate rather than continue searching. \\

Conditional guidance &
Adds instructions for tracking units, converting scales, preserving numerical precision, and returning answers in the required order and format. \\

Search-query rewrite &
When at least four tool calls remain, removes stopwords, keeps up to eight terms of at least three characters, and changes search from requiring all terms to allowing any term. \\

{Policy-graph entries} &
{Adds a reserve of two reads and a stop once the evidence is complete. The serialized policy graph is part of the model's instructions, but the runtime does not execute these entries (Appendix~\ref{app:agent-harnesses}).} \\
\bottomrule
\end{tabularx}
\end{table}
$H_1$ broadens retrieval by allowing any query term to match and encourages the model to verify or calculate once the needed values appear to have been found. Its numerical guidance emphasizes units and precision, but does not explicitly tie each value to its table context, such as headers or nearby notes. We therefore hypothesize that the regular placement of such context in OfficeQA may itself signal that enough evidence has been found. Relocating that context could then change which values the model uses or whether it continues reading.

\noindent\textbf{The first-round Challenger and the counterfactual.} The Challenger returns a natural-language proposal together with a typed transformation specification. As illustrated in Figure~\ref{fig:officeqa-context-appendix}, the proposal is to collect a table's title, unit line, footnotes, source notes, scope explanations, and surrounding explanatory prose before the table, targeting $H_1$'s possible reliance on their customary locations.

\begin{figure}[!htb]
\centering
\includegraphics[page=1,width=0.90\textwidth]{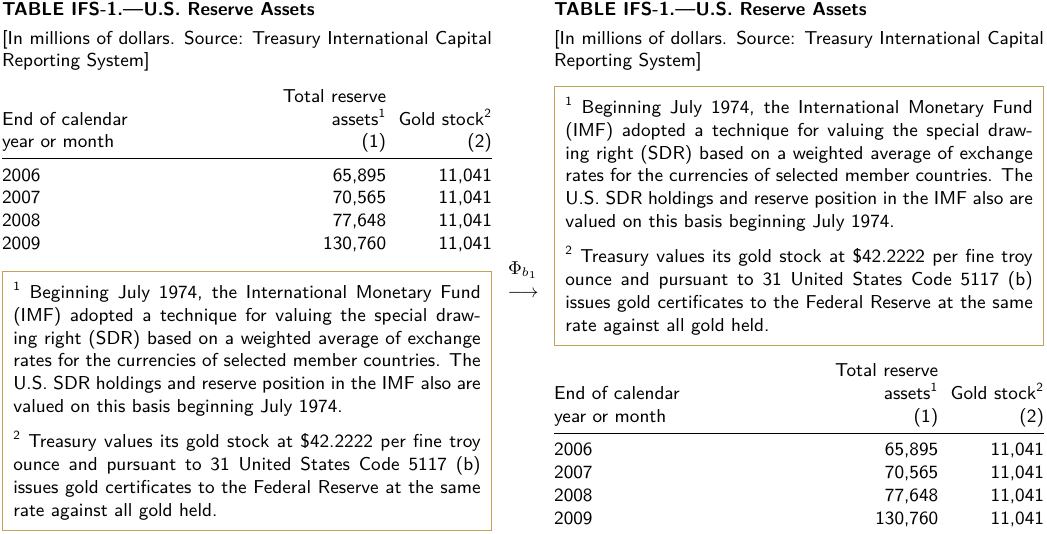}\par\smallskip
\includegraphics[page=2,width=0.90\textwidth]{Figures/OfficeQA/officeqa_additional_examples.pdf}
\caption{Schematic illustration of the Challenger's JSON table-context re-encoding proposal on two OfficeQA excerpts retypeset for readability from the September 2011 \emph{Treasury Bulletin}: IFS-1 (top) and ESF-2 (bottom). Highlighted notes are shown above their tables only to visualize the \texttt{placement="before"} field. The evaluated transformation preserves the complete text string and does not reorder content within it.}
\label{fig:officeqa-context-appendix}
\end{figure}

For evaluation, we execute the Challenger-generated transformation $\Phi_{b_1}$ on the model-visible JSON search results. For each result containing a \texttt{text} field, the code moves the complete string into \texttt{table\_context.content}, sets \texttt{table\_context.placement} to \texttt{"before"}, and removes the original \texttt{text} field. The string's internal order, document identifier, and line references are preserved. No table parsing or within-string reordering is performed. This transformation applies to every text-bearing search result, changing its field names and nesting while leaving the underlying documents, result counts, read limits, tool-call budgets, questions, and scoring rule unchanged.

On $D_{\rm disc}$, the gain destruction under $\Phi_{b_1}$ is 7.14\%, compared with 4.08\% for the other candidate transformation, a ten-line passage-boundary shift, which moves each eligible internal passage boundary ten lines later while preserving complete document coverage and the original line order. We therefore select $\Phi_{b_1}$ for evaluation on $D_{{\rm conf},1}$. Across five rollouts for each of its 24 questions under $B_{b_1}$, the estimated gain destruction is 14\%, exceeding the 7.5\% confirmation threshold, and $B_{b_1}$ is added to the archive.

\noindent\textbf{Second-round candidates and fallback to $H_0$.} In the second round, the Proposer generates three revisions of $H_1$ (Table~\ref{tab:officeqa-round2-designs}). The guidance candidate most directly responds to this concern by instructing the model to track the table context of each value it uses. The structural and integrated candidates instead emphasize combining evidence across periods and documents. All three retain $H_1$'s prompt, memory, retrieval settings, and tool-use guidance.

\begin{table}[htbp]
\centering
\caption{Candidate pool of the second-round Proposer.}
\label{tab:officeqa-round2-designs}
\small
\setlength{\tabcolsep}{4pt}
\renewcommand{\arraystretch}{1.12}
\begin{tabularx}{\textwidth}{@{}p{0.16\textwidth}Y@{}}
\toprule
Candidate & Main changes \\
\midrule
Structural & Decompose multi-period queries, merge results in round-robin order, promote document diversity, and combine needed values across documents. \\
Guidance & Bind each value to its title, date, unit, row label, and column header. Read adjacent headers or footnotes before using totals, subtotals, or year-to-date columns. \\
Integrated & Track each needed value by its fiscal year, calendar year, publication date, period, and whether it is cumulative; add rules for numerical and multi-period questions. \\
\bottomrule
\end{tabularx}
\end{table}

For the selection analysis, we apply $\varepsilon=0.05$. Each harness is evaluated on $D_{\rm evo}$ under both $\Brel$ and $B_{b_1}$. We compute $\widehat\Delta_{b_1}=\widehat G_{\rm rel}-\widehat G_{b_1}$ and require $\widehat\Delta_{b_1}\le0.05$. For the second-round and third-round selections, the archive contains $\Brel$ and $B_{b_1}$, so $B_{b_1}$ provides the only nontrivial constraint, as shown in Table~\ref{tab:officeqa-round2-selection}. All three candidates and $H_1$ fail, giving $H_2=H_0$. The structural and guidance candidates come closest, but each loses 6.12\%. Since the second-round selection returned $H_2=H_0$, CHASE skipped Challenger generation, discovery, and confirmation in this round.

\begin{table}[htbp]
\centering
\caption{Second-round selection at $\varepsilon=0.05$. All values are in \%.}
\label{tab:officeqa-round2-selection}
\small
\setlength{\tabcolsep}{4pt}
\begin{tabular*}{\textwidth}{@{\extracolsep{\fill}}lrrrc@{}}
\toprule
Harness & $\widehat G_{\rm rel}$ & $\widehat G_{b_1}$ & $\widehat\Delta_{b_1}$ & Pass \\
\midrule
$H_1$ & $+8.16$ & $-5.10$ & $13.27$ & No \\
Structural & $+15.31$ & $+9.18$ & $6.12$ & No \\
Guidance & $+9.18$ & $+3.06$ & $6.12$ & No \\
Integrated & $+15.31$ & $-2.04$ & $17.35$ & No \\
\bottomrule
\end{tabular*}
\end{table}

\noindent\textbf{Third-round candidates after failure feedback.} The third-round Proposer starts from $H_2=H_0$. Its input reports the second-round constraint failures and fallback, and explicitly states that a higher released-benchmark score alone is insufficient. Table~\ref{tab:officeqa-round3-designs} summarizes the candidates. All three retain $H_0$'s generic prompt and empty memory, rather than inheriting $H_1$'s instruction to stop repeated searching once the needed values appear to have been found.

\begin{table}[htbp]
\centering
\caption{Candidate pool of the third-round Proposer.}
\label{tab:officeqa-round3-designs}
\small
\setlength{\tabcolsep}{4pt}
\renewcommand{\arraystretch}{1.12}
\begin{tabularx}{\textwidth}{@{}p{0.14\textwidth}Y@{}}
\toprule
Candidate & Main changes \\
\midrule
Structural & Instructs the model to search separately for different periods, combine search results, and assemble evidence across documents, without changing the underlying retrieval procedure. \\
Guidance & Adds conditional instructions to convert source values to the requested units, preserve full precision during calculation, and follow the requested rounding order and answer format. \\
Integrated & Combines these types of numerical guidance with revised retrieval: retain all query terms except stopwords, first require all retained terms to match, and allow matches on any term only when the initial search has low coverage. \\
\bottomrule
\end{tabularx}
\end{table}
Because all second-round candidates violate the archive constraint, the third round starts from $H_2=H_0$ while retaining feedback from the first-round counterfactual and the second-round constraint failures. The integrated candidate makes the clearest change to the earlier retrieval strategy: it preserves the query's specificity and uses broad matching as a recovery step rather than at the outset. It combines this change with explicit numerical guidance, {without inheriting $H_1$'s switch to any-term matching at the outset}. The resulting design therefore revisits both how evidence is retrieved and how the model is instructed to proceed after finding values, rather than simply adding another reminder about table context. This sequence illustrates how counterfactual feedback can inform subsequent harness design, not merely reject candidates with high released-benchmark scores.

As shown in Table~\ref{tab:officeqa-round3-selection}, the integrated candidate achieves the highest released-benchmark score among the candidates and is therefore selected as $H_3$. In contrast to the larger gains that violated the archive constraint in round two, the third-round revision yields a smaller released gain that satisfies the constraint.
\begin{table}[htbp]
\centering
\caption{Third-round selection at $\varepsilon=0.05$. All values are in \%.}
\label{tab:officeqa-round3-selection}
\small
\setlength{\tabcolsep}{4pt}
\begin{tabular*}{\textwidth}{@{\extracolsep{\fill}}lrrrc@{}}
\toprule
Harness  & $\widehat G_{\rm rel}$ & $\widehat G_{b_1}$ & $\widehat\Delta_{b_1}$ & Pass \\
\midrule
Structural  & $-1.02$ & $-1.02$ & $0.00$ & Yes \\
Guidance  & $-2.04$ & $-6.12$ & $4.08$ & Yes \\
Integrated  & $+4.08$ & $0.00$ & $4.08$ & Yes \\
\bottomrule
\end{tabular*}
\end{table}

Figure~\ref{fig:officeqa-candidate-selection} illustrates how the archived counterfactual shapes candidate selection. The second-round candidates achieve larger released gains but violate the archive constraint, whereas the third-round integrated candidate provides a smaller released gain within the same tolerance.
\begin{figure}[htbp]
    \centering
    \includegraphics[width=\textwidth]{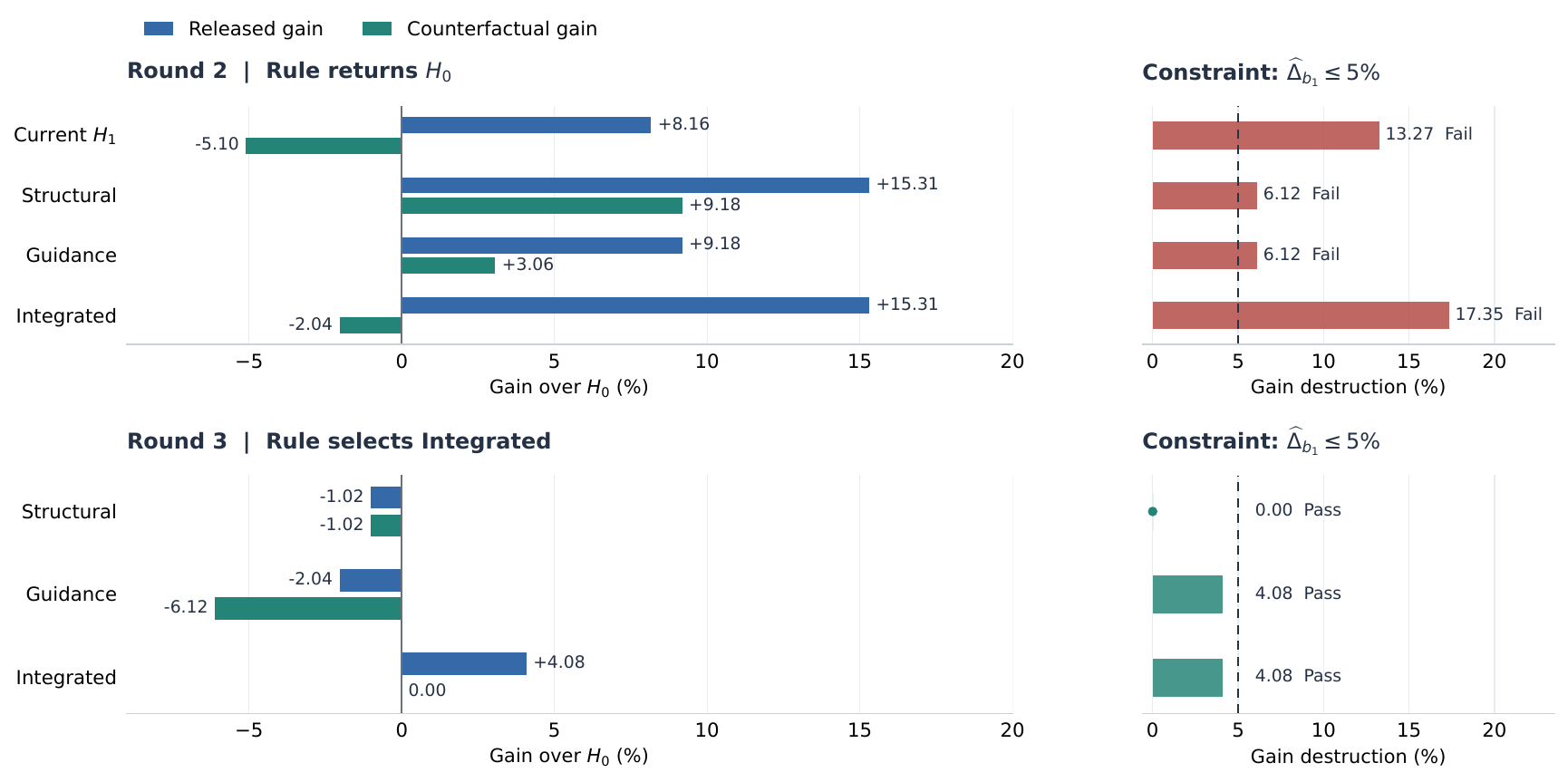}
    \caption{OfficeQA candidate selection in rounds two and three. Left: released and counterfactual gains over $H_0$. Right: gain destruction and the archive constraint $\varepsilon=0.05$.}
    \label{fig:officeqa-candidate-selection}
\end{figure}

After selecting $H_3$, the third-round Challenger evaluates two transformations on $D_{\rm disc}$. A ten-line passage-boundary shift mentioned above produces a gain destruction of 10.20\% and is selected for evaluation on $D_{{\rm conf},3}$. Across five rollouts on each of its 24 questions, the estimated gain destruction is 4.17\%, which is below the 7.5\% confirmation threshold. The transformation is therefore not added to the archive.

\subsection{Token Use}
\label{app:resource-use}

We compare certification token use over 456 evaluation rollouts per method, with three rollouts per evaluated question--benchmark pair. Table~\ref{tab:certification-cost} reports total token use (input plus output tokens) and token use excluding cached input. The latter counts only uncached input and output tokens. Among the five methods, \methodraw{} used the fewest total tokens (93.82 million), whereas \methodmeta{} used the fewest tokens excluding cached input (20.68 million).

\begin{table}[htbp]
\centering
\caption{Token use during certification with three rollouts per evaluated question--benchmark pair.}
\label{tab:certification-cost}
\small
\setlength{\tabcolsep}{9pt}
\begin{tabular}{lrr}
\toprule
& \multicolumn{2}{c}{Token use (millions)} \\
\cmidrule(lr){2-3}
Method
& Total
& Excluding cached input \\
\midrule
\methodraw{}    & 93.82  & 24.30 \\
\methodmeta{}   & 106.71 & 20.68 \\
\methodhc{}     & 118.50 & 24.96 \\
\methodevolve{} & 132.60 & 23.36 \\
\methodchase{}  & 98.15  & 22.39 \\
\bottomrule
\end{tabular}
\end{table}

We report certification separately because \methodraw{} has no optimization stage, and the four optimized methods run different numbers and types of optimization episodes. An episode is one end-to-end execution of one model role on one assigned input. An episode may contain multiple model requests. Table~\ref{tab:search-cost} includes every episode in the reported optimization and selection procedures with recorded token use, whether or not its candidate was retained. These totals describe the distinct optimization procedures rather than an episode-matched comparison. \methodmeta{} used 224.31 million total tokens, whereas \methodevolve{} used 574.16 million. Compared with \methodhc{}, \methodchase{} used 2.43 times as many total tokens. Its optimization procedure also included Challenger discovery and confirmation.

\begin{table}[htbp]
\centering
\caption{Token use during optimization and selection. All episodes with recorded token use are included, whether or not their candidates are retained.}
\label{tab:search-cost}
\small
\setlength{\tabcolsep}{7pt}
\begin{tabular}{lrrr}
\toprule
& & \multicolumn{2}{c}{Token use (millions)} \\
\cmidrule(lr){3-4}
Method
& Episodes
& Total
& Excluding cached input \\
\midrule
\methodmeta{}   & 983   & 224.31 & 45.74  \\
\methodhc{}     & 1,390 & 283.02 & 58.59  \\
\methodevolve{} & 2,214 & 574.16 & 106.75 \\
\methodchase{}  & 3,439 & 688.96 & 154.42 \\
\bottomrule
\end{tabular}
\end{table}

\subsection{\texorpdfstring{{Additional Analyses}}{Additional Analyses}}
\label{app:officeqa-revision}

\noindent{{\bf Paired confidence intervals.} We quantify the uncertainty of Table~\ref{tab:main-results} with a question-level paired bootstrap. For each method, benchmark, and question, we first average the three rollout scores. We then draw $20{,}000$ bootstrap samples of the questions with replacement, use the same sampled questions for all methods and benchmarks, and recompute every metric of Table~\ref{tab:main-results} on each sample. Table~\ref{tab:main-results} reports the 95\% percentile interval of each score, and Table~\ref{tab:officeqa-ci-diff} reports the 95\% percentile interval of the difference between \methodchase{} and each baseline.}

{The intervals exclude zero for the advantage of \methodchase{} over \methodhc{} in $\widehat R_{{\rm avg},\mathcal A_3}$ and $\widehat R_{\min,\mathcal A_3}$, and for the advantage of \methodchase{} over \methodevolve{} on Pro V2.}

\begin{table}[htbp]
\centering
\caption{Differences between \methodchase{} and each baseline (\methodchase{} minus baseline; in \%). Small bracketed numbers give question-level paired bootstrap 95\% confidence intervals. Bold entries have intervals that exclude zero.}
\label{tab:officeqa-ci-diff}
\small
\setlength{\tabcolsep}{0pt}
\begin{tabular*}{\textwidth}{@{\extracolsep{\fill}}l r@{\extracolsep{0pt}\hspace{3pt}}l@{\extracolsep{\fill}} r@{\extracolsep{0pt}\hspace{3pt}}l@{\extracolsep{\fill}} r@{\extracolsep{0pt}\hspace{3pt}}l@{\extracolsep{\fill}} r@{\extracolsep{0pt}\hspace{3pt}}l@{}}
\toprule
& \multicolumn{2}{c}{OfficeQA Pro V2} & \multicolumn{6}{c}{Certification of OfficeQA Full} \\
\cmidrule(lr){2-3}\cmidrule(l){4-9}
\methodchase{} & \multicolumn{2}{c}{$\widehat R_{\rm ProV2}$} & \multicolumn{2}{c}{$\widehat R_{\rm rel}$} & \multicolumn{2}{c}{$\widehat R_{{\rm avg},\mathcal A_3}$} & \multicolumn{2}{c}{$\widehat R_{\min,\mathcal A_3}$} \\
\midrule
$-$\,\methodraw{} & +3.33 & \ci{$-$1.48}{8.52} & +0.88 & \ci{$-$3.51}{5.26} & +2.19 & \ci{$-$1.32}{5.48} & +3.51 & \ci{$-$1.32}{7.46} \\
$-$\,\methodmeta{} & +1.11 & \ci{$-$3.70}{5.93} & +5.26 & \ci{0.00}{10.96} & +4.39 & \ci{0.00}{9.21} & +4.39 & \ci{$-$0.44}{9.65} \\
$-$\,\methodhc{} & +4.07 & \ci{$-$0.37}{8.52} & +4.82 & \ci{$-$0.88}{10.96} & \textbf{+5.26} & \ci{\textbf{0.88}}{\textbf{10.09}} & \textbf{+5.70} & \ci{\textbf{0.44}}{\textbf{10.53}} \\
$-$\,\methodevolve{} & \textbf{+6.30} & \ci{\textbf{0.37}}{\textbf{12.59}} & $-$0.44 & \ci{$-$5.70}{4.82} & +0.22 & \ci{$-$3.95}{4.61} & +0.88 & \ci{$-$4.39}{5.71} \\
\bottomrule
\end{tabular*}
\end{table}

\paragraph{{A counterfactual withheld from optimization.}}
{To test a protocol change that the archive does not contain, we construct a column-reversal counterfactual by hand: in every table whose data columns carry their own headers within the retrieved text, the data columns appear in reverse order, and each header moves with its column. Every cell keeps its row label and column header, and the documents, questions, answers, and scoring are unchanged. The released transformation checks every table that it changes: the change is exactly invertible, and each row keeps its tokens, its byte length, and the association of every cell with its row label and column header. $H_3$, whose search never saw this counterfactual, already meets its constraint on $D_{\rm cert}$ ($\widehat\Delta=0.88\%\le\varepsilon$). We then add the counterfactual to the archive and run Proposer rounds from $H_3$ under the enlarged archive. The first candidate with a positive released gain on $D_{\rm evo}$ that satisfies both constraints on $D_{\rm evo}$ is $H_4$, and it is frozen before certification. Table~\ref{tab:officeqa-me1} reports $H_0$, $H_3$, and $H_4$ on $D_{\rm cert}$. On $D_{\rm cert}$, $H_4$ keeps a positive released gain ($+2.19\%$), its gain destruction is $-2.63\%$ under $B_{b_1}$ and $1.75\%$ under column reversal, both within $\varepsilon$, and it attains the best score under each protocol.}

\begin{table}[htbp]
\centering
\caption{Certification on $D_{\rm cert}$ under the released benchmark, the archived counterfactual $B_{b_1}$, and the withheld column-reversal counterfactual (76 questions, three rollouts each). All scores are in \%. Small brackets give question-level paired bootstrap 95\% confidence intervals. Bold marks the best score in each column.}
\label{tab:officeqa-me1}
\small
\setlength{\tabcolsep}{0pt}
\begin{tabular*}{\linewidth}{@{\extracolsep{\fill}}l r@{\extracolsep{0pt}\hspace{3pt}}l@{\extracolsep{\fill}} r@{\extracolsep{0pt}\hspace{3pt}}l@{\extracolsep{\fill}} r@{\extracolsep{0pt}\hspace{3pt}}l}
\toprule
Harness & \multicolumn{2}{c}{$\Brel$} & \multicolumn{2}{c}{$B_{b_1}$} & \multicolumn{2}{c}{Column reversal} \\
\midrule
$H_0$ & 67.98 & \ci{59.21}{76.75} & 64.47 & \ci{54.82}{73.68} & 67.11 & \ci{57.89}{75.88} \\
$H_3$ & 68.86 & \ci{59.21}{78.07} & 67.98 & \ci{58.77}{76.75} & 67.11 & \ci{57.46}{76.75} \\
$H_4$ & \textbf{70.18} & \ci{60.96}{78.95} & \textbf{69.30} & \ci{60.09}{78.07} & \textbf{67.54} & \ci{57.89}{76.75} \\
\bottomrule
\end{tabular*}
\end{table}

\paragraph{{Relocating the table context.}}
\looseness=-1 {The executed counterfactual $B_{b_1}$ re-encodes retrieved text as table context in the search output (Appendix~\ref{app:officeqa-additional-results}). As a sensitivity analysis, we construct a relocation counterfactual that adds the move of the Challenger's proposal (Figure~\ref{fig:officeqa-counterexample}): it applies $B_{b_1}$ and moves the unit and note context of each fully visible table before the table, and keeps the source text, the retrieval ranking, the line numbers, and the budgets. Table~\ref{tab:officeqa-me2} reports, for each harness, the absolute difference between its scores under the two counterfactuals, with ten rollouts per question. The difference is $0.00\%$ for $H_1$, the harness that the Challenger targeted, with a paired 95\% interval of $[-3.88, 3.67]$, and below $1\%$ for $H_0$ and $H_3$ on $D_{\rm cert}$. Adding the relocation therefore changes the scores little.}

\begin{table}[htbp]
\centering
\caption{Absolute difference between the scores under the relocation counterfactual and under the executed counterfactual $B_{b_1}$, with ten rollouts per question. All values are in \%.}
\label{tab:officeqa-me2}
\small
\begin{tabular}{llcc}
\toprule
Harness & Task set & Questions & Absolute difference \\
\midrule
$H_1$ & $D_{\rm evo}$ & 49 & 0.00 \\
$H_0$ & $D_{\rm cert}$ & 76 & 0.92 \\
$H_3$ & $D_{\rm cert}$ & 76 & 0.66 \\
\bottomrule
\end{tabular}
\end{table}

\FloatBarrier
\renewcommand{\theequation}{D.\arabic{equation}}
\renewcommand{\thetable}{D.\arabic{table}}
\renewcommand{\thefigure}{D.\arabic{figure}}
\setcounter{equation}{0}
\setcounter{table}{0}
\setcounter{figure}{0}

\section{Additional Results on Syn-Ledger}
\label{app:synledger}

Syn-Ledger contains 320 multi-document arithmetic tasks with controlled benchmark-wide shortcuts. Section~\ref{app:synledger-construction} describes task construction, the agent interface, and scoring. Section~\ref{app:synledger-protocols} describes benchmark-wide shortcut construction and validation, and Section~\ref{app:synledger-split} gives the task allocation. {Section~\ref{app:synledger-revision} reports additional analyses of the certification results.}

\subsection{Task Construction, Agent Interface, and Scoring}
\label{app:synledger-construction}

A task consists of a question, twelve ledger documents, and an arithmetic program that determines the answer. Each ledger record specifies an entity, period, accounting category, approval status, version, accounting basis, and amount. The question identifies the requested entity and relevant periods and restricts the calculation to approved records on the enacted basis. Two evidence documents supply the required amounts. The ten distractors comprise two documents for periods not required by the task, two for an incorrect entity, two with superseded records, two with cancelled or draft records, one with an incorrect accounting basis, and one with an irrelevant category. Each document has six records, including contextual rows outside the requested entity or period.

Table~\ref{tab:synledger-tasks} lists the five task families and shows how their required amounts are distributed across the two evidence documents. Each program contributes 64 tasks and uses the same set of 64 distinct two-digit positive answers. The generator first fixes the answer and then samples amounts satisfying the corresponding arithmetic relation. Combined-total tasks use signed adjustments, and integer-share tasks satisfy $100x/y\in\mathbb Z$. Distractor amounts are sampled without using the required amounts. A task is rejected if its answer appears as a complete integer token anywhere in the visible task inputs. The correctness of Syn-Ledger is verified through a separate validation procedure described in Section~\ref{app:synledger-protocols}.

\begin{table}[htbp]
\centering
\caption{Syn-Ledger task families. The symbols $x,y,z$ denote the amounts required by each calculation. All answers are integers.}
\label{tab:synledger-tasks}
\small
\setlength{\tabcolsep}{4pt}
\begin{tabularx}{\textwidth}{@{}lcYY@{}}
\toprule
Task family & Calculation & First evidence document & Second evidence document \\
\midrule
Period difference & $x-y$ & Current appropriation $x$ & Previous appropriation $y$ \\
Combined total & $x+y$ & Capital adjustment $x$ & Operations adjustment $y$ \\
Net balance & $x+y-z$ & Appropriation $x$, adjustment $y$ & Obligation $z$ \\
Integer share & $100x/y$ & Allocated amount $x$ & Reference total $y$ \\
Reconciliation gap & $z-(x+y)$ & Reported total $z$ & Component amounts $x,y$ \\
\bottomrule
\end{tabularx}
\end{table}

The model $\Agent$ retrieves evidence through \texttt{list\_files}, \texttt{search}, and \texttt{open\_file}, and can evaluate arithmetic with \texttt{calculator}. Each task allows at most twelve tool calls: four search/list calls in total, six file-opening calls, and two calculator calls. A file-opening call returns three records. Another call reads the next three. Search returns four results by default and permits up to twelve.

Scoring uses normalized integer exact match. The scorer normalizes surrounding whitespace, signs, and valid comma grouping, then compares the submitted integer with the ground-truth answer. Explanations, multiple numbers, and decimal outputs receive zero. All benchmark variants $B_b$ use the same tools, budgets, and scorer.

Each document is generated from a designated candidate-record block and additional contextual rows, for a total of six records. In an evidence document, the candidate block contains the one or two records required by the arithmetic program. In a distractor, the corresponding block contains structurally matched decoy records that are excluded by the entity, period, approval-status, version, accounting-basis, or category conditions in the question. The remaining records provide context but do not satisfy the requested conditions.

\subsection{Benchmark-wide Shortcut Construction and Validation}
\label{app:synledger-protocols}

\looseness=-1 To construct benchmark-wide shortcuts without changing task semantics, we vary five observable features of each document: filename, directory depth, search rank, candidate-record position, and serialization format (Table~\ref{tab:synledger-protocols}). For each feature, we define a favorable and an unfavorable level. Under the released benchmark, required evidence is more often assigned favorable levels, creating associations that a harness can exploit without identifying evidence from document contents. For example, evidence documents may tend to have shorter filenames, shallower directories, or earlier search ranks. These assignments are stored separately from the ledger contents in a rendering manifest and applied by a deterministic renderer. Changing the manifest therefore changes only these observable features while preserving the question, ground-truth answer, and underlying ledger records.

\begin{table}[htbp]
\centering
\caption{Observable document features used to construct benchmark-wide shortcuts in Syn-Ledger. Search-rank manipulations change only the ordering of broad-query results and do not change document accessibility.}
\label{tab:synledger-protocols}
\small
\setlength{\tabcolsep}{4pt}
\begin{tabularx}{\textwidth}{@{}lYY@{}}
\toprule
Feature & Favorable level & Unfavorable level \\
\midrule
Filename & Short, regular, lexically early & Longer, lexically late \\
Directory depth & Shallow path & Deeper path \\
Search rank & Earlier in broad-query results & Later in broad-query results \\
Candidate-record position & Designated candidate-record block in the first \texttt{open\_file} window & Designated candidate-record block in the continuation \\
Serialization & Fixed-order Markdown table & Equivalent CSV or key--value records \\
\bottomrule
\end{tabularx}
\end{table}
Under $B_0$, the five observable features are balanced with respect to evidence membership. In each task, one evidence document and five distractors share one feature configuration, while the other evidence document and five distractors take the opposite level of all five features. Across each 32-task block, the feature configuration assigned to one evidence document cycles through all $2^5$ possible combinations. Consequently, every feature configuration occurs twice among the 64 evidence documents and ten times among the 320 distractors. Thus neither any individual feature nor any interaction among the five features is associated with evidence membership under $B_0$.

To construct $\Brel$, we modify the feature levels assigned to the two evidence documents while keeping all distractor assignments fixed. For each observable feature, both evidence documents take the favorable level in 24 of the 32 tasks, while the remaining 8 tasks retain one favorable and one unfavorable evidence document. Hence, $56/64=7/8$ of the evidence documents have the favorable level for each feature. The eight tasks with one favorable and one unfavorable evidence document are chosen differently for the five features, preventing their favorable assignments from always occurring on the same tasks.

For each observable feature $j\in\{1,\ldots,5\}$, we define a canonical neutralization $\Phi_j$ that replaces its assignment under $\Brel$ with the corresponding assignment under $B_0$, while leaving the other four features unchanged. Each $\Phi_j$ is deterministic and idempotent, the five canonical neutralizations mutually commute, and applying all five recovers:
\[
    B_0=(\Phi_5\circ\cdots\circ\Phi_1)(\Brel).
\]
We additionally define five placebo transformations that alter irrelevant presentation details for construction audits: line endings, metadata order, trailing whitespace, low-rank distractor order, and section labels. The distractor-ordering placebo leaves the rank of every evidence document unchanged. The five canonical neutralizations and five placebo transformations are used only to audit the benchmark construction. They are not revealed to the Challenger, which proposes transformations~dynamically.

\noindent{\bf Validation of Syn-Ledger.} We validate Syn-Ledger independently of the optimization and evaluation rollouts. For every task, the ground-truth answer is recomputed from both the typed document records and the rendered documents under $\Brel$, $B_0$, the five canonical neutralizations, and the five placebo transformations, and the two computations must agree. We also verify that both designated evidence documents are necessary: replacing either one while keeping the question, the remaining documents, and all observable-feature assignments fixed must change the resulting ground-truth answer. Additional checks verify the intended feature balance under $\Brel$ and $B_0$, the composition of the five canonical neutralizations into $B_0$, deterministic rendering and replay, and scorer behavior on valid and malformed outputs.

\subsection{Preprocessing and Task Allocation}
\label{app:synledger-split}

The preceding construction gives each task three distinct attributes: its \emph{task family} specifies what is computed, its documents carry \emph{observable features} that create benchmark-wide shortcuts, and its \emph{task set} specifies when the task is used in the experiment. The 32-task blocks are only a construction device for balancing the five observable features.

The 320 tasks are organized into ten complete 32-task blocks, indexed 0--9. We partition them into $D_{\rm evo}$, $D_{\rm disc}$, three round-specific confirmation sets $D_{{\rm conf},t}$, and $D_{\rm cert}$, with sizes $32+32+3\times16+208=320$ (Table~\ref{tab:synledger-split}).
\begin{table}[htbp]
\centering
\caption{Syn-Ledger data allocation. Each task appears in exactly one set.}
\label{tab:synledger-split}
\small
\setlength{\tabcolsep}{5pt}
\begin{tabularx}{\textwidth}{@{}lcY@{}}
\toprule
Task set & Size & Construction \\
\midrule
$D_{\rm evo}$ & 32 & Block 0 \\
$D_{\rm disc}$ & 32 & Block 1 \\
$D_{{\rm conf},t}$ & 16 per round & 16 tasks selected from Blocks 2--7 \\
$D_{\rm cert}$ & 208 & The remaining 24 tasks from each of Blocks 2--7, plus all tasks from Blocks 8 and 9 \\
\bottomrule
\end{tabularx}
\end{table}
Without using model scores or trajectories, a deterministic allocator selects one eight-task subset from each of Blocks 2--7. For $t=1,2,3$, the subsets selected from Blocks $2t$ and $2t+1$ form $D_{{\rm conf},t}$. The 24 unselected tasks in each of these six blocks enter $D_{\rm cert}$. Blocks 0 and 1 are assigned intact to $D_{\rm evo}$ and $D_{\rm disc}$, respectively, while Blocks 8 and 9 enter $D_{\rm cert}$ intact.

Each selected subset contains one or two tasks from every task family. For each observable feature, exactly 14 of the subset's $2\times8=16$ evidence documents take the favorable level under $\Brel$. Each confirmation set therefore has $28/32=7/8$ favorable evidence documents for every feature. Because a complete block has 56 favorable evidence documents out of 64 for every feature, the 24 tasks left after selecting a valid subset have $56-14=42$ favorable evidence documents out of 48. It follows that $D_{\rm cert}$ has: $\frac{6\times42+2\times56}{6\times48+2\times64}=\frac{364}{416}=\frac{7}{8}$ favorable evidence documents for every feature. Figure~\ref{fig:synledger-design} summarizes the task construction and deterministic allocation.
\begin{figure}[htbp]
    \centering
    \includegraphics[width=\textwidth]{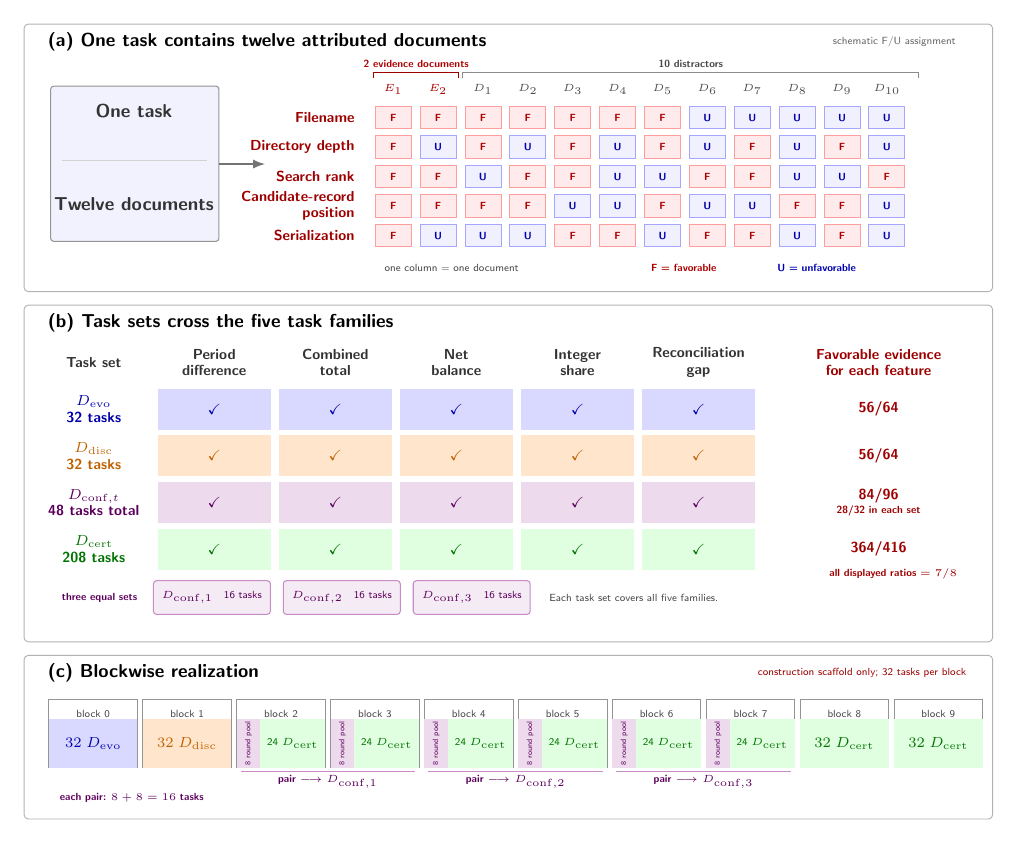}
    \caption{Syn-Ledger task construction and deterministic allocation.}
    \label{fig:synledger-design}
\end{figure}

All task sets are required to cover the five task families. Exact equality of task-family counts is neither possible nor required because 32, 16, and 208 are not all divisible by five. Among allocations satisfying the preceding allocation constraints, we minimize the range of task-family counts within each task set and then balance the joint favorable/unfavorable patterns of each pair of observable features as evenly as possible. Ties are resolved deterministically using a seeded hash of task identifiers. Before model evaluation, we verify set sizes, disjointness, task-family counts, the favorable rate for each feature, pairwise feature balance, and deterministic reconstruction.

\subsection{\texorpdfstring{{Additional Analyses}}{Additional Analyses}}
\label{app:synledger-revision}

\noindent{Because $B_0$ is known on Syn-Ledger, we evaluate the final harnesses against $B_0$ directly instead of through the counterfactuals in the archive.}

{{\bf Paired confidence intervals.} We apply the paired bootstrap of Appendix~\ref{app:officeqa-revision} to the 208 certification tasks. For each method, benchmark, and task, we first average the three rollout scores. We then draw $20{,}000$ bootstrap samples of the tasks with replacement, use the same sampled tasks for all methods and both benchmarks, and recompute every metric of Table~\ref{tab:synledger-results} on each sample. Table~\ref{tab:synledger-results} reports the 95\% percentile interval of each score, and Table~\ref{tab:synledger-ci-diff} reports the 95\% percentile interval of the difference between \methodchase{} and each baseline.}

{The advantage of \methodchase{} in $\widehat R_0$ excludes zero against every baseline, including $B_0$-Access ($+5.61\%$), and its lower $\widehat\Delta_{\rm BS}$ excludes zero against every baseline without access to $B_0$.}

\begin{table}[htbp]
\centering
\caption{Syn-Ledger differences between \methodchase{} and each baseline (\methodchase{} minus baseline; in \%). Small bracketed numbers give task-level paired bootstrap 95\% confidence intervals. Bold entries have intervals that exclude zero. Differences in $\widehat G_{\rm rel}$ and $\widehat G_0$ equal those in $\widehat R_{\rm rel}$ and $\widehat R_0$ and are omitted. \methodraw{} has $\widehat\Delta_{\rm BS}=0$ by definition.}
\label{tab:synledger-ci-diff}
\small
\setlength{\tabcolsep}{0pt}
\begin{tabular*}{\textwidth}{@{\extracolsep{\fill}}l r@{\extracolsep{0pt}\hspace{3pt}}l@{\extracolsep{\fill}} r@{\extracolsep{0pt}\hspace{3pt}}l@{\extracolsep{\fill}} r@{\extracolsep{0pt}\hspace{3pt}}l@{}}
\toprule
\methodchase{} & \multicolumn{2}{c}{$\widehat R_{\rm rel}\,(\uparrow)$} & \multicolumn{2}{c}{$\widehat R_0\,(\uparrow)$} & \multicolumn{2}{c}{$\widehat\Delta_{\rm BS}\,(\downarrow)$} \\
\midrule
$-$\,\methodraw{} & \textbf{+5.29} & \ci{\textbf{0.80}}{\textbf{9.78}} & \textbf{+23.08} & \ci{\textbf{17.47}}{\textbf{28.69}} & \textbf{{\boldmath$-$}17.79} & \ci{\textbf{{\boldmath$-$}23.88}}{\textbf{{\boldmath$-$}11.54}} \\
$-$\,\methodmeta{} & \textbf{+5.13} & \ci{\textbf{0.96}}{\textbf{9.29}} & \textbf{+18.43} & \ci{\textbf{13.30}}{\textbf{23.56}} & \textbf{{\boldmath$-$}13.30} & \ci{\textbf{{\boldmath$-$}19.71}}{\textbf{{\boldmath$-$}6.89}} \\
$-$\,\methodhc{} & \textbf{+6.73} & \ci{\textbf{2.56}}{\textbf{11.06}} & \textbf{+23.56} & \ci{\textbf{18.11}}{\textbf{29.17}} & \textbf{{\boldmath$-$}16.83} & \ci{\textbf{{\boldmath$-$}22.92}}{\textbf{{\boldmath$-$}10.58}} \\
$-$\,\methodevolve{} & \textbf{{\boldmath$-$}6.41} & \ci{\textbf{{\boldmath$-$}10.10}}{\textbf{{\boldmath$-$}2.72}} & \textbf{+20.03} & \ci{\textbf{14.74}}{\textbf{25.48}} & \textbf{{\boldmath$-$}26.44} & \ci{\textbf{{\boldmath$-$}32.05}}{\textbf{{\boldmath$-$}20.99}} \\
$-$\,$B_0$-Access & \textbf{+4.17} & \ci{\textbf{0.48}}{\textbf{7.85}} & \textbf{+5.61} & \ci{\textbf{1.28}}{\textbf{9.94}} & $-$1.44 & \ci{$-$6.89}{4.01} \\
\bottomrule
\end{tabular*}
\end{table}

{{\bf Task families.} Table~\ref{tab:synledger-families} splits $\widehat R_0$ by task family, with the same paired bootstrap applied within each family. \methodchase{} attains the highest score under $B_0$ in every family. Its advantage over each baseline without access to $B_0$ excludes zero in 18 of the 20 family--baseline pairs, and its advantage over $B_0$-Access is positive in every family and excludes zero on net-balance tasks ($+15.00\%$ $[5.00, 25.83]$). The gain destruction is equally consistent: $\widehat\Delta_{\rm BS}$ of \methodevolve{} is positive in every family, from $+2.38\%$ to $+17.50\%$, whereas $\widehat\Delta_{\rm BS}$ of \methodchase{} is negative in every family, from $-28.33\%$ to $-13.01\%$. The results of Table~\ref{tab:synledger-results} therefore do not rest on a single task family.}

\begin{table}[htbp]
\centering
\caption{Syn-Ledger scores under $B_0$ ($\widehat R_0$) on the 208 certification tasks by task family, three rollouts each. The number of tasks in each family is given in parentheses. Bold marks the best score in each column. \methodchase{} attains the highest score in every family. All scores are in \%. Small brackets give task-level paired bootstrap 95\% confidence intervals within each family.}
\label{tab:synledger-families}
\footnotesize
\renewcommand{\ci}[2]{{\tiny[#1,\,#2]}}
\setlength{\tabcolsep}{0pt}
\begin{tabular*}{\textwidth}{@{\extracolsep{\fill}}l r@{\extracolsep{0pt}\hspace{2pt}}l@{\extracolsep{\fill}} r@{\extracolsep{0pt}\hspace{2pt}}l@{\extracolsep{\fill}} r@{\extracolsep{0pt}\hspace{2pt}}l@{\extracolsep{\fill}} r@{\extracolsep{0pt}\hspace{2pt}}l@{\extracolsep{\fill}} r@{\extracolsep{0pt}\hspace{2pt}}l@{}}
\toprule
& \multicolumn{2}{c}{Period} & \multicolumn{2}{c}{Combined} & \multicolumn{2}{c}{Net} & \multicolumn{2}{c}{Integer} & \multicolumn{2}{c}{Reconciliation} \\
Method & \multicolumn{2}{c}{difference (42)} & \multicolumn{2}{c}{total (41)} & \multicolumn{2}{c}{balance (40)} & \multicolumn{2}{c}{share (42)} & \multicolumn{2}{c}{gap (43)} \\
\midrule
\methodraw{} & 5.56 & \ci{2.38}{9.52} & 1.63 & \ci{0.00}{4.07} & 32.50 & \ci{21.67}{44.17} & 7.94 & \ci{3.17}{13.49} & 28.68 & \ci{20.16}{37.21} \\
\methodmeta{} & 6.35 & \ci{2.38}{11.11} & 16.26 & \ci{7.32}{26.02} & 29.17 & \ci{19.17}{40.00} & 15.08 & \ci{7.94}{23.02} & 32.56 & \ci{21.71}{44.19} \\
\methodhc{} & 2.38 & \ci{0.00}{6.35} & 5.69 & \ci{0.81}{12.20} & 19.17 & \ci{10.00}{29.17} & 11.11 & \ci{5.56}{17.46} & 34.88 & \ci{24.03}{46.51} \\
\methodevolve{} & 9.52 & \ci{3.97}{15.87} & 8.94 & \ci{3.25}{16.26} & 26.67 & \ci{17.50}{36.67} & 17.46 & \ci{8.73}{27.78} & 28.68 & \ci{19.38}{38.76} \\
$B_0$-Access & 18.25 & \ci{8.73}{28.57} & 20.33 & \ci{11.38}{30.08} & 50.83 & \ci{39.17}{61.67} & 23.81 & \ci{15.08}{33.33} & 50.39 & \ci{36.43}{63.57} \\
\methodchase{} & \textbf{20.63} & \ci{11.11}{30.95} & \textbf{26.02} & \ci{16.26}{37.40} & \textbf{65.83} & \ci{53.33}{77.50} & \textbf{26.98} & \ci{15.87}{39.68} & \textbf{52.71} & \ci{39.53}{65.12} \\
\bottomrule
\end{tabular*}
\end{table}

\FloatBarrier
\renewcommand{\theequation}{E.\arabic{equation}}
\renewcommand{\thetable}{E.\arabic{table}}
\renewcommand{\thefigure}{E.\arabic{figure}}
\setcounter{equation}{0}
\setcounter{table}{0}
\setcounter{figure}{0}

\section{\texorpdfstring{{Additional Results on FinanceBench}}{Additional Results on FinanceBench}}
\label{app:financebench}

{This section reports two experiments on FinanceBench with gpt-5.4 as the agent model. Section~\ref{app:financebench-setup} describes the benchmark, the corpus, the data allocation, and the scoring. Section~\ref{app:financebench-evolution} runs \methodchase{} and four comparison methods on FinanceBench, traces the \methodchase{} run, and analyzes the certification results. Section~\ref{app:financebench-transfer} evaluates the final OfficeQA harnesses on FinanceBench without further~\mbox{optimization}.}

\subsection{\texorpdfstring{{Benchmark, Data Allocation, and Scoring}}{Benchmark, Data Allocation, and Scoring}}
\label{app:financebench-setup}

\noindent{FinanceBench~\citep{islam2023financebench} evaluates open-book question answering over the public filings of U.S. companies. Its full suite contains 10,231 questions on 40 companies. We use the 150 questions of its public release, 50 in each of three question types. A metrics-generated question asks for a figure in a financial statement or for a ratio of such figures. A domain-relevant question applies one of 25 generic analyst questions, for example whether a company pays a dividend or how its operating margin changed, to a specific company. A novel-generated question is written by an annotator for a specific company and filing. Each question has a reference answer, a justification, and evidence pages in one filing. The 150 questions cite 84 filings of 32 companies in 9 GICS sectors: 64 annual reports (10-K), 8 quarterly reports (10-Q), 6 current reports (8-K), and 6 earnings releases.}

{FinanceBench differs from OfficeQA in three ways. The documents are company filings instead of Treasury Bulletins, so the tables, units, and notes follow the conventions of corporate reporting. The agent must first find the filing of the right company and period among many similar filings. The answers are free-form, so a judge model replaces exact matching. The experiments also use a different agent model, gpt-5.4.}

{The corpus contains 366 parsed filings, including all 84 cited filings. It is indexed as on OfficeQA, and the agent uses the OfficeQA tools with the same budgets. Harness edits and protocol transformations therefore act on the same interface as in Section~\ref{sec:experiments}.}

{Before any inference, we stratify the 150 questions by question type and assign, with a fixed seed, 30 questions to $D_{\rm evo}$, 30 to $D_{\rm disc}$, 15 to each $D_{{\rm conf},t}$ for $t=1,2,3$, and 45 to $D_{\rm cert}$ (Table~\ref{tab:financebench-data}). Each question receives three rollouts on $D_{\rm evo}$ and $D_{\rm disc}$, ten on $D_{{\rm conf},t}$, and five on $D_{\rm cert}$.}

\begin{table}[htbp]
\centering
\caption{FinanceBench questions by type and task set. Each of the three confirmation sets $D_{{\rm conf},t}$ has the counts shown.}
\label{tab:financebench-data}
\small
\setlength{\tabcolsep}{5pt}
\begin{tabularx}{\textwidth}{@{}lYccccc@{}}
\toprule
Question type & What the question asks & $D_{\rm evo}$ & $D_{\rm disc}$ & $D_{{\rm conf},t}$ & $D_{\rm cert}$ & Total \\
\midrule
Metrics-generated & A statement figure or a ratio of figures & 10 & 10 & 5 & 15 & 50 \\
Domain-relevant & A generic analyst question about one company & 10 & 10 & 5 & 15 & 50 \\
Novel-generated & A company-specific question written by an annotator & 10 & 10 & 5 & 15 & 50 \\
\midrule
All & & 30 & 30 & 15 & 45 & 150 \\
\bottomrule
\end{tabularx}
\end{table}

{As in reference-guided LLM evaluation~\citep{zheng2023judging} and recent benchmarks of financial agents on SEC filings~\citep{bigeard2025finance}, a judge model (gpt-5.6-sol, low reasoning) decides whether the final answer is semantically equivalent to the reference answer. It accepts equivalent wording, formatting, units, and harmless rounding. It rejects different quantities, signs, or units, missing essential parts, contradictions, refusals, and empty answers. It sees only the candidate answer and the reference answer and returns a binary decision. Intervals average the rollouts within each question and use 20,000 question-level paired bootstrap resamples.}

\subsection{\texorpdfstring{{Harness Evolution on FinanceBench}}{Harness Evolution on FinanceBench}}
\label{app:financebench-evolution}

\noindent{Every method uses gpt-5.4 with high reasoning as the agent model and starts from the same initial harness $H_0$, the OfficeQA initial harness with the corpus path changed. The Proposer and the Challenger use gpt-5.6-sol with high reasoning, as on OfficeQA. \methodchase{} runs $T=3$ rounds with the same $\varepsilon$ and $\eta_{{\rm conf},t}$ as on OfficeQA. \methodmeta{} and \methodhc{} each run three rounds on the 30 questions of $D_{\rm evo}$ under $\Brel$ with three rollouts per question and keep their own candidate and selection rules (Appendix~\ref{app:comparison-methods}). \methodevolve{} runs its own training and selection on FinanceBench. Every method fixes its final harness before $D_{\rm cert}$ is opened.}

{{\bf Tracing CHASE.} Table~\ref{tab:financebench-rounds} summarizes the three rounds. In each round, the validity firewall rejects one or two of the eight Challenger proposals, and \methodchase{} evaluates two of the remaining proposals on $D_{\rm disc}$ and sends the higher-ranked one to confirmation on $D_{{\rm conf},t}$. In the first round, the Proposer adds guidance that anchors each value to the requested fiscal period and the matching filing, and $H_1$ gains $4.44\%$ over $H_0$ on $D_{\rm evo}$. The Challenger sends a tag that states a boundary offset of $+10$ lines on every retrieval match to confirmation. Its gain destruction is $4.44\%$ on $D_{\rm disc}$ but $-1.56\%$ on $D_{{\rm conf},1}$, below the 7.5\% confirmation threshold, so the archive keeps only $\Brel$. In the second round, the Proposer adds query rewriting that removes stopwords, splits a query at conjunctions and punctuation into at most three parts, and matches any query term, and $H_2$ gains $7.78\%$ on $D_{\rm evo}$. The Challenger now targets $H_2$ with the opposite tag: every retrieval match receives an extra field that states a boundary offset of $-10$ lines, while the returned text and the line bounds are unchanged. Its gain destruction is $12.22\%$ on $D_{\rm disc}$ and $9.78\%$ on $D_{{\rm conf},2}$, above $\eta_{{\rm conf},2}$, so the tag enters the archive as $B_{b_2}$. The tag changes no evidence: a harness gain that falls under $B_{b_2}$ depends on presentation metadata rather than on the content of the filings. In the third round, every candidate satisfies the archive constraint, and the Proposer selects the integrated candidate, which keeps the query rewriting of $H_2$ and adds guidance on operand scope and answer format. The proposal that the Challenger sends to confirmation, which places table context after each table, destroys $0.00\%$ of the gain on $D_{{\rm conf},3}$ and does not enter the archive.}

\begin{table}[htbp]
\centering
\caption{\methodchase{} on FinanceBench by round. Each round selects one Proposer candidate on the 30 questions of $D_{\rm evo}$ (three rollouts each) and sends one Challenger proposal to confirmation on the 15 questions of $D_{{\rm conf},t}$ (ten rollouts each). The set $D_{\rm disc}$ has 30 questions with three rollouts each. Gains are over $H_0$ in the same round, and the confirmation threshold is 7.5\%. Only the $-10$-line tag enters the archive. All values are in \%.}
\label{tab:financebench-rounds}
\small
\setlength{\tabcolsep}{4pt}
\begin{tabularx}{\textwidth}{@{}clcYccc@{}}
\toprule
Round & Selected candidate & $\widehat G_{\rm rel}$ on $D_{\rm evo}$ & Proposal sent to confirmation & $\widehat\Delta$ on $D_{\rm disc}$ & $\widehat\Delta$ on $D_{{\rm conf},t}$ & Archived \\
\midrule
1 & Guidance & +4.44 & Boundary-offset tag, $+10$ lines & 4.44 & $-$1.56 & No \\
2 & Structural & +7.78 & Boundary-offset tag, $-10$ lines & 12.22 & 9.78 & Yes ($B_{b_2}$) \\
3 & Integrated & +3.33 & Table context after each table & 6.67 & 0.00 & No \\
\bottomrule
\end{tabularx}
\end{table}

{After the third round, \methodchase{} re-evaluates $H_0$ through $H_3$ in a fresh pass on $D_{\rm evo}$ under both members of the final archive, with gains measured over $H_0$ in that pass. $H_3$ has the largest released gain among the four, $+2.22\%$, and its gain under $B_{b_2}$ is $+7.78\%$, so its gain destruction is $-5.56\%$, within $\varepsilon$. \methodchase{} therefore selects $H_3$ as its final harness. Its executed changes are the query rewriting of $H_2$, applied while at least eight tool calls remain, and four guidance modules, each activated by features of the question, that anchor each value to the requested fiscal period and filing, bind each operand to its statement row, column, unit, and scope, and ask for the substituted expression and the requested precision before the answer.}

{The final archive is $\mathcal A_3=\{\Brel,B_{b_2}\}$, and the certification metrics of Section~\ref{sec:experimental-methods-evaluation} apply unchanged. Table~\ref{tab:financebench-cert} reports them on $D_{\rm cert}$. \methodchase{} attains the best score on each evaluation metric, and its advantages over \methodmeta{} exclude zero: $+11.78\%$ $[6.00, 18.00]$ in $\widehat R_{{\rm avg},\mathcal A_3}$ and $+12.44\%$ $[5.78, 20.44]$ in $\widehat R_{\min,\mathcal A_3}$. Because $B_{b_2}$ lowers the score of every method, $\widehat R_{\min,\mathcal A_3}$ equals the score under $B_{b_2}$ for every method.}

\begin{table}[htbp]
\centering
\caption{FinanceBench certification on the 45 questions of $D_{\rm cert}$, five rollouts each, with gpt-5.4 as the agent model. All methods start from the same initial harness $H_0$, and \methodraw{} keeps it unchanged. The final archive $\mathcal A_3$ contains $\Brel$ and the boundary-offset tag $B_{b_2}$ confirmed by \methodchase{}. Bold marks the best score in each column. \methodchase{} attains the highest score on each evaluation metric. All scores are in \%. Small brackets give question-level paired bootstrap 95\%~confidence~intervals.}
\label{tab:financebench-cert}
\small
\setlength{\tabcolsep}{0pt}
\begin{tabular*}{\linewidth}{@{\extracolsep{\fill}}l r@{\extracolsep{0pt}\hspace{3pt}}l@{\extracolsep{\fill}} r@{\extracolsep{0pt}\hspace{3pt}}l@{\extracolsep{\fill}} r@{\extracolsep{0pt}\hspace{3pt}}l@{}}
\toprule
Method & \multicolumn{2}{c}{$\widehat R_{\rm rel}(\uparrow)$} & \multicolumn{2}{c}{$\widehat R_{{\rm avg},\mathcal A_3}(\uparrow)$} & \multicolumn{2}{c}{$\widehat R_{\min,\mathcal A_3}(\uparrow)$} \\
\midrule
\methodraw{} & 76.89 & \ci{66.67}{86.22} & 75.78 & \ci{65.56}{85.11} & 74.67 & \ci{64.00}{84.00} \\
\methodmeta{} & 67.11 & \ci{57.33}{76.44} & 65.56 & \ci{55.56}{75.11} & 64.00 & \ci{53.33}{73.78} \\
\methodhc{} & 77.78 & \ci{67.11}{87.56} & 76.00 & \ci{65.78}{85.56} & 74.22 & \ci{64.00}{84.00} \\
\methodevolve{} & 76.89 & \ci{65.78}{87.11} & 75.78 & \ci{65.11}{85.78} & 74.67 & \ci{63.54}{84.89} \\
\methodchase{} & \textbf{78.22} & \ci{67.56}{88.00} & \textbf{77.33} & \ci{66.89}{86.89} & \textbf{76.44} & \ci{65.78}{86.22} \\
\bottomrule
\end{tabular*}
\end{table}

{Table~\ref{tab:financebench-gains} reports the gain of each optimized method over \methodraw{} under $\Brel$ and $B_{b_2}$ and the gain destruction $\widehat\Delta_{b_2}$. \methodchase{} has the largest gain under both protocols, and its gain under $B_{b_2}$ ($+1.78\%$) exceeds its gain under $\Brel$ ($+1.33\%$), so its gain destruction is negative. The gain of \methodhc{} instead changes sign: $+0.89\%$ under $\Brel$ and $-0.44\%$ under $B_{b_2}$. The archive therefore reorders the methods that the released score ranks second and third: \methodhc{} leads \methodraw{} and \methodevolve{} under $\Brel$ and trails both under $B_{b_2}$. The gain of \methodmeta{} under $\Brel$ is $-9.78\%$ $[-16.89, -2.67]$. \methodmeta{} selects its final harness among all of its candidates by the released score on the 30 questions of $D_{\rm evo}$. The selected harness, the best of its nine candidates, scores $84.44\%$ there against $81.11\%$ for $H_0$, yet falls $9.78\%$ below \methodraw{} on $D_{\rm cert}$. Selection by the released score on a small search set thus rewards a gain that does not hold on the certification questions.}

\begin{table}[htbp]
\centering
\caption{Gains over \methodraw{} on $D_{\rm cert}$ under $\Brel$ and the boundary-offset tag $B_{b_2}$, and the gain destruction $\widehat\Delta_{b_2}=\widehat G_{\rm rel}-\widehat G_{b_2}$, which is better when lower. Bold marks the best value in each column. All values are in \%. Small brackets give question-level paired bootstrap 95\%~confidence~intervals.}
\label{tab:financebench-gains}
\small
\setlength{\tabcolsep}{0pt}
\begin{tabular*}{\linewidth}{@{\extracolsep{\fill}}l r@{\extracolsep{0pt}\hspace{3pt}}l@{\extracolsep{\fill}} r@{\extracolsep{0pt}\hspace{3pt}}l@{\extracolsep{\fill}} r@{\extracolsep{0pt}\hspace{3pt}}l@{}}
\toprule
Method & \multicolumn{2}{c}{$\widehat G_{\rm rel}(\uparrow)$} & \multicolumn{2}{c}{$\widehat G_{b_2}(\uparrow)$} & \multicolumn{2}{c}{$\widehat\Delta_{b_2}(\downarrow)$} \\
\midrule
\methodmeta{} & $-$9.78 & \ci{$-$16.89}{$-$2.67} & $-$10.67 & \ci{$-$19.56}{$-$2.22} & +0.89 & \ci{$-$8.00}{9.78} \\
\methodhc{} & +0.89 & \ci{$-$4.44}{6.22} & $-$0.44 & \ci{$-$4.00}{2.67} & +1.33 & \ci{$-$5.78}{8.89} \\
\methodevolve{} & 0.00 & \ci{$-$5.78}{5.78} & 0.00 & \ci{$-$7.56}{6.67} & 0.00 & \ci{$-$8.44}{9.33} \\
\methodchase{} & \textbf{+1.33} & \ci{$-$4.89}{7.56} & \textbf{+1.78} & \ci{$-$3.11}{7.11} & \textbf{{\boldmath$-$}0.44} & \ci{$-$6.67}{5.78} \\
\bottomrule
\end{tabular*}
\end{table}

{Table~\ref{tab:financebench-types} splits the gains by question type. Domain-relevant questions carry the largest released gains, and \methodchase{} has the largest one ($+12.00\%$). Under $B_{b_2}$ every optimized method falls below \methodraw{} on these questions, and \methodchase{} falls least ($-2.67\%$). On novel-generated questions, the gains of \methodmeta{}, \methodhc{}, and \methodevolve{} under $\Brel$ are negative with intervals that exclude zero, whereas the interval of \methodchase{} contains zero.}

\begin{table}[htbp]
\centering
\caption{Gains over \methodraw{} on $D_{\rm cert}$ by question type, 15 questions per type and five rollouts each, under $\Brel$ and the boundary-offset tag $B_{b_2}$. All values are in \%. Small brackets give question-level paired bootstrap 95\% confidence intervals within each type.}
\label{tab:financebench-types}
\small
\setlength{\tabcolsep}{4pt}
\begin{tabular}{ll r@{\hspace{3pt}}l r@{\hspace{3pt}}l}
\toprule
Question type & Method & \multicolumn{2}{c}{$\widehat G_{\rm rel}$} & \multicolumn{2}{c}{$\widehat G_{b_2}$} \\
\midrule
\multirow{4}{*}{Metrics-generated} & \methodmeta{} & $-$16.00 & \ci{$-$24.00}{$-$8.00} & $-$16.00 & \ci{$-$30.67}{$-$2.67} \\
 & \methodhc{} & +1.33 & \ci{$-$2.67}{5.33} & +1.33 & \ci{$-$2.67}{5.33} \\
 & \methodevolve{} & 0.00 & \ci{$-$8.00}{6.67} & +6.67 & \ci{0.00}{13.33} \\
 & \methodchase{} & $-$4.00 & \ci{$-$9.33}{1.33} & +2.67 & \ci{$-$2.67}{8.00} \\
\midrule
\multirow{4}{*}{Domain-relevant} & \methodmeta{} & 0.00 & \ci{$-$16.00}{14.67} & $-$14.67 & \ci{$-$33.33}{1.33} \\
 & \methodhc{} & +9.33 & \ci{$-$4.00}{20.00} & $-$4.00 & \ci{$-$10.67}{1.37} \\
 & \methodevolve{} & +8.00 & \ci{$-$6.67}{21.33} & $-$10.67 & \ci{$-$29.33}{5.33} \\
 & \methodchase{} & +12.00 & \ci{$-$2.67}{25.33} & $-$2.67 & \ci{$-$12.00}{6.67} \\
\midrule
\multirow{4}{*}{Novel-generated} & \methodmeta{} & $-$13.33 & \ci{$-$22.67}{$-$4.00} & $-$1.33 & \ci{$-$13.33}{13.33} \\
 & \methodhc{} & $-$8.00 & \ci{$-$16.00}{$-$1.33} & +1.33 & \ci{$-$4.00}{6.67} \\
 & \methodevolve{} & $-$8.00 & \ci{$-$16.00}{$-$1.33} & +4.00 & \ci{$-$4.00}{13.33} \\
 & \methodchase{} & $-$4.00 & \ci{$-$13.33}{4.00} & +5.33 & \ci{$-$4.00}{18.67} \\
\bottomrule
\end{tabular}
\end{table}

\subsection{\texorpdfstring{{Transfer of Frozen OfficeQA Harnesses}}{Transfer of Frozen OfficeQA Harnesses}}
\label{app:financebench-transfer}

\noindent{We freeze the final OfficeQA harnesses of \methodraw{} ($H_0$), \methodchase{} ($H_3$), and \methodevolve{} and evaluate them on all 150 FinanceBench questions under $\Brel$ with five rollouts each. The harnesses were developed with gpt-5.6-sol on OfficeQA and run here with gpt-5.4, so the transfer changes both the benchmark and the agent model. Only the binding of the external tools to the FinanceBench corpus changes. No method is optimized or reselected on FinanceBench.}

{Table~\ref{tab:financebench-transfer} reports the scores. The frozen \methodchase{} harness matches \methodraw{}, with a paired difference of $0.00\%$ $[-3.33, 3.07]$, and scores above the frozen \methodevolve{} harness, with a paired difference of $+0.67\%$ $[-2.53, 3.73]$. The edits that \methodchase{} accepted on OfficeQA therefore show no observed mean-score loss on a benchmark and an agent model that played no role in their development. Sections~\ref{app:financebench-evolution} and~\ref{app:financebench-transfer} differ in their harnesses and question sets, so their tables answer different questions and are not compared directly.}

\begin{table}[htbp]
\centering
\caption{Frozen OfficeQA harnesses on all 150 FinanceBench questions under $\Brel$, five rollouts each, with gpt-5.4 as the agent model. All scores are in \%. Small brackets give question-level paired bootstrap 95\% confidence intervals.}
\label{tab:financebench-transfer}
\small
\setlength{\tabcolsep}{4pt}
\begin{tabular}{l r@{\hspace{3pt}}l}
\toprule
Frozen harness & \multicolumn{2}{c}{$\Brel$} \\
\midrule
\methodraw{} & 73.20 & \ci{66.93}{79.33} \\
\methodevolve{} & 72.53 & \ci{66.00}{78.80} \\
\methodchase{} & 73.20 & \ci{66.80}{79.33} \\
\bottomrule
\end{tabular}
\end{table}

\FloatBarrier
\renewcommand{\theequation}{F.\arabic{equation}}
\renewcommand{\thetable}{F.\arabic{table}}
\renewcommand{\thefigure}{F.\arabic{figure}}
\setcounter{equation}{0}
\setcounter{table}{0}
\setcounter{figure}{0}

\section{\texorpdfstring{{Additional Results on SWE-Bench Pro V2}}{Additional Results on SWE-Bench Pro V2}}
\label{app:swebench}

{This section reports harness evolution on SWE-Bench Pro V2 (SWE-Bench Pro for short), a benchmark of software engineering tasks for coding agents. Because coding-agent rollouts are costly, we compare \methodchase{} only with \methodraw{} and with \methodevolve{}, the baseline with the highest released score on OfficeQA and Syn-Ledger (Tables~\ref{tab:main-results} and~\ref{tab:synledger-results}). Section~\ref{app:swebench-setup} describes the benchmark, the harness and the protocol in this setting, the data allocation, and the scoring. Section~\ref{app:swebench-evolution} runs \methodchase{} and the comparison methods, traces the \methodchase{} run, and reports the certification results.}

\subsection{\texorpdfstring{{Benchmark, Harness, Data Allocation, and Scoring}}{Benchmark, Harness, Data Allocation, and Scoring}}
\label{app:swebench-setup}

\noindent{SWE-Bench Pro V2~\citep{dan2026swebenchprov2} is a revised public split of SWE-Bench Pro with 642 tasks from 11 repositories written in Go, Python, JavaScript, and TypeScript. Each task gives a repository snapshot and an issue that states the required behavior and interface. A submitted patch succeeds if it passes all fail-to-pass and pass-to-pass tests, which the agent does not see. The release prescribes how agents are run, and every patch is graded in a clean container. We use its HARD-51 subset, the 51 tasks that at least two of five model families fail.}

{SWE-Bench Pro differs from OfficeQA and FinanceBench in three ways. The agent changes code in a repository through shell commands instead of answering a question from documents. Success is decided by hidden tests instead of an exact or judged match with a reference answer. The experiments use another agent model and agent framework: claude-sonnet-4-6 with high reasoning effort in mini-swe-agent 2.4.6, with the Proposer and the Challenger using claude-opus-5-5 with high reasoning effort.}

{The harness has the six components of Section~\ref{sec:officeqa-setup}, realized in mini-swe-agent. The system prompt and the long-term memory are part of the agent's instructions. The retrieval policy rewrites file reads, for example into windows of a fixed number of lines, and the tool policy holds rules for shell commands. The policy graph holds rules over features of the running episode, such as the number of steps used and the number of files modified, and its actions inject guidance modules into the episode. The protocol consists of how the issue is presented and how each command observation, a JSON object with the command output and its metadata, is serialized. All observations are subject to the same byte quota.}

{For the formal experiment we stratify the 51 tasks by repository and assign, with a fixed seed, 10 tasks each to $D_{\rm evo}$, $D_{\rm disc}$, and $D_{\rm conf}$ and 21 tasks to $D_{\rm cert}$ (Table~\ref{tab:swebench-data}). The search rounds reuse the same ten confirmation tasks for empirical archive admission. Every evaluated harness receives five rollouts per task in every task set, and $D_{\rm cert}$ stays sealed until the final harness of each method is frozen.}

\begin{table}[htbp]
\centering
\caption{HARD-51 tasks by repository language and task set. The 11 repositories are grouped by their main language.}
\label{tab:swebench-data}
\small
\setlength{\tabcolsep}{6pt}
\begin{tabular}{lrrrrr}
\toprule
Language & $D_{\rm evo}$ & $D_{\rm disc}$ & $D_{\rm conf}$ & $D_{\rm cert}$ & Total \\
\midrule
Go & 3 & 3 & 5 & 9 & 20 \\
Python & 5 & 5 & 3 & 7 & 20 \\
JavaScript or TypeScript & 2 & 2 & 2 & 5 & 11 \\
\midrule
Total & 10 & 10 & 10 & 21 & 51 \\
\bottomrule
\end{tabular}
\end{table}

{A rollout scores 1 if its patch passes all required tests and 0 otherwise. Each harness score is the mean over tasks of the per-task mean over rollouts. Intervals average the rollouts within each task and use 20,000 task-level paired bootstrap resamples.}

\subsection{\texorpdfstring{{Harness Evolution on SWE-Bench Pro}}{Harness Evolution on SWE-Bench Pro}}
\label{app:swebench-evolution}

\noindent{Every method uses the same agent model and starts from the same initial harness $H_0$: the default system prompt of mini-swe-agent, empty long-term memory and guidance modules, and empty retrieval, tool, and policy rules. \methodraw{} keeps $H_0$. \methodchase{} runs at most three rounds before a counterfactual enters the archive. In each round, the Proposer submits up to three candidates on $D_{\rm evo}$, and the Challenger submits transformations, of which at most two are evaluated on $D_{\rm disc}$ and the higher-ranked one is sent to confirmation on $D_{\rm conf}$. We set $\varepsilon$ and the confirmation threshold to $0.05$. After admission, the Proposer may generate up to five batches of three candidates and stops at the first batch that improves the best feasible released score.}

{\methodevolve{} pools the 20 tasks of $D_{\rm evo}$ and $D_{\rm conf}$ for training and runs three epochs, each with three batches of seven, seven, and six tasks. A candidate is accepted if it scores at least as well as the current harness on the current batch and loses at most $5\%$ on each of the two most recent batches. At the end of each epoch, the harness with the highest released score on the 10 tasks of $D_{\rm disc}$, chosen from \methodraw{} and the accepted harnesses, becomes the current harness. Both methods therefore use the same 30 development tasks.}

{The Challenger can frame the issue between \texttt{<issue>} and \texttt{</issue>} tags, reverse the order of the top-level fields of each command observation, or combine the two operations. The validity firewall screens candidates during development. It requires that the original issue text can be recovered from the presented issue. It requires that the repository, its tests, and the gold patch are unchanged, and that the gold patch still passes all required tests under the transformation. It also requires that every source line visible under the byte quota is the same as under $\Brel$. JSON compaction was excluded before search because it changes which source lines fit under the quota. Gold patches passed under the archived transformation on all 30 development tasks and, after the archive and CHASE harness were frozen, on all 21 certification tasks.}

{{\bf Tracing CHASE.} In the first round, the Proposer limits every file read to a window of 200 lines, and $H_1$ scores $44\%$ on $D_{\rm evo}$ against $42\%$ for $H_0$. In the second round, the Proposer adds a workflow to the system prompt and lessons to the long-term memory, and $H_2$ scores $48\%$ on $D_{\rm evo}$. The workflow turns the issue into a checklist, searches for the named identifiers before reading code, starts editing within about 20 to 25 steps, wires every addition into imports, registries, and configuration, verifies the change with the repository's own build and test tools, and reviews the full diff before submission. The Challenger proposes field reversal, issue framing, and their combination, and ranks the combination first on $D_{\rm disc}$ (gain destruction of $-16$ percentage points, compared with $-18$ for the other screened candidate). It sends the combination $\Phi_{b_2}$ to confirmation. On $D_{\rm conf}$, the gain of $H_2$ over $H_0$ is $+8\%$ under $\Brel$ and $-8\%$ under $B_{b_2}$, so its gain destruction is $16\%$, above the threshold, and $B_{b_2}$ enters the archive. The counterfactual $B_{b_2}$ leaves every key, value, and nesting level of the observations unchanged and keeps the issue text recoverable, so fields remain accessible by name. The reversal of the gain measures sensitivity to the combined presentation change (Appendix~\ref{app:legitimate-vs-shortcut}).}

{With $B_{b_2}$ in the archive, the Proposer selects $H_3$ in the second repair batch, which scores $52\%$ on $D_{\rm evo}$. The final selection evaluates $H_0$ through $H_3$ on $D_{\rm evo}$ under both members of the archive. Among the candidates that satisfy the constraint, $H_3$ has the highest released score, and its gain over $H_0$ is $+10\%$ under both protocols, so its gain destruction is $0\%$. \methodchase{} therefore selects $H_3$ as its final harness. Relative to $H_2$, $H_3$ plans about 15 exploration steps before the first edit, adds memory lessons for Ansible modules and Go configuration changes, and adds one policy-graph rule: when no file has been modified and the step count is 30 or 31, the rule adds guidance to begin editing at the interfaces named in the issue. The certification logs record 199 instances in which the next model call received this guidance.}

{The final archive is $\mathcal A=\{\Brel,B_{b_2}\}$, and the certification metrics of Section~\ref{sec:experimental-methods-evaluation} apply unchanged. Table~\ref{tab:swebench-results} reports them on the 21 tasks of $D_{\rm cert}$ with five rollouts per task under each protocol. \methodchase{} attains the best score on each evaluation metric. Its advantage over \methodraw{} is $+11.43\%$ $[1.90, 22.86]$ under $\Brel$ and $+10.95\%$ $[0.48, 22.86]$ in $\widehat R_{{\rm avg},\mathcal A}$, and its scores exceed those of \methodevolve{} by $2.86\%$ under $\Brel$, $4.76\%$ in $\widehat R_{{\rm avg},\mathcal A}$, and $6.67\%$ in $\widehat R_{\min,\mathcal A}$. Under $B_{b_2}$, \methodchase{} scores $67.62\%$, \methodevolve{} $60.95\%$, and \methodraw{} $57.14\%$. The certified gain destruction of \methodchase{} is therefore only $0.95\%$. The gain of \methodevolve{} over \methodraw{} falls from $+8.57\%$ under $\Brel$ to $+3.81\%$ under $B_{b_2}$, so its gain destruction is $4.76\%$. The paired intervals for the score differences between \methodchase{} and \methodevolve{} include zero. Whereas the gain of $H_2$ reversed under $B_{b_2}$ during the search, the gain of the final harness holds on the sealed tasks.}

{The final harness is also cheaper to run. Over the 210 certification rollouts of each method, the agent consumes 579.7 million tokens with $H_3$ and 746.7 million tokens with $H_0$, so the \methodchase{} harness uses 22\% fewer tokens.}

\end{document}